\documentclass[letterpaper,12pt]{article}

\usepackage{jmlr2e}
\usepackage[utf8]{inputenc}
\usepackage[T1]{fontenc}
\usepackage[english]{babel}
\usepackage{mathtools}
\usepackage{mleftright}
\usepackage[mathscr]{euscript}
\usepackage{dutchcal}
\usepackage{graphicx,url}
\usepackage[left=1in,right=1in,top=1in,bottom=1in,letterpaper]{geometry}

\usepackage{dsfont}
\usepackage{setspace}
\usepackage{amssymb}
\usepackage[margin = 1cm]{caption}
\usepackage{sidecap}
\usepackage{subfigure}
\usepackage{float}
\usepackage{amsmath}
\usepackage{color,epstopdf}
\usepackage[noend]{algpseudocode}
\usepackage{hyperref}
\usepackage{algorithm,algpseudocode}
\usepackage{breqn}

\usepackage{epsfig}
\usepackage{tabu}
\usepackage[table,xcdraw]{xcolor}
\usepackage{arydshln}
\usepackage{xcolor}
\usepackage{mathrsfs}

\newcommand{\expect}{\operatorname{\mathbb{E}}\expectarg}
\DeclarePairedDelimiterX{\expectarg}[1]{[}{]}{%
  \ifnum\currentgrouptype=16 \else\begingroup\fi
  \activatebar#1
  \ifnum\currentgrouptype=16 \else\endgroup\fi
}

\newcommand{\LinesNumbered}{
  \setboolean{algocf@linesnumbered}{true}%
  \renewcommand{\algocf@linesnumbered}{\everypar={\nl}}}%
\let\oldnl\nl
\newcommand{\nonl}{\renewcommand{\nl}{\let\nl\oldnl}}

\newcommand{\innermid}{\nonscript\;\delimsize\vert\nonscript\;}
\newcommand\tab[1][1cm]{\hspace*{#1}}
\newcommand{\quotes}[1]{``#1''}
\newcommand{\activatebar}{%
  \begingroup\lccode`\~=`\|
  \lowercase{\endgroup\let~}\innermid 
  \mathcode`|=\string"8000
}

\algnewcommand{\Inputs}[1]{%
  \State \textbf{Inputs:}
  \Statex \hspace*{\algorithmicindent}\parbox[t]{.8\linewidth}{\raggedright #1}
}
\algnewcommand{\Initialize}[1]{%
  \State \textbf{Initialize:}
  \Statex \hspace*{\algorithmicindent}\parbox[t]{.8\linewidth}{\raggedright #1}
}
\algnewcommand{\Data}[1]{%
  \State \textbf{Data:}
  \Statex \hspace*{\algorithmicindent}\parbox[t]{.8\linewidth}{\raggedright #1}
}
\DeclareMathOperator{\E}{\mathbb{E}}
\DeclareMathOperator{\arcsinh}{arcsinh}

\newtheorem{lem}{Lemma}
\newtheorem{pro}{Proposition}
\newtheorem{claim}{Claim}

\begin{document}
\title{Stochastic Large-scale Machine Learning Algorithms with Distributed Features and Observations}

\author{\name Biyi Fang \email biyifang2021@u.northwestern.edu \\
    \addr Department of Engineering Science and Applied Mathematics\\
    Northwestern University\\
    Evanston, IL 60208, USA
    \AND
    \name  Diego Klabjan \email d-klabjan@northwestern.edu \\
    \addr Department of Industrial Engineering and Management Sciences\\
    Northwestern University\\
    Evanston, IL 60208, USA}
    
\editor{}

\maketitle



\begin{abstract}
\noindent As the size of modern datasets exceeds the disk and memory capacities of a single computer, machine learning practitioners have resorted to parallel and distributed computing. Given that optimization is one of the pillars of machine learning and predictive modeling, distributed optimization methods have recently garnered ample attention, in particular when either observations or features are distributed, but not both.  
We propose a general stochastic algorithm where observations, features, and gradient components can be sampled in a double distributed setting, i.e., with both features and observations distributed. Very technical analyses establish convergence properties of the algorithm under different conditions on the learning rate (diminishing to zero or constant). Computational experiments in Spark demonstrate a superior performance of our algorithm versus a benchmark in early iterations of the algorithm, which is due to the stochastic components of the algorithm.\\
\noindent \textbf{Keywords.} optimization, machine learning, stochasticity, convexity, large scale
\end{abstract}

\section{Introduction}
\tab As technology advances, collecting and analyzing large-scale and real time data has become widely used in a variety of fields. Large-scale machine learning can not only present a useful summary of  a dataset but it can also make predictions. In the era of big data, large scale datasets have become more accessible. \quotes{Large} usually refers to both the number of observations and high feature dimension. For example, an English Wikipedia dataset can have 11 million documents (observations) and over several hundred of thousands unique word types (features). This demands a sophisticated large-scale machine learning system able to take advantages of all available information from such datasets and not only random samples. On the other hand, as large scale data is becoming more accessible, storing the whole dataset on a single server is often impossible due to the inadequate disk and memory capacities. Consequently, it is considerable to store and analyze them distributively. Often the data collection process by design distributes observations and features. There is a large amount of literature dealing with optimization problems subject to a dataset with either distributed observations or distributed features. Nevertheless, very limited contribution has been made to the case where both the observations and features are in a distributed environment. 

In this paper, we propose an algorithm,  namely SODDA (StOchastic Doubly Distributed Algorithm), designed for a doubly distributed dataset and inspired by the work  \cite{harikandeh2015stopwasting} and \cite{nathan2017optimization}. The algorithm is aimed to solve a series of optimization problems which can be formulated as the minimization of a finite sum of convex functions plus a convex regularization term if necessary. SODDA is a primal method building on the  previous RAndom Distributed Stochastic Algorithm (RADiSA) \cite{nathan2017optimization}. SODDA  first further splits the partitions (a partition is a set of features and observations stored locally) with respect to features into sub-partitions with no overlap; then in each iteration, randomly chooses sub-partitions associated with different blocks of features; lastly, similar to stochastic gradient descent (SGD), updates in parallel each sub-block of the current local solution by using observations from the randomly selected sub-partition of local observations and the local sub-block of features, coupled with the Stochastic Variance-Reduced Gradient (SVRG). One generalization of SVRG utilized in SODDA is that SODDA does not require a full solution update; instead, it allows each sub-block of the current local solution to be updated individually and assembled at the end of each iteration. Although, we might amplify the error by approximately computing the gradient, SODDA reduces the communication cost significantly. Another technique aiming to cut down the communication cost is estimating the full gradient needed as part of the SVRG component, which is a big distinction between SODDA and RADiSA. RADiSA requires the full gradient in each outer iteration, which is computationally demanding, especially when the solution is far from an optimal solution. SODDA has three stochastic components: the first two are that it randomly selects blocks of local features and subsets of local observations to execute the estimated gradient, and the third component that randomly chooses further sub-blocks of local features to record the approximated gradient, which contributes to a reduction of the number of gradient coordinate computations required in early iterations. In other words, only random coordinates of the gradient are computed.

In this paper, we not only propose a more computationally efficient method, SODDA, when compared to RADiSA \cite{nathan2017optimization}, but also present a complete technical proof of convergence. For a smooth and strongly convex function, we prove that SODDA enjoys at least a sublinear convergence rate and a linear convergence rate for a  diminishing learning rate and a constant learning rate, respectively. Furthermore, we prove that SODDA iterates converge to an optimal solution when using a constant learning rate selected from a certain interval. Moreover, the convergence property of RADiSA, which is not provided in \cite{nathan2017optimization}, is implied directly from SODDA. In summary, we make the following five contributions.
\begin{itemize}
\item We provide a better scalable  stochastic doubly distributed method, i.e. SODDA, for doubly distributed datasets. This algorithm  does not require the calculation of a full gradient, thus it is a less computationally intensive methodology for doubly distributed setting problems.
\item We provide a proof of a sublinear convergence result for smooth and strongly-convex loss functions when using a sequence of decreasing learning rates. 
\item We show that SODDA iterates converge with linear rate to a neighborhood of an optimal solution  when using an arbitrary constant learning rate.
\item  We further argue that SODDA iterates converge to an optimal solution in the strongly-convex case when using any constant learning rate in a specified interval.
\item We present numerical results showing that SODDA outperforms RADiSA-avg, which is the best doubly distributed optimization algorithm in \cite{nathan2017optimization}, on all instances considered in early iterations. More precisely, SODDA finds good quality solutions faster than RADiSA-avg. 
\end{itemize}
 
The paper is organized as follow. In the next section, we review several related works in distributed optimization. In Section 3, we state the formal optimization problem and standard assumptions underlying our analyses, followed by the exposition of the SODDA algorithm. In Section 4, we show the convergence analyses of SODDA with respect to both a decreasing learning rate and a constant learning rate. In Section 5, we present experimental results comparing SODDA with RADiSA-avg. 

\section{Related Work}
There are a large number of extensions of the plain stochastic gradient descent algorithm related to distributed datasets, however, a full retrospection of this immense literature exceeds the scope of this work. In this section, we state several approaches which are most related to our new method and interpret the relationships among them.

In plain SGD, the gradient of the aggregate function is approximated by one randomly picked function \cite{robbins1951stochastic}. It saves a heavy load of computation when compared with gradient descent, whereas  more often than not, the convergence happens to be slow. Recently, a large variety of approaches have been proposed targeted on accelerating the convergence rate and dealing with observations in the distributed setting. \vspace{0.5cm}\\
\textit{\textbf{SGD for distributed observations:}} One attempt that works for datasets with distributed observations is parallelizing it by means of mini-batch SGD. Both the synchronous version \cite{chen2016revisiting} and the asynchronous version \cite{tsitsiklis1986distributed}  basically work in the following way: the parameter server performs parameter updates after all worker nodes send their own gradients based on local information in parallel, and then broadcasts the updated parameters to all worker nodes afterwards. In the synchronous approach, the master node needs to wait until all gradients are collected but in the asynchronous approach, the master node performs updates whenever it is needed. An alternative method which introduces the concept of variance reduced is CentralVR \cite{de2016efficient}, where the master node not only needs to spread parameters but also the full gradient after every certain number of iterations, and each worker node would involve the full gradient as a corrector when computing their own gradients. As a consequence, the variance in the estimation of the  stochastic gradient could be reduced and a larger learning rate is allowed to accomplish faster convergence and higher accuracy. \vspace{0.5cm}\\
\textit{\textbf{SGD for distributed features:}} Another attempt for distributed features is parallelization via features. Block successive upper bound minimization (BSUM)\cite{hong2015unified} is one of the methods working for datasets with distributed features, where the master node spreads all parameters and each worker node conducts parameter updates on a randomly chosen and non-overlapping subset of the feature vector. Distributed Block Coordinate Descent  \cite{marevcek2015distributed} is another approach designed for datasets with distributed features. The parameters associated with these feature blocks are partitioned accordingly. In the algorithm, each processor randomly chooses a certain number of blocks out of those stored locally, performs the corresponding parameter updates in parallel, and then transmits to other processors. However, it is impossible to avoid communication when computing the gradient of all parameters unless there are extra assumptions on the objective loss function, which does not usually hold. An alternative approach is Communication-Efficient Distributed Dual Coordinate Ascent
(CoCoA) \cite{jaggi2014communication}, which is a primal-dual method also working for data with distributed features. By exploiting the fact that the associated blocks of dual variables work in different processors without overlap, the algorithm aggregates the parallel updates efficiently from the different processors without much conflict and reduces the necessary communication dramatically. A faster converging method  extended from the aforementioned approach is CoCoA$^+$  \cite{ma2015adding}, which allows a larger learning rate for parameter updates by introducing a more generalized local CoCoA subproblem at each processor. \vspace{0.5cm}\\
\textit{\textbf{SGD for distributed observations and features:}} Given datasets with distributed observations and features, all the  methods mentioned so far are not applicable. One of the algorithms fitting the bill is a block distributed ADMM \cite{parikh2014block}, which is the block splitting variant of ADMM. Nonetheless, the convergence rate of ADMM-based methods is slow. Random parallel stochastic algorithm (RAPSA)\cite{mokhtari2016class} is another algorithm which utilizes multiple parallel units to operate on a randomly chosen subset of blocks of the feature vector. It needs to access the whole feature vector to perform a parameter update while SODDA only needs a subset of the feature. Decentralized double stochastic averaging gradient algorithm (DSA) \cite{mokhtari2016dsa} is an alternative method designed for doubly distributed datasets, whereas the global cost function that DSA optimizes is a linear combination of the local objective functions which only contain local parameters, compared to the loss function of SODDA which contains global parameters. A faster converging and more pertinent algorithm is RADiSA \cite{nathan2017optimization}, which is also focusing on settings where both the observations and features of the problem at hand are stored in a distributed fashion. RADiSA conducts parameter updates based on stochastic partial gradients, which are calculated from randomly selected local observations and a randomly assigned sub-block of local features in parallel. RADiSA is a special case of SODDA, since the full gradient required by it is replaced by an approximated gradient which only uses partial observations and features. Consequently, SODDA provides a faster convergence than RADiSA without sacrificing too much accuracy. Meanwhile, we present technical convergence analyses for SODDA under different types of learning rates which imply the convergence of RADiSA.

\section{Algorithm}
We consider the problem of optimizing a finite but large sum of smooth functions, i.e., given a training set $\left\{(x_i,y_i)\right\}_{i=1}^N$, where each $x_i^T\in\mathbb{R}^d$ is associated with a corresponding label $y_i$,
\begin{align}
\label{obj}
\min_{\omega\in\mathbb{R}^d}F(\omega):=\frac{1}{N}\sum_{i=1}^N \bar{f}(x_i\omega,y_i)=\frac{1}{N}\sum_{i=1}^N f_i(x_i\omega).
\end{align}
Several machine learning loss functions fit this model, e.g. least square, logistic regression, and hinge loss. 

In SODDA, we assume that the training set $\left\{(x_i,y_i)\right\}_{i=1}^N$ is distributed across both observations and features. More specifically, the features and observations are split into $Q$ and $P$ partitions, respectively. We denote the matrix corresponding to the $p$'th observation partition and its $q$'th feature partition as $x^{p,q}\in\mathbb{R}^{\frac{N}{P}\times\frac{d}{Q}}$. Note that the partitions consisting of same features share the common block of parameters $\omega_{[q]}$. In Figure \hyperref[fig:1]{1}, there are 12 partitions. Parameters $\omega_{[1]}$ correspond to all parameters under $q=1$. In order to efficiently parallelize the computation, optimization of each partition can be done concurrently by considering only local observations and features. This strategy poses a big challenge in how to  combine the parameters. For example, $\omega_{[1]}$ are modified by all processors working on $x^{1,1}$, $x^{2,1}$, $x^{3,1}$, $x^{4,1}$. It is unclear how to combine them (averaging them is a possible strategy however this would not yield convergence, see e.g. \cite{weimer2010convenient}, \cite{zinkevich2010parallelized}). To circumvent this, we further artificially subdivide the features.  

To this end, we define a function $\pi_q(p):\left\{1,2,\cdots,P\right\}\to\left\{1,2,\cdots,P\right\}$ in the corresponding $q$'th feature partition, where $P$ is the number of partitions for observations. A sub-matrix of the training set from the partition $x^{p,q}$  corresponding to block $\omega_{q,\pi_q(p)}$ is denoted by $x^{p,q,\pi_q(p)}$, see Figure \hyperref[fig:1]{1}. Each processor operates on random observations from partition $x^{p,q}$ and all feature in $x^{p,q,\pi_q(p)}$. Note that $\omega_{[q]}=\left(\omega_{q,\pi_q(p)}\right)_{p=1}^P$. Let us define $n=N/P$, $m=d/Q$, $\tilde{m}=d/QP$, and $j_{q,\pi_q(p)}$ be randomly chosen from $\left\{ 1,2,\cdots,n \right\}$ associated with sub-block $x^{p,q,\pi_q(p)}$. In Figure \hyperref[fig:1]{1}, where $P=3$ and $Q=4$, $x_{j_{12}}^{3,1,\pi_1(3)=2}\in\mathrm{R}^{\tilde{m}}$ represents a random observation $j_{12}$ with a subset of features selected from the 2nd sub-block of the block corresponding to the observation partition $3$ and feature partition $1$ given $\pi_1(3)=2$. Similarly, $x_{j_{23}}^{4,2,\pi_2(4)=3}$ symbolizes a random observation $j_{23}$ with a subset of features selected from the 3rd sub-block of the block $x^{4,2}$.
\begin{figure}
\centering
\resizebox{\textwidth}{!}{
\begin{tabular}{rllllllllllll}
\cline{2-13}
&\multicolumn{1}{|l:}{$\omega_{11}$} & \multicolumn{1}{l:}{$\omega_{12}$} & \multicolumn{1}{l:}{$\omega_{13}$ } & \multicolumn{1}{l|}{$\omega_{14}$} & \multicolumn{1}{l:}{$\omega_{21}$} & \multicolumn{1}{l:}{$\omega_{22}$ } & \multicolumn{1}{l:}{$\omega_{23}$} &\multicolumn{1}{l|}{$\omega_{24}$} & \multicolumn{1}{l:}{$\omega_{31}$}& \multicolumn{1}{l:}{$\omega_{32}$} & \multicolumn{1}{l:}{$\omega_{33}$} & \multicolumn{1}{l|}{$\omega_{34}$} 
\\ \cline{2-13}
&\multicolumn{4}{c}{$\underbrace{\quad\quad\quad\quad\quad\quad\quad\quad\quad\quad\quad\quad\quad\quad\quad\quad\quad\quad\quad\quad\quad}_{\omega_{[1]}}$}&\multicolumn{4}{c}{$\underbrace{\quad\quad\quad\quad\quad\quad\quad\quad\quad\quad\quad}_{\omega_{2}}$}&\multicolumn{4}{c}{$\underbrace{\quad\quad\quad\quad\quad\quad\quad\quad\quad\quad\quad\quad}_{\omega_{3}}$}\\
& && & & &&&&&&&                      \\ 
\cline{2-13}
\multicolumn{1}{r}{$p=1\left \{  \right.$}&\multicolumn{1}{|l:}{$x_{j_{11}}^{1,1,\pi_1(1)=1}$} & \multicolumn{1}{l:}{} & \multicolumn{1}{l:}{ } & \multicolumn{1}{l|}{} & \multicolumn{1}{l:}{} & \multicolumn{1}{l:}{ } & \multicolumn{1}{l:}{} &\multicolumn{1}{l|}{} & \multicolumn{1}{l:}{}& \multicolumn{1}{l:}{} & \multicolumn{1}{l:}{} & \multicolumn{1}{l|}{} \\\cline{2-13}
\multicolumn{1}{r}{$p=2\left \{  \right.$}&\multicolumn{1}{|l:}{} & \multicolumn{1}{l:}{} & \multicolumn{1}{l:}{ } & \multicolumn{1}{l|}{$x_{j_{14}}^{2,1,\pi_1(2)=4}$} & \multicolumn{1}{l:}{} & \multicolumn{1}{l:}{ } & \multicolumn{1}{l:}{} &\multicolumn{1}{l|}{} & \multicolumn{1}{l:}{}& \multicolumn{1}{l:}{} & \multicolumn{1}{l:}{} & \multicolumn{1}{l|}{$x_{j_{34}}^{2,3,\pi_3(2)=4}$} \\ \cline{2-13}
\multicolumn{1}{r}{$p=3\left \{  \right.$}&\multicolumn{1}{|l:}{} & \multicolumn{1}{l:}{$x_{j_{12}}^{3,1,\pi_1(3)=2}$} & \multicolumn{1}{l:}{ } & \multicolumn{1}{l|}{} & \multicolumn{1}{l:}{} & \multicolumn{1}{l:}{ } & \multicolumn{1}{l:}{} &\multicolumn{1}{l|}{} & \multicolumn{1}{l:}{}& \multicolumn{1}{l:}{} & \multicolumn{1}{l:}{} & \multicolumn{1}{l|}{} \\ \cline{2-13}
\multicolumn{1}{r}{$p=4\left \{  \right.$}&\multicolumn{1}{|l:}{} & \multicolumn{1}{l:}{} & \multicolumn{1}{l:}{$x_{j_{13}}^{4,1,\pi_1(4)=3}$ } & \multicolumn{1}{l|}{} & \multicolumn{1}{l:}{} & \multicolumn{1}{l:}{ } & \multicolumn{1}{l:}{$x_{j_{23}}^{4,2,\pi_2(4)=3}$} &\multicolumn{1}{l|}{} & \multicolumn{1}{l:}{}& \multicolumn{1}{l:}{} & \multicolumn{1}{l:}{} & \multicolumn{1}{l|}{} \\ \cline{2-13}
&\multicolumn{4}{c}{$\underbrace{\quad\quad\quad\quad\quad\quad\quad\quad\quad\quad\quad\quad\quad\quad\quad\quad\quad\quad\quad\quad\quad}_{q=1}$}&\multicolumn{4}{c}{$\underbrace{\quad\quad\quad\quad\quad\quad\quad\quad\quad\quad\quad}_{q=2}$}&\multicolumn{4}{c}{$\underbrace{\quad\quad\quad\quad\quad\quad\quad\quad\quad\quad\quad\quad}_{q=3}$}\\
\end{tabular}}
\label{fig:1}
\caption{$Q=3$, $P=4$}
\end{figure}

Next, we introduce notation for the partial gradient. For any $\mathcal{C}\subseteq \left\{1,\cdots,d\right\}$ and any $j$, let us denote $\bar{\triangledown}_{w_{\mathcal{C}}}f_j(\cdot)\in \mathbb{R}^{d}$ as the vector defined by
\begin{align*}
\left(\bar{\triangledown}_{\omega_{\mathcal{C}}}f_j\left(\cdot\right)\right)_k=
\left\{\begin{matrix}
0,\quad &k \notin\mathcal{C}\\ 
\left(\triangledown f_j\left(\cdot\right)\right)_k,\quad &k\in\mathcal{C}. 
\end{matrix}\right.
\end{align*} 
We need this notation since we sample gradient components. The loss function using this notation becomes
\begin{align*}
F(\omega)=\frac{1}{N}\sum_{k=1}^P\sum_{j=1}^n f_j^k\left( \sum_{q=1}^Q\sum_{p=1}^P x_j^{k,q,\pi_q(p)}\omega_{q,\pi_q(p)}\right),
\end{align*}
where $f_j^k$ is $\bar{f}$ associated with observation $j$ in observation partition $k$.

Given the fact that the data is doubly distributed, SODDA further divides the features $x^{\cdot,q}$ into $P$ subsets along all observations, i.e. $x^{\cdot,q}=[x^{\cdot,q,1},\cdots,x^{\cdot,q,P}]$. In each iteration, SODDA first computes an approximation of the full gradient at the current parameter vector. Then, after randomly choosing a sub-matrix from each matrix $x^{q,p}$ as long as there is no overlap with respect to $\omega$, each processor is assigned a sub-matrix of the local dataset and updates its local parameter by employing generalized SVRG. In the end of each iteration, SODDA concatenates all partial parameters which becomes the incumbent parameter vector for the next iteration.

The entire algorithm is exhibited in Algorithm \ref{alg:1}. Steps~\ref{init1}-~\ref{var_init} initiate all the parameters. Steps~\ref{b_t}-\ref{d_t} give the subsets of features and observations used to compute the partial gradient of the current iterate $\tilde{\omega}$ in step~\ref{grad}. Since the dataset is doubly distributed, the algorithm computes an estimate of the exact full gradient so as to reduce the communication cost in step~\ref{grad}. Additionally, we use the term no feature sampling to address the case when $\mathcal{b}^t = d$; in other words, the whole feature vector is employed to compute the gradient $\mu^t$. To this end, we have three random components. The first one is the common one to sample observations. The second one is to compute only a random subset of subgradient coordinates, and the third one is to evaluate these components not at the exact $x\omega$ but only on the subset of the underlying inner product summation terms. Step~\ref{11} determines how sub-blocks are selected in each block of the dataset associated with $\omega_{[q]}$, for each $q$. The definition of $\left(\pi_q\right)_{q=1}^{Q}$ guarantees that one, and only one sub-block is selected with respect to $\omega_{q,\pi_q(p)}$, i.e. $x^{p,q,\pi_q(p)}$. Then, for each sub-block $x^{p,q,\pi_q(p)}$, after randomly picking an observation $x^{p,q,\pi_q(p)}_{j_{q,\pi_q(p)}}$ in the selected sub-block in step~\ref{15}, each block of parameter updates is given in step~\ref{update}. Instead of using the full vector, we estimate the stochastic gradient by using local features and narrow down the variance by involving the approximated full gradient. Finally, at the end of each iteration, after each processor finishes its own task, step~\ref{agg} aggregates all the updated partial solutions, i.e. $\omega=\left [ \omega_{[1]},\omega_{[2]},\cdots,\omega_{[Q]} \right ]$, where partial parameters $\omega_{[q]}$ represent the concatenation of the local parameters $\omega_{q,\pi_q(p)}$, for $p=1,2,\cdots,P$. 

\begin{algorithm}[H]
  \caption{SODDA}
  \scriptsize
  \label{alg:1}
  \begin{algorithmic}[1]
    \Inputs{batch size B, learning rate $\gamma_t$, sequence $\left\{\mathcal{b}^t,\mathcal{c}^t,\mathcal{d}^t \right\}_{t=0}^{\infty}$ where $\mathcal{c}^t\leq\mathcal{b}^t\leq d$, $\mathcal{d}^t\leq N$ for every $t$ }\label{init1}
    \Data{$x^{p,q,k}\in \mathrm{R}^{n\times \tilde{m}}$ for $p,k=1,\cdots,P$ and $q=1,\cdots,Q$, $\omega_{q,\pi_q(p)}\in \mathrm{R}^{\tilde{m}}$}\label{data_init}
    \Initialize{$w^0 \gets 0$}\label{var_init}
    \For{t = $0,1,2,\cdots$ }
      \State $\mathcal{B}^t$=$\mathcal{b}^t$ elements uniformly at random sampled without replacement from all features\label{b_t}
      \State $\mathcal{C}^t$=$\mathcal{c}^t$ elements uniformly at random sampled without replacement from $\mathcal{B}^t$\;\label{c_t}
      \State $\mathcal{D}^t$=$\mathcal{d}^t$ elements uniformly at random sampled without replacement from all observations\label{d_t}
      \State $\mu^t=\frac{1}{\mathcal{d}^t}\sum_{j\in \mathcal{D}^t}\bar{\triangledown}_{\omega_{\mathcal{C}^t}} f_j(x_{j}^{\mathcal{B}^t}\omega^t_{\mathcal{B}^t})$\label{grad}
        \For{$q=1,2,\cdots,Q$,} \label{inner_begin}
        \State select function $\left(\pi_q\right)_{q=1}^Q$\label{11}
        \EndFor 
          \State \textbf{end for}
        \For{$p=1,2,\cdots,P$ and $q=1,2,\cdots,Q$,} \textbf{in parallel}\label{inner_start}
        \State $\bar{\omega}^{(0)}_{q,\pi_q(p)}=\omega^t_{q,\pi_q(p)}$\label{transfer}
          \For{$i=0,\cdots,B-1$}\label{for}
          \State randomly pick $j_{q,\pi_q(p)}\in \left\lbrace 1,\cdots,n\right\rbrace$ \label{15}
          \State  $\bar{\omega}^{(i+1)}_{q,\pi_q(p)}=\bar{\omega}^{(i)}_{q,\pi_q(p)}-\gamma_{t+1} \left [ \triangledown_{\omega_{q,\pi_q(p)}} f_{j_{q,\pi_q(p)}}^{\pi_q(p)}(x_{j_{q,\pi_q(p)}}^{p,q,\pi_q(p)}\bar{\omega}^{(i)}_{q,\pi_q(p)})-\triangledown_{\omega_{q,\pi_q(p)}} f_{j_{q,\pi_q(p)}}^{\pi_q(p)}(x_{j_{q,\pi_q(p)}}^{p,q,\pi_q(p)}w^t_{q,\pi_q(p)})+\mu^t_{q,\pi_q(p)} \right ]$
          \label{update}  
          \EndFor 
          \State \textbf{end for}\label{end}
      \EndFor
      \State \textbf{end for}\label{inner_end}
      \State $\omega^{t+1}=$ $\left[ \omega_{[1]},\omega_{[2]},\cdots,\omega_{[Q]} \right]$, where $\omega_{[q]}=\left [ \bar{\omega}_{q1}^{(B)},\bar{\omega}_{q2}^{(B)},\cdots,\bar{\omega}_{qP}^{(B)} \right ]$\label{agg}
      \EndFor \State \textbf{end for}
  \end{algorithmic}
\end{algorithm}
\section{Analysis}
In this section, we prove that the sequence of the loss function values $F(\omega^t)$ generated by SODDA approaches the optimal loss function value $F(\omega^*)$. We assume the existence and the uniqueness of the minimizer $\omega^*$ that achieves the optimal loss function value. Meanwhile, we require the following standard assumptions.
\subsubsection*{Assumption 1:}\label{as:1}
\begin{itemize}
\item Functions $f_i(x_i\omega)$ are differentiable with respect to $\omega$ for every $i=1,\cdots,N$.
\item For every $i=1,\cdots,N$, the norm of the gradient $\triangledown f_i(x_i\omega)$ is bounded for all $\omega$, more precisely, there exists a constant $M_1$, such that, for any $\omega$, 
\begin{align*}
    \left\|\triangledown f_i(x_i\omega)\right\|\leq M_1.
\end{align*}
\end{itemize}
\subsubsection*{Assumption 2:}\label{as:2}
\begin{itemize}
\item The expectation function $ F(\omega)$ is strongly convex with parameter $\xi>0$.
\end{itemize}
\subsubsection*{Assumption 3:}\label{as:3}
\begin{itemize}
\item The loss gradients $\triangledown f_i(x_i\omega )$ are Lipschitz continuous with respect to the Euclidian norm with parameter $L\geq 1$, i.e., for all $\omega$, $\hat{\omega}\in\mathbb{R}^{d}$ and any $i$, it holds
\begin{align*}
\left\| \triangledown f_i(x_i\omega)-\triangledown f_i(x_i\hat{\omega}) \right\|\leq L\left\| \omega-\hat{\omega} \right\|.
\end{align*}
\end{itemize}
\subsubsection*{Assumption 4:}
\label{as:4}
\begin{itemize}
\item The sample variance of the norms of the gradients is bounded by $G^2$ for all $\omega^t$, i.e.
\begin{align*}
\frac{1}{N-1}\sum_{j=1}^N\left(\left\|\triangledown f_j(x_jw^t)\right\|^2-\left\|\triangledown F(\omega^t) \right\|^2  \right)\leq G^2.
\end{align*}
\end{itemize}
The restriction imposed by Assumption \hyperref[as:1]{1} provides an upper bound to the first gradient of each data point, which is a standard condition in stochastic approximation literature \cite{robbins1951stochastic}. Its intent is to limit the variance of the stochastic gradients \cite{nemirovski2009robust}. In Assumption \hyperref[as:2]{2}, only the expected loss function $F(\omega)$ is enforced to be strongly convex, whereas the individual loss functions $f_i$ could even be non-convex. Notice that in Assumption \hyperref[as:3]{3}, since each individual function $\triangledown f_i$ is imposed to be Lipschitz-continuous with constant $L$ with respect to $\omega$, both the gradient of the expected loss function $\triangledown F(\omega)$ and the individual function $\triangledown f_i$ are $L$-Lipschitz continuous with respect to $\omega$ and $\omega_{q,\pi_q(p)}$ for any $q\in \left\{1,2,\cdots,Q\right\}$ and any $p\in \left\{1,2,\cdots,P\right\}$. Note that if $f_i$'s are $\epsilon$-Lipschitz continuous for some $0<\epsilon\leq 1$, we can take $L=1$. Assumption \hyperref[as:4]{4} is also standard, see e.g. \cite{harikandeh2015stopwasting}. Moreover, these assumptions hold for several widely used machine learning loss functions, i.e. hinge, square, logistic loss. 

Under these standard assumptions, by finding a relationship for the sequence of the loss function errors $F(\omega^t)-F(\omega^*)$ and employing the supermartingale convergence argument, which is a standard technique for analyzing stochastic optimization problems (see e.g. textbooks  \cite{benveniste2012adaptive},  \cite{bertsekas1989parallel},   \cite{borkar2008stochastic}), we prove that the sequence of the loss function values $F(\omega^t)$ converges to the optimal function value $F(\omega^*)$ almost surely when using the standard diminishing learning rate, i.e. non-summable and squared summable. Consequently, the sequence of $\omega^t$ enjoys the almost sure convergence to $\omega^*$ when taking Assumption \hyperref[as:2]{2} into consideration.
\begin{theorem}
\label{thm:5}
If there is no feature sampling in step~\ref{b_t} and Assumptions \hyperref[as:1]{1}-\hyperref[as:4]{4} hold, and the sequence of learning rates are non-summable $
\sum_{t=1}^{\infty}\gamma_t=\infty$ and square summable $\sum_{t=1}^{\infty}\gamma_t^2<\infty$, and the sequence $(\mathcal{c}^t,\mathcal{d}^t)_{t=0}^{\infty}$ is selected so that $\mathcal{c}^t\leq d$ and $\mathcal{d}^t\leq N$, then the sequence of parameters $\omega^t$ generated by SODDA converges almost surely to the optimal solution $\omega^*$, that is
\begin{align}
\lim\limits_{t\to\infty} \left\|\omega^t-\omega^*\right\|=0\quad\mathrm{a.s.}
\label{eq:thm5-0}
\end{align}
\end{theorem}
\begin{proof}
See Appendix \hyperref[proof:thm5]{C}.
\end{proof}

Based on the fact that the exact form for the update step in expectation is not  available in SODDA, i.e., we do not explicitly know $\omega^{t+1}-\omega^t $
any technique that relies on such an explicit formula is inappropriate. This expectation is given  by steps~\ref{inner_begin}-\ref{inner_end} in the algorithm. The main challenges are coming from evaluating individual function gradient $\triangledown f(x_{qp}\omega_{q,\pi_q(p)})$ and grouping different sub-blocks all together. The way we deal with it, which is borrowed from \cite{bertsekas2000gradient}, is grouping the first two partial gradients together and treating the last partial gradient $\mu_{q,\pi_q(p)}^t$ as a corrector.

The very technical proof follows the following steps. In steps~\ref{for}-\ref{end}, SODDA utilizes local features from a random observation to update the corresponding subset of parameters, and involves the information from the estimated full gradient to reduce  unreasonable fluctuation. Therefore, in association with the conditional Jensen's inequality and properties of Lipschitz continuity, the norm of the difference and the square norm of the difference of the first two terms in the  bracket of the updating procedure in step~\ref{update} are able to be bounded by a function involving $\gamma_t$. Thus, representing $\omega^{t+1}$ as a function of $\omega^t$ and applying strong convexity of $F$, coupled with all the bounds derived before, yield a supermartingale relationship for the sequence of loss function errors $F(\omega^t)-F(\omega^*)$. Combined with the property that $\gamma_t$ is non-summable but square summable, (\ref{eq:thm1-0}) is achieved by applying the supermartingale
convergence theorem.  

Theorem \ref{thm:5} asserts the almost sure convergence of the iterates generated by SODDA with  non-summable and squared summable learning rate. Furthermore, given $\gamma_t=\frac{1}{t}$ and $B$ big enough, the following theorem states that the loss function $F(\omega^t)$ converges to the optimal value $F(\omega^*)$ with probability 1 and the rate of convergence in expectation is at least in the order of $\mathcal{O}(\frac{1}{t})$.

\begin{theorem}
\label{thm:6}
Under Assumptions \hyperref[as:1]{1}-\hyperref[as:4]{4} with no feature sampling, if the learning rate is defined as $\gamma_t:=\frac{1}{t}$ for $t=1,2,\cdots$, and the batch size is chosen such that $B\geq \frac{d}{2 \xi}$, and the sequence $(\mathcal{c}^t,\mathcal{d}^t)_{t=0}^{\infty}$ satisfies $\mathcal{c}^t\leq d$ and $\mathcal{d}^t\leq N$, then there exists a positive constant $C_1$ such that the expected loss function errors $\expect[\big]{F(\omega^t)-F(\omega^*)}$ of SODDA converge to 0 at least with a sublinear convergence rate of order $\mathcal{O}(1/t)$, i.e.
\begin{align}
\expect[\big]{F(\omega^t)-F(\omega^*)}\leq \frac{Q}{1+t},
\label{eq:thm6-0}
\end{align} 
where constant $Q$ is defined as
\begin{align}
Q=\max\left\{F(\omega^0)-F(\omega^*),
\cdots, \left([\lambda]+2\right) \expect[\big]{F(\omega^{[\lambda]+1})-F(\omega^*)}, \frac{C_1}{\lambda-1} \right\},
\label{eq:thm6-1}
\end{align}
with $\lambda=\frac{2\xi B}{d}$.
\end{theorem}
\begin{proof}
See Appendix \hyperref[proof:thm6]{C}
\end{proof}

Given the specific relationship between the learning rate $\gamma_t$ and the iterator $t$, i.e. $\gamma_t=1/t$, applying the supermartingale convergence theorem and performing induction on an upper bound of $\E\left[F(\omega^t)-F(\omega^*) \right]$ allow us to establish at least sublinear convergence of SODDA. 
  
A diminishing learning rate is beneficial if the exact convergence is required. If we are only interested in a specific accuracy, it is more efficient to choose a constant learning rate. In the following theorem, we employ a similar argument used in proving Theorem \ref{thm:5} and Theorem \ref{thm:6} except that $B$ and $\gamma$ are linked by a condition. Again by providing a supermartingale relationship for the sequence of the loss function errors $F(\omega^t)-F(\omega^*)$, we are able to study the convergence properties generated by SODDA for a constant learning rate $\gamma$.

\begin{theorem}
\label{thm:7}
If there is no feature sampling in step~\ref{b_t} and Assumptions \hyperref[as:1]{1}-\hyperref[as:4]{4} hold true, and the learning rate is constant $\gamma_{t}=\gamma $ such that $BL\gamma QP\leq 1$, which also implies that $\gamma\leq 1$, and the sequence $(\mathcal{c}^t,\mathcal{d}^t)_{t=0}^{\infty}$ satisfies $\mathcal{c}^t\leq d$ and $\mathcal{d}^t\leq N$, then there exists a positive constant $C_2$ such that the sequence of parameters $\omega^t$ generated by SODDA converges almost surely to a neighborhood of the optimal solution $\omega^*$, that is
\begin{align}
\liminf_{t\to\infty} F(\omega^t)-F(\omega^*)\leq \frac{C_2dB^3\gamma}{2\xi}\quad\mathrm{a.s.}
\label{eq:thm7-0}
\end{align}
Moreover, if the constant learning rate $\gamma$ is chosen such that $\gamma <\min\left\{ \frac{d}{2\xi B},\frac{1}{BLQP},1\right\}$, then the expected loss function errors $\E\left[F(\omega^t)-F(\omega^*)\right]$ converge linearly to an error bound as
\begin{align}
\label{eq:thm7-1}
\E\left[F(\omega^t)-F(\omega^*)\right]\leq \left(1-\frac{2\xi B}{d}\gamma \right)^t\left(F(\omega^0)-F(\omega^*)\right)+\frac{C_2dB^3\gamma}{2\xi}.
\end{align}
\end{theorem}
\begin{proof}
See Appendix \hyperref[proof:thm7]{D}.
\end{proof}

Note that $BL\gamma QP\leq 1$ trades off $\gamma$ and $B$, i.e. the larger $B$ is, the smaller $\gamma$ must be. The major difficulties are similar but not identical to those of Theorem \ref{thm:5}. We apply a similar idea but treating both $B$ and $\gamma_t=\gamma$ as variables to obtain an upper bound in closed form for the difference of the first two partial gradients in step~\ref{update}. Then Lipschitz continuity of $\triangledown F(\omega)$ leads to a supermartingale relationship for the sequence of the loss function errors $F(\omega^t) - F(\omega^*)$. As a consequence, claims in (\ref{eq:thm3-0}) and (\ref{eq:thm3-1}) follow according to the supermartingale convergence theorem. The only distinction between Theorem \ref{thm:5} and Theorem \ref{thm:7} is caused by the property of the learning rate. The error exists in each iteration, which is a function of the learning rate $\gamma_t$, however, in Theorem \ref{thm:1}, the error function goes to $0$ as the number of iterations increases, which is not the case when the learning rate is a constant. Therefore, we can only ensure a relatively high-quality solution.

In order to allow feature sampling we have to control the growth of $\omega^t$.  
\subsection{Analyses with Feature Sampling}
To this end, we require the following assumption together with Assumptions \hyperref[as:2]{2}-\hyperref[as:4]{4}. In this subsection, we also do not require $\mathcal{b}^t = d$, i.e., step~\ref{b_t} in Algorithm \ref{alg:1} now requires sampling.
\subsubsection*{Assumption 5:}
\label{as:5}
\begin{itemize}
\item There exists a constant $M_2$, such that
\begin{align*}
    \left\|\omega^t\right\|\leq \frac{M_2}{2},
\end{align*}
for any $t$.
\end{itemize}
The restriction in Assumption \hyperref[as:5]{5} is reasonable and also has been used inwork\cite{harikandeh2015stopwasting}. Notice that without Assumptions \hyperref[as:1]{1} and \hyperref[as:5]{5}, we can not further assume the boundness of the sample variance of the norms of the gradients in Assumption \hyperref[as:4]{4}. Next, we present an example showing that SODDA does not converge under Assumptions \hyperref[as:2]{2} and \hyperref[as:3]{3}.

\begin{theorem}
\label{thm:8}
There is a convex loss function and $\mathcal{P}$ where SODDA does not converge when only given Assumptions \hyperref[as:2]{2} and \hyperref[as:3]{3}, and any $\gamma_t \leq K$ for every $t$ and constant $K$ depending on input data.
\end{theorem} 
\begin{proof}
See Appendix \hyperref[proof:thm8]{E}
\end{proof}

The main role of Assumption \hyperref[as:5]{5} is to maintain a reasonable error generated by the stochastic partial gradients in steps~\ref{update}. Then, under these standard assumptions and applying the similar trick as in the proof of Theorem \ref{thm:5}, we argue that the sequence of $\omega^t$ enjoys the almost sure convergence to $\omega^*$.

\begin{theorem}
\label{thm:1}
If Assumptions \hyperref[as:2]{2}-\hyperref[as:5]{5} hold true, and the sequence of learning rates are non-summable $
\sum_{t=1}^{\infty}\gamma_t=\infty$ and square summable $\sum_{t=1}^{\infty}\gamma_t^2<\infty$, and the sequences $(\mathcal{b}^t,\mathcal{c}^t,\mathcal{d}^t)_{t=0}^{\infty}$ are selected so that $\mathcal{b}^t\in \left[ \max\left\{\mathcal{c}^t,\frac{d}{1+\frac{4d\eta\gamma_{t+1}^2}{\mathcal{c}^tM_2^2L^2}}\right\},d\right]$ for some constant $\eta\geq 0$, $\mathcal{c}^t\leq d$ and $\mathcal{d}^t\leq N$, then the sequence of parameters $\omega^t$ generated by SODDA converges almost surely to the optimal solution $\omega^*$, that is
\begin{align}
\lim\limits_{t\to\infty} \left\|\omega^t-\omega^*\right\|=0\quad\mathrm{a.s.}
\label{eq:thm1-0}
\end{align}
\end{theorem}
\begin{proof}
See Appendix \hyperref[proof:thm1]{F}.
\end{proof}

The proof of Theorem \ref{thm:1} is very similar to the proof of Theorem \ref{thm:5}. The only difference is caused by feature sampling. The main challenges are how to pick a suitable $\mathcal{b}^t$ and how to narrow down the error generated by $\mathcal{b}^t$. Then, given an appropriate $\mathcal{b}^t$, the norm of the estimator of the full gradient $\mu^t$ and its square are bounded by a function containing the full gradient at $\omega^t$ and the learning rate $\gamma_t$. Meanwhile, $\eta$ is a positive constant which controls the divergence of the approximate full gradient from the exact full gradient in step~\ref{grad}, i.e. when $\eta$ is 0, the whole feature vector is used as $\mathcal{b}^t=d$. The rest of the proof is identical to the proof of Theorem \ref{thm:5}.

Theorem \ref{thm:1} asserts the almost sure convergence of the iterates generated by SODDA with non-summable and squared summable learning rate. Furthermore, given $\gamma_t=\frac{1}{t}$ and $B$ big enough, the following theorem states that the loss function $F(\omega^t)$ converges to the optimal value $F(\omega^*)$ with probability 1 and the rate of convergence in expectation is at least in the order of $\mathcal{O}(\frac{1}{t})$.
\begin{theorem}
\label{thm:2}
Under Assumptions \hyperref[as:2]{2}-\hyperref[as:5]{5}, if the learning rate is defined as $\gamma_t:=\frac{1}{t}$ for $t=1,2,\cdots$, and the batch size is chosen such that $B\geq \frac{d}{2 \xi}$, and the sequence $(\mathcal{b}^t,\mathcal{c}^t,\mathcal{d}^t)_{t=0}^{\infty}$ satisfies the same conditions as in Theorem \ref{thm:1}, then there exists a positive constant $C_3$ such that the expected loss function errors $\expect[\big]{F(\omega^t)-F(\omega^*)}$ of SODDA converges to 0 at least with a sublinear convergence rate of order $\mathcal{O}(1/t)$, i.e.
\begin{align}
\expect[\big]{F(\omega^t)-F(\omega^*)}\leq \frac{Q}{1+t},
\label{eq:thm2-0}
\end{align} 
where constant $Q$ is defined as
\begin{align}
Q=\max\left\{F(\omega^0)-F(\omega^*),
\cdots, \left([\lambda]+2\right) \expect[\big]{F(\omega^{[\lambda]+1})-F(\omega^*)}, \frac{C_3}{\lambda-1} \right\},
\label{eq:thm2-1}
\end{align}
with $\lambda=\frac{2\xi B}{d}$.
\end{theorem}
\begin{proof}
See Appendix \hyperref[proof:thm2]{F}.
\end{proof}
  
A diminishing learning rate is beneficial if the exact convergence is required. If we are only interested in a specific accuracy, it is more efficient to choose a constant learning rate. 
\begin{theorem}
\label{thm:3}
If Assumptions \hyperref[as:2]{2}-\hyperref[as:5]{5} hold true, and the learning rate is constant $\gamma_{t}=\gamma $ such that $BL\gamma QP\leq 1$, which also implies that $\gamma\leq 1$, and the sequence $(\mathcal{b}^t,\mathcal{c}^t,\mathcal{d}^t)_{t=0}^{\infty}$ satisfies the same conditions as in Theorem \ref{thm:1}, then there exists a positive constant $C_4$ such that the sequence of parameters $\omega^t$ generated by SODDA converges almost surely to a neighborhood of the optimal solution $\omega^*$, that is
\begin{align}
\liminf_{t\to\infty} F(\omega^t)-F(\omega^*)\leq \frac{C_4dB^3\gamma}{2\xi}\quad\mathrm{a.s.}
\label{eq:thm3-0}
\end{align}
Moreover, if the constant learning rate $\gamma$ is chosen such that $\gamma <\min\left\{ \frac{d}{2\xi B},\frac{1}{BLQP},1\right\}$, then the expected loss function errors $\E\left[F(\omega^t)-F(\omega^*)\right]$ converges linearly to an error bound as
\begin{align}
\label{eq:thm3-1}
\E\left[F(\omega^t)-F(\omega^*)\right]\leq \left(1-\frac{2\xi B}{d}\gamma \right)^t\left(F(\omega^0)-F(\omega^*)\right)+\frac{C_4dB^3\gamma}{2\xi}.
\end{align}
\end{theorem}
\begin{proof}
See Appendix \hyperref[proof:thm3]{G}.
\end{proof}

Theorem \ref{thm:3} guarantees that SODDA finds good quality solutions when using an appropriate learning rate $\gamma$ and batch size $B$. Notice that although methods of type SVRG achieve linear convergence to the exact solution in expectation under a constant step size, SVRG performs the exact full gradient after every certain number of iterations which would trigger emergence of communication under the doubly distributed setting and is unnecessary especially in early iterations. In addition, based on (\ref{eq:thm3-1}), there is a trade-off between the accuracy and the convergence rate. Although reducing the learning rate $\gamma$ or batch size $B$ narrows down the error bound $\frac{C_4dB^3\gamma}{2\xi}$ and contributes significantly to a more accurate convergence, the constant convergence rate $1-\frac{2\xi B}{d}\gamma$ suffers greatly since it increases and gets closer to 1, which leads to a slower convergence rate.

To address this problem, in the following theorem, by considering not only the loss function errors $F(\omega^t)-F(\omega^*)$ but also the errors $\left\|\omega^t-\omega^* \right\|^2$, we prove that the sequence of the loss function values $F(\omega^t)$ generated by SODDA converges to the optimal value $F(\omega^*)$ for any constant learning rate selected from a certain region. In addition, we are able to further assert that the sequence of $\omega^t$ converges to $\omega^*$ when taking Assumption \hyperref[as:2]{2} into account. Furthermore, since we employ an approximation of the exact full gradient for the sake of the efficiency of the algorithm in Theorem \ref{thm:3}, the algorithm converges only to a neighborhood of an optimal solution under a constant learning rate. In the following theorem, if we are allowed to employ the exact full gradient in expectation, then the algorithm in Theorem \ref{thm:4} converges to the exact solution in expectation under a constant step size.
\begin{theorem}
\label{thm:4}
If Assumptions \hyperref[as:2]{2}-\hyperref[as:5]{5} hold true, and the learning rate $\gamma_t=\gamma$ is a constant such that $\gamma \in (0,\min\left\{1,\frac{1}{BLQP},\gamma_1,\gamma_2 \right\})$, where
both $\gamma_1$ and $\gamma_2$ are positive constants specified in Appendix \hyperref[proof:thm4]{E}, 
and the sequence $(\mathcal{b}^t,\mathcal{c}^t,\mathcal{d}^t)_{t=0}^{\infty}=(d,\mathcal{c}^t,N)_{t=0}^{\infty}$ for arbitrary positive $\mathcal{c}^t\leq d$, then the sequence of parameters $\omega^t$ generated by SODDA converges to $\omega^*$, that is
\begin{align}
\label{eq:thm4-0}
\lim\limits_{t\to\infty}\left\|\omega^t-\omega^*\right\|=0.
\end{align}
\end{theorem}
\begin{proof}
See Appendix \hyperref[proof:thm4]{H}.
\end{proof}
The Lyapunov analysis, which is a common strategy to deal with a constant learning rate (see e.g. \cite{schmidt2017minimizing}), fails for our algorithm due to the analogous reasons as those for Theorem \ref{thm:1}. The success of the Lyapunov analysis heavily relies on the number of negative terms available when computing the loss function errors $F(\omega^t)-F(\omega^*)$ and the errors $\omega^t-\omega^*$. Unfortunately, the doubly distributed data setting results in lack of information in each iteration in step~\ref{update}, which leads to a  scarcity of negative terms to ensure the decrease of the loss function value.

Our steps to study the convergence analysis are as follows. We first establish either exact forms or upper bounds for all terms involving gradients. Then, from the update rule,  we find a criteria for the constant learning rate $\gamma$ so as to make the errors $\omega^{t}-\omega^*$ at least not increase as the number of iterations increases. In addition, given Lipschitz continuity of $\triangledown F(\omega)$, we find a recursive formula regarding the loss function error, which provides another constraint for $\gamma$ such that the loss function error vanishes as $t$ increases. Finally, the convergence of SODDA and the existence of $\gamma$ are guaranteed by two cubic inequality constraints aforementioned. 
\section{Numerical Study}
In this section, we compare the SODDA method with RADiSA-avg \cite{nathan2017optimization}, which is the best known optimization algorithm for solving problem \hyperref[obj]{(1)} with doubly distributed data. All the algorithms are implemented in Scala with Spark 2.0. The experiments are conducted in a Hadoop cluster with 4 nodes, each containing 8 Intel Xeon 2.2GHz cores. We conduct experiments
on three different-size synthetic datasets that are larger than the datasets in  \cite{nathan2017optimization} and two datasets used in \cite{wongchaisuwat2018truth} extracted from SemMed Database. For all of these datasets, we train one of the most popular classification models: binary classification hinge loss support vector machines (SVM), and set the learning rate $\gamma_t=\frac{1}{(1+\sqrt{t-1})}$, which is also employed in \cite{nathan2017optimization}. Furthermore, we set the feature partition number $Q=3$ and observation partition number $P=5$, which is also one of the cases studied in \cite{nathan2017optimization}. We do not compare different learning rates and $Q,P$ since these have been extensively studied in \cite{nathan2017optimization}.
\subsection{SVM with Synthetic data}
We first compare SODDA with RADiSA-avg \cite{nathan2017optimization} using synthetic data. The datasets for these experiments are generated based on a standard procedure introduced in \cite{zhang2012efficient}, which is also  used in \cite{nathan2017optimization}: the $x_i$'s and $z$ are sampled from the uniform distribution in $[-1,1]$, and $y_i:=\mathrm{sgn}(x_iz)$ with probability 0.01 of flipping the sign. In addition, all the data is in the dense format and the features are standardized to have unit variance. The size of each partition from the small-size dataset is $50,000\times 6,000$, the one from the mid-size data is $60,000\times 7,000$ and the one from the large-size data is $60,000\times 9,000$. The information about these three datasets is listed in Table \hyperref[table:1]{1}. 

\begin{table}[H]
\centering
\begin{tabular}{|l|l|l|l|}
\hline
data size & small  & medium &  large \\ \hline
 $P\times Q$ & $5\times 3$ &$5\times 3$  &$5\times 3$  \\ \hline 
 size of each partition & $50,000\times 6,000$  & $60,000\times 7,000$ & $60,000\times 9,000$ \\ \hline
Number of Spark executors used &18 &25 &25 \\ \hline
\end{tabular}
\caption{Synthetic datasets for numerical experiments }
\label{table:1}
\end{table}

First, we conduct $\mathcal{b}^t,\mathcal{c}^t,\mathcal{d}^t$ subsequence related experiments. We justify the value of $(\mathcal{b}^t,\mathcal{c}^t,\mathcal{d}^t)$ from the small-size dataset, since the other two datasets would take more computational time. We study the impact of $(\mathcal{b}^t,\mathcal{c}^t,\mathcal{d}^t)$ to the performance of SODDA by varying one of the three parameters $(\mathcal{b}^t,\mathcal{c}^t,\mathcal{d}^t)$ while keeping the other two parameters fixed. 

The most important results are presented in Figure \hyperref[fig:2]{2}. In Figure \hyperref[fig1]{2(a)}, we study the cases where the number of total observations used to estimate the full gradient in step~\ref{grad} varies from $60\%$ to $90\%$ with $\mathcal{b}^t=\mathcal{c}^t=100\%$. In Figure \hyperref[fig2]{2(b)}, we consider the cases when $\mathcal{c}^t$ varies from $40\%$ to $80\%$ given that every feature is involved to compute the approximated full gradient, i.e. $\mathcal{b}^t=100\%$. Figure \hyperref[fig3]{2(c)} represents the cases where only partial features are used in step~\ref{grad} but everything available is fully used, i.e. $\mathcal{b}^t=\mathcal{c}^t$. In Figures \hyperref[fig4]{2(d)}-\hyperref[fig6]{(f)}, we study three different $\mathcal{b}^t$ choices and for each one we vary $\mathcal{c}^t$. Figure \hyperref[fig7]{2(g)} is an extension of Figure \hyperref[fig4]{2(d)} showing the long-time performance under the corresponding set of parameters. 

In these plots, we observe that every set of parameters $(\mathcal{b}^t,\mathcal{c}^t,\mathcal{d}^t)$ with the small-size dataset outperforms RADiSA-avg in early iterations, however, the 
benefits peak at certain points. More precisely, from Figure \hyperref[fig1]{2(a)}, we discover that the marginal benefit grows up dramatically when $\mathcal{d}^t$ increases from $60\%$ to $80\%$ and slows down from $80\%$ to $90\%$, thus, $\mathcal{d}^t=85\%$ seems to be most beneficial. When it comes to $\mathcal{c}^t$, we observe that although the value of $\mathcal{c}^t$ does not influence the accuracy of the solution, a higher value of $\mathcal{c}^t$ leads to a faster convergence speed to a good quality solution in Figure \hyperref[fig2]{2(b)}. Thus, we set $\mathcal{c}^t=80\%$ as a good value. From Figures \hyperref[fig3]{2(c)}-\hyperref[fig7]{(g)}, we observe that the value of $\mathcal{b}^t$ affects the accuracy of the solution significantly, therefore, we set $\mathcal{b}^t=85\%$ after taking both the accuracy of the solution and the computational time into consideration.          
 
In these figures, we observe that SODDA always outperforms RADiSA-avg in early iterations on the small-size dataset, and there is a trade-off between the accuracy of the loss function value and the sampling sizes used in the algorithm. More precisely, using less data leads to a faster convergence speed but a less accurate solution, while using more data contributes to a more accurate solution but requires more time. 
\begin{figure}[H]
\minipage{0.33\textwidth}
\includegraphics[width=\linewidth]{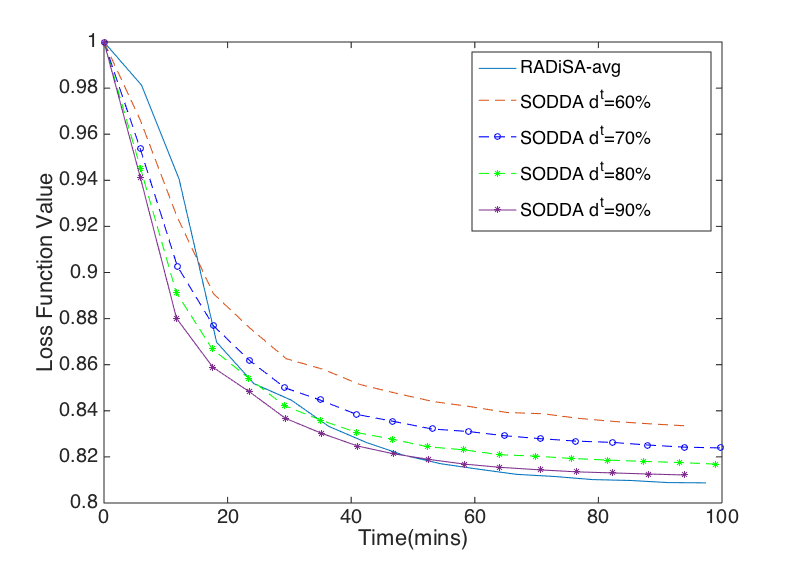}
\caption*{(a)}
  \label{fig1}
\endminipage\hfill
\minipage{0.33\textwidth}
\includegraphics[width=\linewidth]{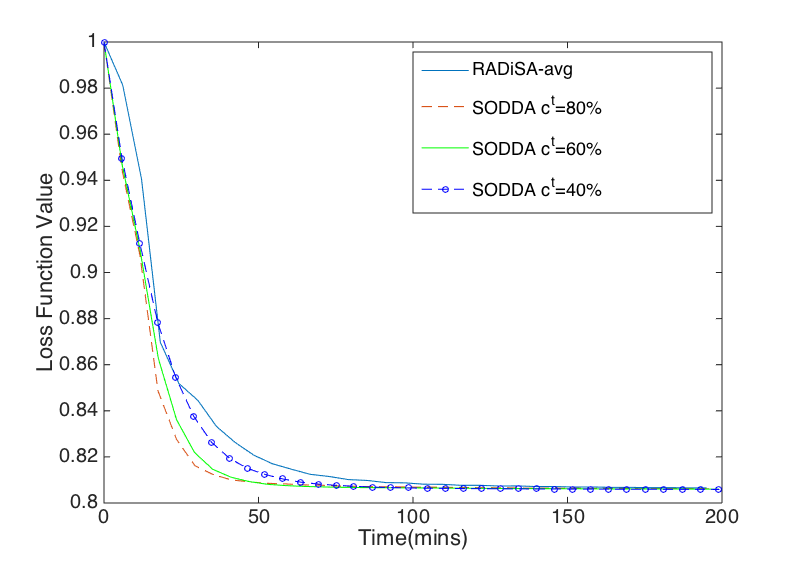}
\caption*{(b)}
\label{fig2}
\endminipage\hfill
\minipage{0.33\textwidth}
\includegraphics[width=\linewidth]{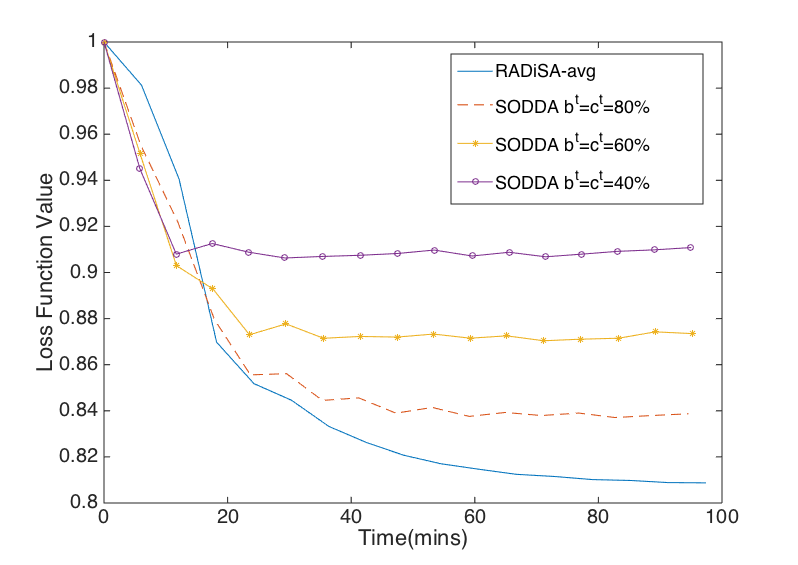}
\caption*{(c)}
 \label{fig3}
\endminipage\hfill
\minipage{0.33\textwidth}
\includegraphics[width=\linewidth]{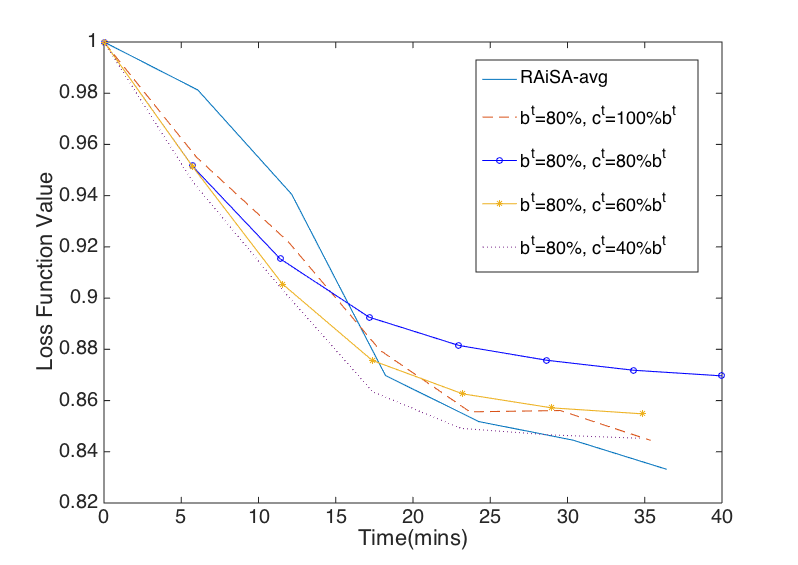}
\caption*{(d)}
\label{fig4}
\endminipage\hfill
\minipage{0.33\textwidth}
\includegraphics[width=\linewidth]{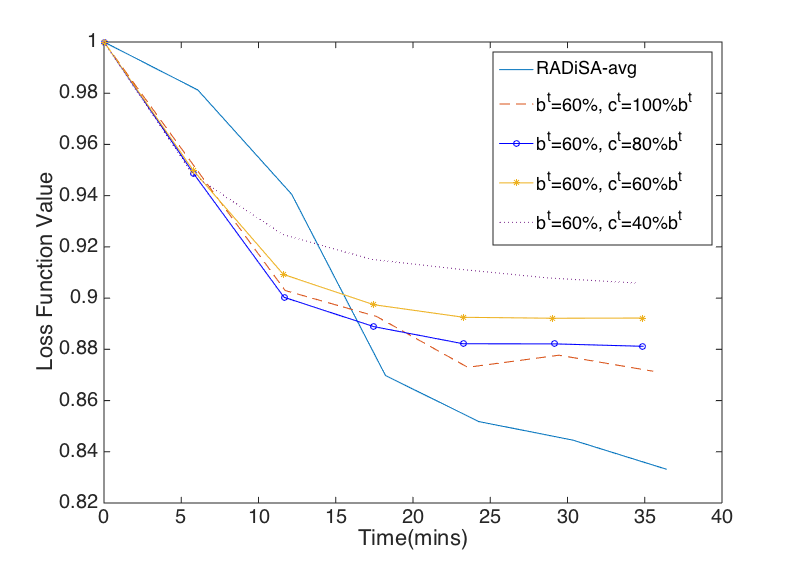}
\caption*{(e)}
  \label{fig5}
\endminipage\hfill
\minipage{0.33\textwidth}
\includegraphics[width=\linewidth]{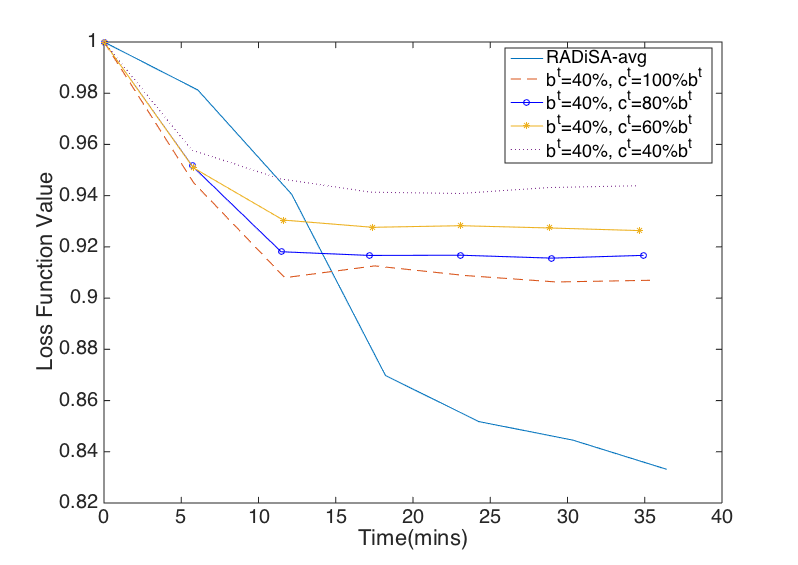}
\caption*{(f)}
\label{fig6}
\endminipage\hfill
\minipage{0.33\textwidth}
\includegraphics[width=\linewidth]{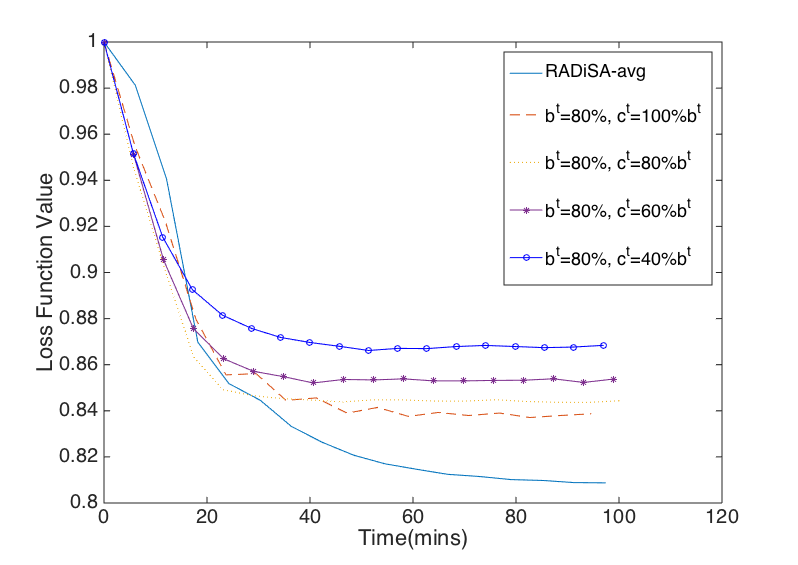}
\caption*{(g)}
\label{fig7}
\endminipage\hfill
\caption{Comparison of SODDA and RADiSA-avg on small-size dataset}
\label{fig:2}
\end{figure}
After specifying the values of $\mathcal{b}^t$, $\mathcal{c}^t$ and $\mathcal{d}^t$ to be $(85\%,80\%,85\%)$, we test both SODDA and RADiSA-avg on the mid- and large-size datasets with three different seeds. The results are presented in Figure \hyperref[fig8]{3}. As we can observe, SODDA always exhibits a stronger and faster convergence than RADiSA-avg. It is interesting that as the size of the dataset increases, the intersection time of SODDA and RADiSA-avg comes later, which gives SODDA more advantages over RADiSA-avg when dealing with large datasets. 
\begin{figure}[h]
\minipage{0.5\textwidth}
\includegraphics[width=\linewidth]{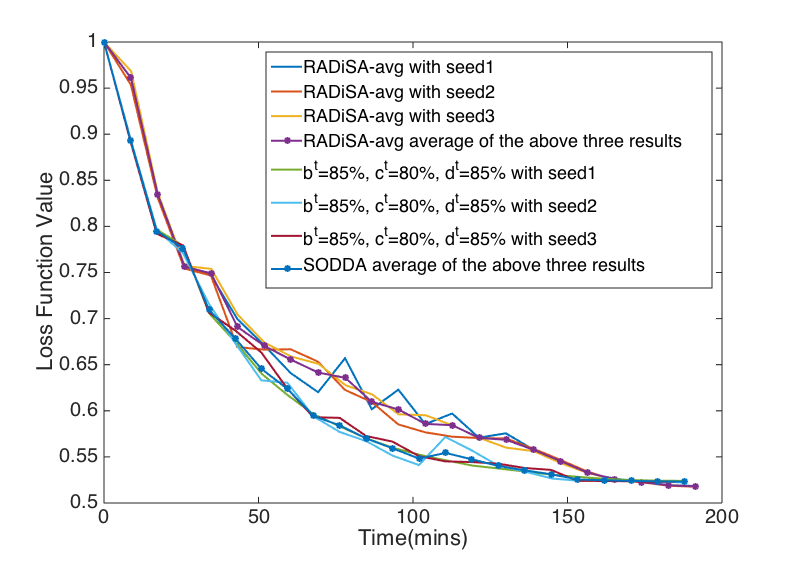}
\caption*{(a)}
  \label{fig8}
\endminipage\hfill
\minipage{0.5\textwidth}
\includegraphics[width=\linewidth]{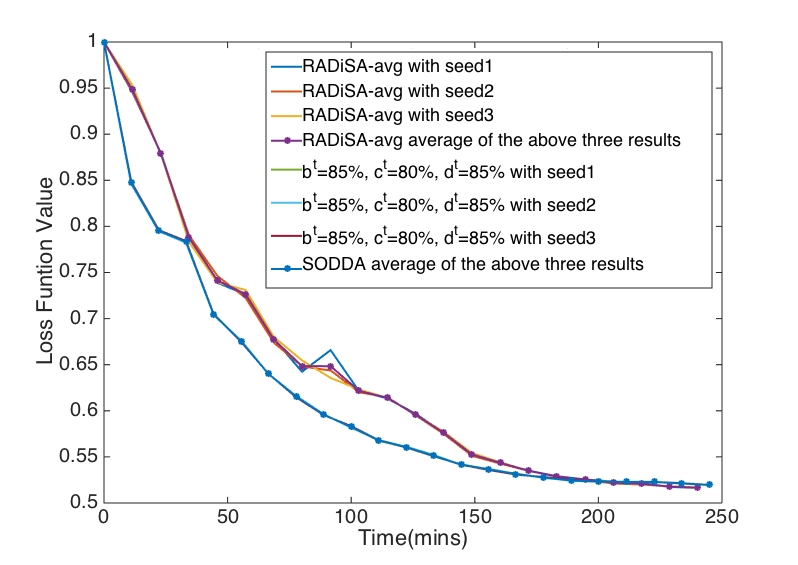}
\caption*{(b)}
\label{fig9}
\endminipage\hfill
\caption{Comparison of SODDA and RADiSA-avg for three different seeds on the mid- and large-size datasets}
\end{figure}

In the SODDA algorithm, we randomly choose a subset of observations and a block of features to estimate the full gradient in step~\ref{grad}. Moreover, both SODDA and RADiSA-avg utilize an observation randomly selected from a randomly chosen sub-matrix in the update step, where SODDA employs a sub-block of the approximated full gradient as a corrector but RADiSA-avg employs the exact full gradient. In order to eliminate the uncertainty about the choice of seeds, we conduct experiments on the large-size dataset under the same set of parameter $(\mathcal{b}^t,\mathcal{c}^t,\mathcal{d}^t)$ with different seeds. Table \hyperref[tab:2]{2} summarizes the influence of the change of the seed on the large-size dataset. For 10 different seeds, we run 40 iterations for each. The first two columns present the average of the difference of the maximum objective value and the average function value across the 10 seeds, and the average of the difference of the average function value and the minimum objective value across the 10 seeds, respectively. Similarly, the remaining terms are defined as the maximum of the difference of the maximum objective value and the average function value, and the maximum of the difference of the average function value and the minimum objective value. As we can see in Table  \hyperref[tab:2]{2}, the perturbation caused by the change of the seed is negligible especially when compared to the objective function value, which is a positive characteristic. Thus, in the remaining experiments, we no longer need to consider the impact of the randomness caused by either SODDA or RADiSA-avg.

\begin{table}[H]
\centering
\label{tab:2}
\begin{tabular}{|l|l|l|l|l|}
\hline
 & avg(max-avg) &avg(avg-min)  &max(max-avg)  &max(avg-min)  \\ \hline
SODDA & $0.4600\times 10^{-4}$ &$0.0251\times 10^{-4}$  & $0.2500\times 10^{-3}$ & $3.0000\times 10^{-3}$ \\ \hline
RADiSA-avg & $1.6373\times 10^{-4}$ &  $1.2606\times 10^{-4}$ & $1.8000\times 10^{-3}$ & $2.3500\times 10^{-3}$ \\ \hline
\end{tabular}
\caption{Variation of SODDA and RADiSA-avg by using different seeds}
\end{table}

\subsection{SVM with SemMed Database}
In the last set of experiments, we study the performances of SODDA with the $\left(\mathcal{b}^t,\mathcal{c}^t,\mathcal{d}^t\right)$ selected in the previous section and RADiSA-avg on the Semantic MEDLINE Database \cite{kilicoglu2012semmeddb} with SemRep, a semantic interpreter of biomedical text \cite{rindflesch2003interaction} as an extraction tool to construct the knowledge graph (KG). Like the preprocessing done in \cite{wongchaisuwat2018truth}, we apply the inference method, which is called the Path Ranking Algorithm (PRA) \cite{lao2010relational}, to KG constructed from SemRep. The model under consideration is still linear SVM, and all the datasets considered are in the sparse format. The first dataset DIAG-neg10 is based on relationship \quotes{DIAGNOSES,} while LOC-neg5 is created in a similar manner based on \quotes{LOCATION OF.} The data is summarized in Table \hyperref[table:3]{3}. 
 
Figure \hyperref[fig10]{4} illustrates the convergence paths of the objective loss function $F(\omega)$ generated by SODDA and RADiSA-avg versus time. We observe that using SODDA is much better than RADiSA with respect to not only the running time but also the loss reduction in early iterations. Comparing Figure \hyperref[fig11]{4(a)} with Figure \hyperref[fig10]{4(b)}, we discover that the superior behavior of RADiSA over RADiSA-avg is more apparent and robust when applied to larger datasets, which is expected since  it is more beneficial for datasets with larger size to perform partial computation instead of full computation of gradients in step~\ref{grad}.

\begin{table}[H]
\centering
\begin{tabular}{|l|r|r|r|}
\hline
Dataset& Observations ($N$) & Features ($d$) & Size of each partition ($n\times m$) \\ \hline
DIAG-neg10& 425,185 & 26,946 & $85,037\times 8,982$  \\ \hline
LOC-neg5& 5,638,696 & 26,966 & $1,127,740\times 8,989$  \\ \hline
\end{tabular}
\caption{Datasets extracted from SemMed database }
\label{table:3}
\end{table}

\begin{figure}[H]
\minipage{0.5\textwidth}
\includegraphics[width=\linewidth]{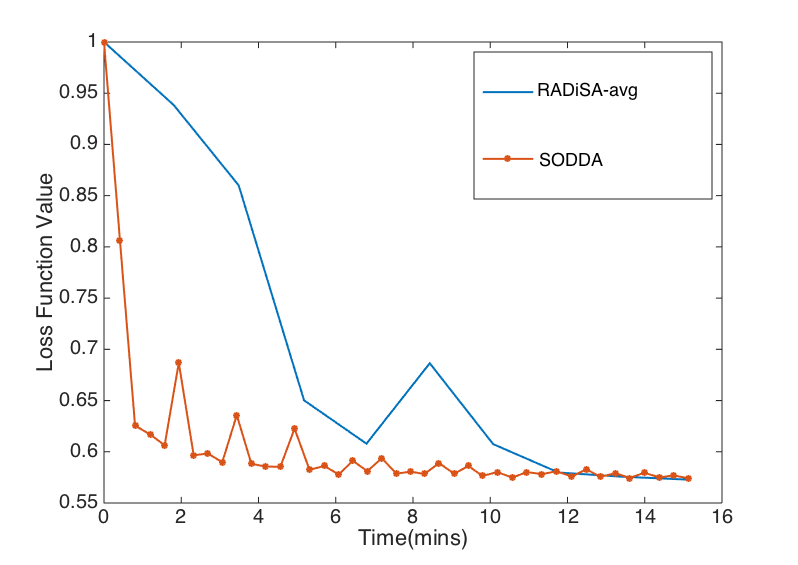}
\caption*{(a) DIAG-neg10}
  \label{fig10}
\endminipage\hfill
\minipage{0.5\textwidth}
\includegraphics[width=\linewidth]{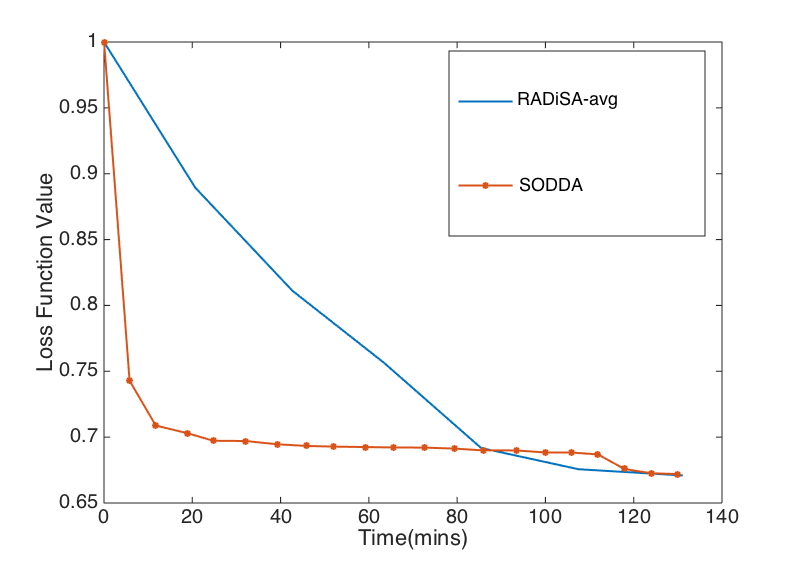}
\caption*{(b) LOC-neg5}
\label{fig11}
\endminipage\hfill
\caption{Comparison of SODDA and RADiSA-avg on SemMed database}
\end{figure}

\subsection{Key Findings}
From the first set of experiments conducted on different synthetic datasets in the dense format, we justify a good set of parameters $(\mathcal{b}^t,\mathcal{c}^t,\mathcal{d}^t)=(85\%,80\%,85\%)$ and eliminate the potential impact of the randomness involved in SODDA and RADiSA to the performance of the convergence. Furthermore, we discover that SODDA always exhibits a stronger and faster convergence than RADiSA-avg for every dataset considered and parameter values chosen. In the second set of experiments, we observe the same dominance of SODDA when compared to RADiSA-avg on sparse datasets.

In conclusion, SODDA provides a faster, stronger and more robust convergence than RADiSA-avg for both dense and sparse datasets.

\bibliographystyle{apa}
\bibliography{references}

\newpage
\section{Appendix}
\subsection*{A\tab Problem Set-up}
We study the optimization problem of minimizing
\begin{align*}
\min_{\omega\in\mathbb{R}^{d}}F(\omega):=\frac{1}{N}\sum_{i=1}^N f_i(x_i\omega)=\frac{1}{N}\sum_{k=1}^P\sum_{j=1}^n f_j^k\left( \sum_{q=1}^Q\sum_{p=1}^P x_j^{p,q,k}\omega_{q,\pi_q(p)}\right),
\end{align*}  
where the features and the observations of the data $\left\{(x_i,y_i)\right\}_{i=1}^N$ are split into $Q$ and $P$ partitions respectively, and each feature partition is further separated into $P$ smaller divisions. We have
\begin{align*}
n&=N/P,\quad m=d/Q,\quad \tilde{m}=d/QP,\\
\omega&=\left( \omega_{11},\omega_{12},\cdots,\omega_{1P},\omega_{21},\cdots,\omega_{2P},\cdots,\omega_{QP} \right).
\end{align*} 
\subsection*{B\tab Notation}
Recall that  in steps~\ref{inner_begin}-~\ref{inner_end}, the inner loop of SODDA performs iterations on each parameter subset $\omega_{q,\pi_q(p)}$ (for $i\geq 0$):
\begin{dmath*}
\tiny
\bar{\omega}^{(i+1)}_{q,\pi_q(p)}=\bar{\omega}^{(i)}_{q,\pi_q(p)}-\gamma_{t+1} \left [ \triangledown_{\omega_{q,\pi_q(p)}} f_{j_{q,\pi_q(p)}}^{p}(x_{j_{q,\pi_q(p)}}^{p,q,\pi_q(p)}\bar{\omega}^{(i)}_{q,\pi_q(p)})-\triangledown_{\omega_{q,\pi_q(p)}} f_{j_{q,\pi_q(p)}}^{p}(x_{j_{q,\pi_q(p)}}^{p,q,\pi_q(p)}\tilde{w}_{q,\pi_q(p)})+\mu^t_{q,\pi_q(p)} \right ], 
\end{dmath*}
where $j_{q,\pi_q(p)}$ is a randomly selected observation in sub-block $x^{p,q,\pi_q(p)}$. It is convenient to use the notation
\begin{align*}
v^{t,i}&=
\begin{bmatrix}
\triangledown_{\omega_{11}} f_{j_{11}}^{\pi_1^{-1}(1)}\left( x_{j_{11}}^{\pi_1^{-1}(1),1,1}\bar{\omega}_{11}^{t,i-1} \right)-\triangledown_{\omega_{11}} f_{j_{11}}^{\pi_1^{-1}(1)}\left( x_{j_{11}}^{\pi_1^{-1}(1),1,1}\tilde{\omega}_{11} \right) \\ 
\triangledown_{\omega_{12}} f_{j_{12}}^{\pi_1^{-1}(2)}\left( x_{j_{12}}^{\pi_1^{-1}(2),1,2}\bar{\omega}_{12}^{t,i-1} \right)-\triangledown_{\omega_{12}} f_{\pi_1^{-1}(2)}^{\pi_1(2)}\left( x_{j_{12}}^{\pi_1^{-1}(2),1,2}\tilde{\omega}_{12} \right) \\ 
\vdots\\ 
\triangledown_{\omega_{1P}} f_{j_{1P}}^{\pi_1^{-1}(P)}\left( x_{j_{1P}}^{\pi_1^{-1}(P),1,P}\bar{\omega}_{1P}^{t,i-1} \right)-\triangledown_{\omega_{1P}} f_{j_{1P}}^{\pi_1^{-1}(P)}\left( x_{j_{1P}}^{\pi_1^{-1}(P),1,P}\tilde{\omega}_{1P} \right)\\ 
\triangledown_{\omega_{21}} f_{j_{21}}^{\pi_2^{-1}(1)}\left( x_{j_{21}}^{\pi_2^{-1}(1),2,1}\bar{\omega}_{21}^{t,i-1} \right)-\triangledown_{\omega_{21}} f_{j_{21}}^{\pi_2^{-1}(1)}\left( x_{j_{21}}^{\pi_2^{-1}(1),2,1}\tilde{\omega}_{21} \right)\\ 
\vdots\\ 
\triangledown_{\omega_{2P}} f_{j_{2P}}^{\pi_2^{-1}(P)}\left( x_{j_{2P}}^{\pi_2^{-1}(P),2,P}\bar{\omega}_{2P}^{t,i-1} \right)-\triangledown_{\omega_{2P}} f_{j_{2P}}^{\pi_2^{-1}(P)}\left( x_{j_{2P}}^{\pi_2^{-1}(P),2,P}\tilde{\omega}_{2P} \right)\\ 
\vdots\\ 
\triangledown_{\omega_{QP}} f_{j_{QP}}^{\pi_Q^{-1}(P)}\left( x_{j_{QP}}^{\pi_Q^{-1}(P),Q,P}\bar{\omega}_{QP}^{t,i-1} \right)-\triangledown_{\omega_{QP}} f_{j_{QP}}^{\pi_Q^{-1}(P)}\left( x_{j_{QP}}^{\pi_Q^{-1}(P),Q,P}\tilde{\omega}_{QP} \right)
\end{bmatrix}
\in \mathbb{R}^{d},
\end{align*}
where $\pi_q^{-1}(p) \in\left\{1,2,\cdots,P \right\}$ is the inverse function of $\pi_q(p)$, for all $q$ and $p$. With this notation, we can integrate all subsets $\omega_{q,\pi_q(p)}$ and simplify the inner loop of the SODDA as follows
\begin{center}
\begin{tabular}{ll}
\multicolumn{2}{l}{$\omega^{t}$}     \\
\multicolumn{1}{c}{} & $\bar{\omega}^{t,0}=\omega^{t}$   \\
                     & $\bar{\omega}^{t,1}=\bar{\omega}^{t,0}-\gamma_{t+1} \left( \mu^t+v^{t,1} \right)$   \\
                     & $\bar{\omega}^{t,2}=\bar{\omega}^{t,1}-\gamma_{t+1} \left( \mu^t+v^{t,2} \right)$   \\
                     & $\vdots$\\
                     &$\bar{\omega}^{t,B}=\bar{\omega}^{t,B-1}-\gamma_{t+1} \left( \mu^t+v^{t,B} \right)$\\
\multicolumn{2}{l}{$\omega^{t+1}=\bar{\omega}^{t,B}$} . 
\end{tabular}
\end{center}

In what follows Assumptions \hyperref[as:2]{2}-\hyperref[as:5]{5} hold. Lastly, we define $\mathcal{F}^t$ as the sigma algebra that measures the history of the algorithm up until iteration $t$.

We also introduce $f\in\hat{\mathcal{O}}(g)$ if there exists a constant $C>0$ such that $f(x)\leq C\cdot g(x)$ for every $x\geq 0$.

\subsection*{C\tab Diminishing Learning Rate Convergence without Feature Sampling}
\begin{lem}
\label{lem:1}
Let $\Phi=\left\{\phi_1,\cdots,\phi_{R}\right\}$ be a set of random vectors measurable with respect to $\sigma-$algebra $\mathcal{H}$, let $g:\Phi\to\mathbb{R}^k$ be a measurable function, and let $\mathcal{b}$ be an integer such that $1\leq \mathcal{b}\leq R$.
Let $\mathcal{B}$ be a set of size $\mathcal{b}$ uniformly and randomly selected vectors from $\Phi$ without replacement. Given two constants $w_1$ and $w_2$, we have
\begin{align*}
\E\left[\left.w_1\sum_{i\in\mathcal{B}}g(\phi_i)+w_2\sum_{i\notin \mathcal{B}}g(\phi_i)\right|\mathcal{H} \right]=\left(\frac{\mathcal{b}}{R}w_1+\frac{R-\mathcal{b}}{R}w_2\right)\sum_{i=1}^R g(\phi_i).
\end{align*} 
\end{lem}
\begin{proof}
Using the definition of the expectation we obtain
\begin{align*}
\E\left[\left.w_1\sum_{i\in\mathcal{B}}g(\phi_i)+w_2\sum_{i\notin \mathcal{B}}g(\phi_i)\right|\mathcal{H} \right]=\sum_{\mathcal{B}}\frac{1}{\binom{R}{\mathcal{b}}}\left[w_1\sum_{i\in\mathcal{B}}g(\phi_i)+w_2\sum_{i\notin \mathcal{B}}g(\phi_i)\right],
\end{align*}
where the first summation indicates summation over all subsets of $\mathcal{B}$ of cardinality $\mathcal{b}$. Thus, the expected value of $w_1\sum_{i\in\mathcal{B}}g(\phi_i)+w_2\sum_{i\notin \mathcal{B}}g(\phi_i)$ with respect to $\mathcal{B}$ and conditioning on $\mathcal{H}$ is
\begin{align*}
\E\left[\left.w_1\sum_{i\in\mathcal{B}}g(\phi_i)+w_2\sum_{i\notin \mathcal{B}}g(\phi_i)\right|\mathcal{H} \right]&=\left(\frac{\mathcal{b}}{R}w_1+\left(1-\frac{\mathcal{b}}{R}\right)w_2\right)\sum_{i=1}^R g(\phi_i)\\
&=\left(\frac{\mathcal{b}}{R}w_1+\frac{R-\mathcal{b}}{R}w_2\right)\sum_{i=1}^R g(\phi_i),
\end{align*}
since each $i$ is selected with probability $\frac{\mathcal{b}}{R}$.
\end{proof}

\begin{lem}
\label{lem:2}
If Assumption \hyperref[as:1]{1} holds, then $\left\|\triangledown F(\omega^t)\right\|$ and $\sum_{j=1}^N\left\|\triangledown f_j(x_j\omega^t)\right\|^2$ for any $t$ satisfy
\begin{align}
\label{eq:lem2-0}
\left\|\triangledown F(\omega^t)\right\|&\leq M_1,\\
\label{eq:lem2-1}
\sum_{j=1}^N\left\|\triangledown f_j(x_j\omega^t)\right\|^2&\leq NM_1^2.
\end{align}
\end{lem}
\begin{proof}
Using Assumptions \hyperref[as:1]{1} we obtain
\begin{align*}
    \left\|\bigtriangledown F(\omega^t)\right\|= \left\|\frac{1}{N}\sum_{i=1}^N\bigtriangledown f_j(x_j\omega^t) \right\|\leq \frac{1}{N}\sum_{i=1}^N\left\|\bigtriangledown f_j(x_j\omega^t) \right\|\leq \frac{1}{N}\sum_{i=1}^NM_1=M_1.
\end{align*}
Similarly, for any $\omega^t$ we have
\begin{align*}
\sum_{j=1}^N\left\|\triangledown f_j(x_jw^t)\right\|^2\leq \sum_{i=1}^NM_1^2=NM_1^2.
\end{align*}
This completes the proof of the lemma.
\end{proof}

We assume that $\omega^{*}$ is the unique optimal solution to \hyperref[obj]{(1)}. Under these standard assumptions and the previous results, our first proposition argues a supermartingale relationship for the sequence of the loss function errors $F(\omega^t)-F(\omega^*)$.
\begin{pro}
\label{pro1}
If Assumptions \hyperref[as:1]{1}-\hyperref[as:4]{4} hold true, and the sequence of learning rates satisfies $\gamma_t\leq 1$ for all $t$, and the sequences $(\mathcal{c}^t,\mathcal{d}^t)_{t=0}^{\infty}$ are selected so that $\mathcal{c}^t\leq d$ and $\mathcal{d}^t\leq N$, then the loss function error sequence $F(\omega^t)-F(\omega^*)$ generated by SODDA satisfies
\begin{align}
\E\left[F(\omega^{t+1})-F(\omega^*)|\mathcal{F}^t\right]\leq (1-\frac{2\xi B}{d}\gamma_{t+1})[F(\omega^t)-F(\omega^*)]+C_1\gamma_{t+1}^2,
\label{eq:pro1-0}
\end{align}
where $C_1$ is a positive constant.
\end{pro} 
\begin{proof}
We write $\triangledown F(\omega^{t})=\left(\triangledown F(\omega^{t})_{11},\cdots,\triangledown F(\omega^{t})_{1P},\triangledown F(\omega^{t})_{21},\cdots,\triangledown F(\omega^{t})_{QP}\right)$ and $e^t=(e^t_{11},\cdots,e^t_{1P},e^t_{21},\cdots,e^t_{QP})$. In order to simplify the notation, we also denote $\pi=\left(\pi_q\right)_{q=1}^Q$ and $j_{q,\pi_q(p)}^{(i)}$, the index drawn in step~\ref{11} of the algorithm for given $\pi_q(p)$, where everything computed is at iteration $t$.
\begin{claim}
\label{cl2}
For any $t$ we have
\begin{align}
\label{eq:cl2-0}
&\E\left[\left.\frac{1}{\mathcal{d}^t}\sum_{j\in\mathcal{D}^t}\bar{\triangledown}_{\omega_{\mathcal{C}^t}}f_j(x_j\omega^t)\right|\mathcal{F}^t\right]=\frac{\mathcal{c}^t}{d}\triangledown F(\omega^t),\\
\label{eq:cl2-1}
&\E\left[\left.\left\|\frac{1}{\mathcal{d}^t}\sum_{j\in\mathcal{D}^t}\bar{\triangledown}_{\omega_{\mathcal{C}^t}}f_j(x_j\omega^t)\right\|^2\right|\mathcal{F}^t\right]\leq\frac{\mathcal{c}^tM_1^2}{d}.
\end{align}
\end{claim}
\begin{proof}
Applying Lemma \ref{lem:1} with $w_1=1$, $w_2=0$, $\Phi=\left\{ \bar{\triangledown}_{\omega_{\mathcal{C}^t}}f_j(x_j\omega^t)\right\}_{j=1}^N$, $g(z)=z$, $\mathcal{H}=\sigma(\mathcal{F}^t,\mathcal{C}^t)$, $\mathcal{B}=\mathcal{D}^t$ and the law of iterated expectation imply
\begin{align}
&\E\left[\left.\E\left[\left.\frac{1}{\mathcal{d}^t}\sum_{j\in\mathcal{D}^t}\bar{\triangledown}_{\omega_{\mathcal{C}^t}}f_j(x_j\omega^t)\right|\mathcal{F}^t,\mathcal{C}^t\right]\right|\mathcal{F}^t\right]=\frac{1}{\mathcal{d}^t}\cdot\frac{\mathcal{d}^t}{N}\sum_{j=1}^N\E\left[\left.\bar{\triangledown}_{\omega_{\mathcal{C}^t}}f_j(x_j\omega^t)\right|\mathcal{F}^t\right]\nonumber.
\end{align}
For each $j$, we in turn have 
\begin{align*}
&\E\left[\left.\bar{\triangledown}_{\omega_{\mathcal{C}^t}}f_j(x_j\omega^t)\right|\mathcal{F}^t\right]=\frac{1}{\binom{d}{\mathcal{c}^t}}\sum_{\mathcal{C}^t}\bar{\triangledown}_{\omega_{\mathcal{C}^t}}f_j(x_j\omega^t)\\
&\quad=\frac{1}{\binom{d}{\mathcal{c}^t}}\cdot\binom{d-1}{\mathcal{c}^t-1}\triangledown f_j(x_j\omega^t)=\frac{\mathcal{c}^t}{d}\triangledown f_j(x_j\omega^t).
\end{align*}
This yields
\begin{align}
&\E\left[\left.\frac{1}{\mathcal{d}^t}\sum_{j\in\mathcal{D}^t}\bar{\triangledown}_{\omega_{\mathcal{C}^t}}f_j(x_j\omega^t)\right|\mathcal{F}^t\right]=\frac{1}{N}\sum_{j=1}^N\frac{\mathcal{c}^t}{d}\triangledown f_j(x_j\omega^t)=\frac{\mathcal{c}^t}{Nd}\sum_{j=1}^N\triangledown f_j(x_j\omega^t).
\label{eq:cl2-2}
\end{align}
By substituting the definition of $\triangledown F(\omega^t)$ into (\ref{eq:cl2-2}) claim (\ref{eq:cl2-0}) follows.

Let us proceed to find an upper bound for the expected value of $\left\|\frac{1}{\mathcal{d}^t}\sum_{j\in\mathcal{D}^t}\bar{\triangledown}_{\omega_{\mathcal{C}^t}}f_j(x_j\omega^t)\right\|^2$ given $\mathcal{F}^t$. Applying the law of iterated expectation and Lemma \ref{lem:1} with $w_1=1$, $w_2=0$, $\Phi=\left\{ \bar{\triangledown}_{\omega_{\mathcal{C}^t}}f_j(x_j\omega^t)\right\}_{j=1}^N$, $g(z)=\left\|z\right\|^2$, $\mathcal{H}=\sigma(\mathcal{F}^t,\mathcal{c}^t)$ and $\mathcal{B}=\mathcal{D}^t$ give
\begin{align}
&\E\left[\left.\left\|\frac{1}{\mathcal{d}^t}\sum_{j\in\mathcal{D}^t}\bar{\triangledown}_{\omega_{\mathcal{C}^t}}f_j(x_j\omega^t)\right\|^2\right|\mathcal{F}^t\right]\leq \frac{1}{\mathcal{d}^t}\E\left[\sum_{j\in\mathcal{D}^t}\left.\left\|\bar{\triangledown}_{\omega_{\mathcal{C}^t}}f_j(x_j\omega^t)\right\|^2\right|\mathcal{F}^t  \right]\\
&\quad=\frac{1}{\mathcal{d}^t}\E\left[\left.\E\left[\left.\sum_{j\in\mathcal{D}^t}\left\|\bar{\triangledown}_{\omega_{\mathcal{C}^t}}f_j(x_j\omega^t)\right\|^2\right|\mathcal{F}^t,\mathcal{C}^t\right]\right|\mathcal{F}^t  \right]
=\frac{1}{\mathcal{d}^t}\cdot \frac{\mathcal{d}^t}{N}\E\left[\sum_{j=1}^N\left.\left\|\bar{\triangledown}_{\omega_{\mathcal{C}^t}}f_j(x_j\omega^t)\right\|^2\right|\mathcal{F}^t  \right]\nonumber\\
&\quad=\frac{1}{N}\sum_{j=1}^N\E\left[\left.\left\|\bar{\triangledown}_{\omega_{\mathcal{C}^t}}f_j(x_j\omega^t)\right\|^2\right|\mathcal{F}^t  \right]\nonumber,
\end{align}
which in turn yields
\begin{align}
\E\left[\left.\left\|\frac{1}{\mathcal{d}^t}\sum_{j\in\mathcal{D}^t}\bar{\triangledown}_{\omega_{\mathcal{C}^t}}f_j(x_j\omega^t)\right\|^2\right|\mathcal{F}^t\right]&\leq\frac{1}{N}\sum_{j=1}^N\frac{\mathcal{c}^t}{d}\E\left[\left.\left\|\triangledown f_j(x_j\omega^t)\right\|^2\right|\mathcal{F}^t  \right]\nonumber\\
&=\frac{\mathcal{c}^t}{dN}\sum_{j=1}^N\E\left[\left.\left\|\triangledown f_j(x_j\omega^t)\right\|^2\right|\mathcal{F}^t  \right],
\label{eq:cl2-3}
\end{align}
where we apply Lemma \ref{lem:1} for each $j$ again with  $w_1=1$, $w_2=0$, $\Phi=\left\{\left( \triangledown f_j(x_j\omega^t)\right)_{i}\right\}_{i=1}^{d}$, $g(z)=z^2$, $\mathcal{H}=\mathcal{F}^t$ and $\mathcal{B}=\mathcal{C}^t$. By inserting (\ref{eq:lem2-1}) from  Lemma \ref{lem:2} into (\ref{eq:cl2-3}) the claim in (\ref{eq:cl2-1}) follows.
\end{proof}

For $i=1,2,\cdots,B$, the expected value of $\left\|v^{t,i}\right\|^2$ given all the preceding information is
\begin{align}
&\E\left[\left\|v^{t,1}\right\|^2\vert\mathcal{F}^t,\mathcal{B}^t,\mathcal{C}^t,\mathcal{D}^t,\pi\right]=0\label{eq:pro1-3}\\
&\E\left[\left\|v^{t,i}\right\|^2\vert\mathcal{F}^t,\mathcal{B}^t,\mathcal{C}^t,\mathcal{D}^t,\pi,j_{11}^{(1)},j_{12}^{(1)},\cdots,j_{QP}^{(1)},j_{11}^{(2)},j_{12}^{(2)},\cdots,j_{QP}^{(2)},\cdots,j_{11}^{(i-2)},j_{12}^{(i-2)},\cdots,j_{QP}^{(i-2)}\right]\nonumber\\
&\quad=\sum_{j_{QP}^{(i-1)}=1}^n\cdots\sum_{j_{11}^{(i-1)}=1}^n \frac{1}{n^{QP}} 
\left \|
\begin{pmatrix}
\triangledown_{\omega_{11}}f_{j_{11}^{(i-1)}}^{\pi_1^{-1}(1)}\left( x_{j_{11}^{(i-1)}}^{\pi_1^{-1}(1),1,1}\bar{\omega}_{11}^{t,i-1} \right)-\triangledown_{\omega_{11}}f_{j_{11}^{(i-1)}}^{\pi_1^{-1}(1)}\left( x_{j_{11}^{(i-1)}}^{\pi_1^{-1}(1),1,1}\omega^t_{11} \right)\\
\vdots\\
\triangledown_{\omega_{1P}}f_{j_{1P}^{(i-1)}}^{\pi_1^{-1}(P)}\left( x_{j_{1P}^{(i-1)}}^{\pi_1^{-1}(P),1,P}\bar{\omega}_{1P}^{t,i-1} \right)-\triangledown_{\omega_{1P}}f_{j_{1P}^{(i-1)}}^{\pi_1^{-1}(P)}\left( x_{j_{1P}^{(i-1)}}^{\pi_1^{-1}(P),1,P}\omega^t_{1P} \right)\\
\vdots\\
\triangledown_{\omega_{QP}}f_{j_{QP}^{(i-1)}}^{\pi_Q^{-1}(P)}\left( x_{j_{QP}^{(i-1)}}^{\pi_Q^{-1}(P),Q,P}\bar{\omega}_{QP}^{t,i-1} \right)-\triangledown_{\omega_{QP}}f_{j_{QP}^{(i-1)}}^{\pi_Q^{-1}(P)}\left( x_{j_{QP}^{(i-1)}}^{\pi_Q^{-1}(P),Q,P}\omega^t_{QP} \right)
\end{pmatrix}
\right \|^2,\label{eq:pro1-4}
\end{align}  
for $i=2,3,\cdots,B$.

We prove a bound of the expected value of $\left\|v^{t,i}\right\|^2$ given $\mathcal{F}^t$ by induction for $i\in\left\{1,2,\cdots,B\right\}$.
\begin{claim}
\label{cl:3}
For $i=1,2,\cdots,B$, we have
\begin{align}
\label{eq:cl3-0}
\E\left[\left\|v^{t,i}\right\|^2|\mathcal{F}^t\right] = \hat{\mathcal{O}}(\gamma_{t+1}^2). 
\end{align}
\end{claim}
\begin{proof}
The claim holds for $i=1$ due to (\ref{eq:pro1-3}).

For $i=1,2,\cdots,k-1$, we assume that 
\begin{align}
\label{eq:cl3-1}
&\E\left[\left\|v^{t,i}\right\|^2|\mathcal{F}^t\right] = \hat{\mathcal{O}}(\gamma_{t+1}^2).
\end{align}

Now consider $v^{t,k}$. Let us show that the expected value of $\left\|v^{t,k}\right\|^2$ is bounded. By using (\ref{eq:pro1-4}) we have
\begin{align}
&\E\left[\left\|v^{t,k}\right\|^2|\mathcal{F}^t,\mathcal{B}^t,\mathcal{C}^t,\mathcal{D}^t,\pi,j_{11}^{(1)},j_{12}^{(1)},\cdots,j_{QP}^{(1)},j_{11}^{(2)},j_{12}^{(2)},\cdots,j_{QP}^{(2)},\cdots,j_{11}^{(k-2)},j_{12}^{(k-2)},\cdots,j_{QP}^{(k-2)} \right]\nonumber\\
&\leq\sum_{j_{QP}^{(k-1)}=1}^n\cdots\sum_{j_{11}^{(k-1)}=1}^n\sum_{q=1}^Q\sum_{p=1}^P \frac{1}{n^{QP}}\left\|\triangledown_{\omega_{qp}}f_{j_{qp}^{(k-1)}}^{\pi_q^{-1}(p)}\left( x_{j_{qp}^{(k-1)}}^{\pi_q^{-1}(p),q,p}\bar{\omega}_{qp}^{t,k-1} \right)-\triangledown_{\omega_{qp}}f_{j_{qp}^{(k-1)}}^{\pi_q^{-1}(p)}\left( x_{j_{qp}^{(k-1)}}^{\pi_q^{-1}(p),q,p}\omega^t_{qp} \right)  \right\|^2\nonumber\\
&=\sum_{q=1}^Q\sum_{p=1}^P \left( \frac{1}{n}\sum_{j_{q,\pi_q(p)}^{(k-1)}=1}^n\left\|\triangledown_{\omega_{qp}}f_{j_{qp}^{(k-1)}}^{\pi_q^{-1}(p)}\left( x_{j_{qp}^{(k-1)}}^{\pi_q^{-1}(p),q,p}\bar{\omega}_{qp}^{t,k-1} \right)-\triangledown_{\omega_{qp}}f_{j_{qp}^{(k-1)}}^{\pi_q^{-1}(p)}\left( x_{j_{qp}^{(k-1)}}^{\pi_q^{-1}(p),q,p}\omega^t_{qp} \right)  \right\|^2 \right)\nonumber\\
&\leq \sum_{q=1}^Q\sum_{p=1}^P\left( \frac{1}{n}\cdot nL^2\left\|\bar{\omega}_{qp}^{t,k-1}-\bar{\omega}_{qp}^{t,0} \right\|^2 \right)=\sum_{q=1}^Q\sum_{p=1}^P\left(  L^2\left\|\bar{\omega}_{qp}^{t,k-1}-\bar{\omega}_{qp}^{t,0} \right\|^2 \right)\nonumber\\
& = \sum_{q=1}^Q\sum_{p=1}^PL^2\left\|\gamma_{t+1}\left[ (k-1)\mu^t_{qp}+v_{qp}^{t,1}+\cdots+v_{qp}^{t,k-1} \right] \right\|^2.
\label{eq:cl3-2}
\end{align}
Applying the definition of $\mu^t$ yields
\begin{align}
\E\left[\left\|\mu^t\right\|^2|\mathcal{F}^t  \right]
&= \E\left[\left.\left\|\frac{1}{\mathcal{d}^t}\sum_{j\in\mathcal{D}^t}\bar{\triangledown}_{\omega_{\mathcal{C}^t}}f_j(x_j\omega^t)\right\|^2\right|\mathcal{F}^t\right]\leq \frac{\mathcal{c}^tM_1^2}{d}=\hat{\mathcal{O}}(1).
\label{eq:cl3-02} 
\end{align}
The second inequality holds due to  (\ref{eq:cl2-1}) in Claim \ref{cl2}. By using the law of iterated expectation, (\ref{eq:cl3-1}), (\ref{eq:cl3-2}) and (\ref{eq:cl3-02}) we get
\begin{equation}
\begin{split}
\label{eq:cl3-3}
&\E\left[\left\|v^{t,k}\right\|^2|\mathcal{F}^t\right]=\E\left[\E\left[\left\|v^{t,k}\right\|^2|\mathcal{F}^t,,\mathcal{B}^t,\mathcal{C}^t,\mathcal{D}^t,\pi,j_{11}^{(1)}\cdots,j_{QP}^{(k-2)} \right]|\mathcal{F}^t\right]\\
&\leq \E\left[   \left.\sum_{q=1}^Q\sum_{p=1}^PL^2\left\|\gamma_{t+1}\left[ (k-1)\mu^t_{qp}+\sum_{i=1}^{k-1}v_{qp}^{t,i} \right]  \right\|^2  \right|\mathcal{F}^t\right] \\
&\leq kL^2\gamma_{t+1}^2\sum_{q=1}^Q\sum_{p=1}^P\left( \E\left[\left. \left\| (k-1) \mu^t_{qp} \right\|^2\right|\mathcal{F}^t\right]+ \sum_{i=1}^{k-1}\E\left[ \left\| v_{qp}^{t,i}\right\|^2|\mathcal{F}^t\right]\right)
\\
&\leq kL^2\gamma_{t+1}^2QP\left((k-1)^2\E\left[\left. \left\|  \mu^t \right\|^2\right|\mathcal{F}^t\right]+ \sum_{i=1}^{k-1}\E\left[ \left\| v^{t,i}\right\|^2|\mathcal{F}^t\right]\right)
\\
&=kL^2\gamma_{t+1}^2QP\left[(k-1)^2\hat{\mathcal{O}}(1)+(k-1)\hat{\mathcal{O}}(\gamma_{t+1}^2)\right]=\hat{\mathcal{O}}(\gamma_{t+1}^2).
\end{split}
\end{equation}
This completes the proof of the claim.
\end{proof}
By using the conditional Jensen's inequality and (\ref{eq:cl3-0}) we get
\begin{align}
\expect[\big]{\left\|v^{t,i}\right\||\mathcal{F}^t}=\expect[\big]{\sqrt{\left\|v^{t,i}\right\|^2}|\mathcal{F}^t}\leq\sqrt{\expect[\big]{\left\|v^{t,i}\right\|^2|\mathcal{F}^t}}=\hat{\mathcal{O}}(\gamma_{t+1}).
\label{eq:pro1-5}
\end{align}

By summing up all increments in iteration $t$, we obtain
\begin{align*}
\omega^{t+1}&=\omega^t-\gamma_{t+1} \left[ B\mu^t+v^{t,1}+v^{t,2}+\cdots+v^{t,B} \right].
\end{align*}
Then, the expected value of the difference $\omega^{t+1}-\omega^t$ given $\mathcal{F}^t$ is
\begin{equation}
\begin{split}
\E\left[\omega^{t+1}-\omega^t|\mathcal{F}^t\right]&=-\gamma_{t+1}\E\left[B\mu^t+v^{t,1}+\cdots+v^{t,B}|\mathcal{F}^t\right]\\
&=-\gamma_{t+1}B\frac{\mathcal{c}^t}{d}\triangledown F(\omega^t) -\gamma_{t+1}\sum_{i=1}^{B}\expect[\big]{v^{t,i}|\mathcal{F}^t},
\label{eq:pro1-6}
\end{split}
\end{equation}
by using (\ref{eq:cl2-0}) in Claim \ref{cl2}. Moreover, the expected value of the squared norm $\left\| \omega^{t+1}-\omega^t \right\|^2$ given $\mathcal{F}^t$ is
\begin{equation}
\begin{split}
\label{eq:pro1-7}
&\E\left[ \left\|\omega^{t+1}-\omega^t\right\|^2|\mathcal{F}^t\right]=\E\left[ \left\|\gamma_{t+1}\left[ B\mu^t+v^{t,1}+v^{t,2}+\cdots+v^{t,B} \right] \right\|^2|\mathcal{F}^t\right]\\
&\quad\leq \gamma_{t+1}^2(B+1) \left\{ B^2\E\left[ \left\| \mu^t\right\|^2|\mathcal{F}^t\right]+\sum_{i=1}^{B}\E\left[ \left\| v^{t,i}\right\|^2|\mathcal{F}^t\right]\right\}\\
&\quad=\hat{\mathcal{O}}(\gamma_{t+1}^2)\left\{B^2\cdot\hat{\mathcal{O}}(1) +B\cdot\hat{\mathcal{O}}(\gamma_{t+1}^2)  \right\}=\hat{\mathcal{O}}(\gamma_{t+1}^2),
\end{split}
\end{equation}
due to (\ref{eq:cl2-1}), (\ref{eq:cl3-0}) and (\ref{eq:cl3-02}).
From
\begin{align*}
-\triangledown F(\omega^t)\cdot v^{t,i}\leq \left\|\triangledown F(\omega^t)\right\| \left\|v^{t,i}\right\|,
\end{align*}
for every $\triangledown F(\omega^t)$ and $v^{t,i}$, by using (\ref{eq:lem2-0}) and (\ref{eq:pro1-5}) we obtain
\begin{align}
\label{eq:pro1-8}
-\gamma_{t+1}\triangledown F(\omega^t)\cdot \E\left[ v^{t,i}|\mathcal{F}^t\right]&=\gamma_{t+1}\E\left[-\triangledown F(\omega^t)\cdot v^{t,i}|\mathcal{F}^t\right]\leq \gamma_{t+1}\E\left[\left\|\triangledown F(\omega^t)\right\| \cdot \left\|v^{t,i}\right\||\mathcal{F}^t\right]\nonumber\\
&=\gamma_{t+1}\left\|\triangledown F(\omega^t)\right\|\cdot \E\left[ \left\|v^{t,i}\right\||\mathcal{F}^t\right]=\hat{\mathcal{O}}(\gamma_{t+1}^2),
\end{align}
since $\E\left[XY|\mathcal{H}\right]=X\E\left[Y|\mathcal{H}\right]$ if $X$ is $\mathcal{H}-$measurable.

For convex $F$ we have
\begin{align*}
F(\omega^{t+1})\leq F(\omega^t)+\triangledown F(\omega^t)^T \left(\omega^{t+1}-\omega^t\right)+\frac{L}{2}\left\| \omega^{t+1}-\omega^t\right\|^2,
\end{align*}
which in turn yields
\begin{equation}
\begin{split}
\label{eq:pro1-10}
&\E\left[F(\omega^{t+1})|\mathcal{F}^t\right]\leq F(\omega^t)+\triangledown F(\omega^t)^T\E\left[\left(\omega^{t+1}-\omega^t\right)|\mathcal{F}^t\right]+\frac{L}{2}\E \left[\left\| \omega^{t+1}-\omega^t\right\|^2|\mathcal{F}^t\right]\\
&=F(\omega^{t})+\triangledown F(\omega^t)^T\left\{ -\gamma_{t+1}B \frac{\mathcal{c}^t}{d}\triangledown F(\omega^t) -\gamma_{t+1}\sum_{i=1}^{B}\expect[\big]{v^{t,i}|\mathcal{F}^t} \right\} +\frac{L}{2}\E\left[\left\| \omega^{t+1}-\omega^t\right\|^2|\mathcal{F}^t\right]\\
&=F(\omega^t)-\gamma_{t+1} \frac{\mathcal{c}^tB}{d} \left\| \triangledown F(\omega^t)\right\|^2-\gamma_{t+1} \triangledown F(\omega^t)^T \sum_{i=1}^{B} \E\left[v^{t,i}|\mathcal{F}^t\right]+\frac{L}{2}\E\left[\left\| \omega^{t+1}-\omega^t\right\|^2|\mathcal{F}^t\right]\\
&\leq F(\omega^t)-\gamma_{t+1} \frac{\mathcal{c}^tB}{d} \left\| \triangledown F(\omega^t)\right\|^2+\hat{\mathcal{O}}(\gamma_{t+1}^2)\leq  F(\omega^t)-\gamma_{t+1} \frac{\mathcal{c}^tB}{d} \left\| \triangledown F(\omega^t)\right\|^2+C_1\gamma_{t+1}^2,
\end{split}
\end{equation}
where $C_1$ is a positive constant and we use (\ref{eq:pro1-6}), (\ref{eq:pro1-7}) and (\ref{eq:pro1-8}). Subtracting the optimal objective function $F(\omega^{*})$ to the both sides of (\ref{eq:pro1-10}) and using the fact that $\mathcal{c}^t\geq 1$ imply that
\begin{equation}
\label{eq:pro1-11}
\E\left[F(\omega^{t+1})-F(\omega^{*})|\mathcal{F}^t\right]\leq F(\omega^t)-F(\omega^{*})-
\gamma_{t+1} \frac{B}{d} \left\| \triangledown F(\omega^t)\right\|^2+C_1\gamma_{t+1}^2.
\end{equation}

We proceed to find a lower bound of $\left\|\triangledown F(\omega^t)\right\|^2$ in terms of $F(\omega^t)-F(\omega^{*})$. Assumption \hyperref[as:2]{2} implies that, for any $y,z\in \mathbb{R}^m$
\begin{align}
\label{eq:pro1-12}
F(y)\geq F(z)+\triangledown F(z)^T(y-z)+\frac{\xi}{2}\left\|y-z\right\|^2.
\end{align}
For fixed $z$, the right hand side of (\ref{eq:pro1-12}) is a quadratic function of $y$ and it gets its minimum at $\hat{y}=z-\frac{1}{\xi}\triangledown F(z)$. Therefore
\begin{align}
\label{eq:pro1-13}
F(y)\geq F(z)+\triangledown F(z)^T(\hat{y}-z)+\frac{\xi}{2}\left\|\hat{y}-z\right\|^2=F(z)-\frac{1}{2\xi}\left\|\triangledown F(z)\right\|^2,
\end{align}
for any $y,z\in \mathbb{R}^{d}$. Setting $y=\omega^{*}$ and $z=\omega^t$ in (\ref{eq:pro1-13}) gives
\begin{align}
\label{eq:pro1-14}
\left\|\triangledown F(\omega^t)\right\|^2\geq 2\xi\left(F(\omega^t)-F(\omega^{*})\right).
\end{align}
Substituting the lower bound in  (\ref{eq:pro1-14}) by the norm of gradient square $\left\|\triangledown F(\omega^t)\right\|^2$ in (\ref{eq:pro1-11}) yields the proposition in (\ref{eq:pro1-0}).
\end{proof}
Proposition \hyperref[pro1]{1} represents a supermartingale relationship for the sequence of the loss function errors $F(\omega^t)-F(\omega^*)$. In the following theorem, by employing the supermartingale convergence argument, we show that if the sequence of learning rates satisfy the standard stochastic approximation diminishing learning rate rule (non-summable and squared summable), the sequence of loss function errors $F(\omega^t)-F(\omega^*)$ converges to 0 almost
surely. Combining with strong convexity of $F(\omega)$ in Assumption \hyperref[as:2]{2}, this result implies that $\left\|\omega^t-\omega^* \right\|$ converges to 0 almost surely.
\subsubsection*{Proof of Theorem \ref{thm:5}}
\label{proof:thm5}
\begin{proof}
We use the relationship in (\ref{eq:pro1-0}) to build a supermartingale sequence. First, let us define
\begin{align}
\label{eq:thm1-1}
&\alpha^t:=F(\omega^t)-F(\omega^{*})+\sum_{u=t}^{\infty}C_1\gamma_{u+1}^2,\\
\label{eq:thm1-2}
&\beta^t:=\frac{2\xi B}{d}\gamma_{t+1}\left(F(\omega^t)-F(\omega^{*})\right).
\end{align}
Note that $\alpha^t$ is well-defined since $\sum_{u=t}^{\infty}\gamma_{u+1}^2<\sum_{u=1}^{\infty}\gamma_u^2<\infty$ . The definition of $\alpha^t$ and $\beta^t$ in (\ref{eq:thm1-1}) and (\ref{eq:thm1-2}), and the inequality in (\ref{eq:pro1-0}) imply the expected value of $\alpha^{t+1}$ given $\mathcal{F}^t$ is
\begin{align}
\label{eq:thm1-3}
\E\left[\alpha^{t+1}|\mathcal{F}^t\right]\leq \alpha^t-\beta^t.
\end{align}  
Since $\alpha^t $ and $\beta^t$ are nonnegative and due to (\ref{eq:thm1-3}), they satisfy the conditions of the supermartingale convergence theorem. Thus, we conclude that
\begin{align}
\label{eq:thm1-4}
&\mathrm{(i)}\quad \alpha^t \mathrm{\,\,converges\,\, to\,\, a\,\, limit\,\, a.s.,\, and }\\
\label{eq:thm1-5}
&\mathrm{(ii)}\quad \sum_{t=1}^{\infty}\beta^t<\infty. \quad \mathrm{a.s.}
\end{align}

Property (\ref{eq:thm1-5}) yields
\begin{align*}
\sum_{t=0}^{\infty}\frac{2\mathcal{c}^t\xi B}{d}\gamma_{t+1}\left(F(\omega^t)-F(\omega^{*})\right)<\infty.  \quad \mathrm{a.s.}
\end{align*}
Since $\sum_{t=0}^{\infty}\gamma_{t+1}=\infty$, there exists a subsequence of $F(\omega^t)-F(\omega^{*})$ which converges to 0, i.e.
\begin{align}
\label{eq:thm1-6}
\liminf_{t\to\infty} F(\omega^t)-F(\omega^*)=0.\quad \mathrm{a.s.}
\end{align}
Since $\sum_{u=t}^{\infty}C_1\gamma_{u+1}^2$ is deterministic and due to (\ref{eq:thm1-4}), $F(\omega^t)-F(\omega^*)$ converges to a limit almost surely. In association with (\ref{eq:thm1-6}) we conclude
\begin{align}
\label{eq:thm1-7}
\lim\limits_{t\to\infty}F(\omega^t)-F(\omega^*)=0.\quad \mathrm{a.s.}
\end{align}

We proceed to show the almost convergence of $\left\|\omega^t-\omega^*\right\|^2$.
Using (\ref{eq:pro1-12}) again and setting $y=\omega^t$ and $z=\omega^*$ implies
\begin{align}
\label{eq:thm1-8}
F(\omega^t)\geq F(\omega^*)+\triangledown F(\omega^*)^T(\omega^t-\omega^*)+\frac{\xi}{2}\left\|\omega^t-\omega^*\right\|^2.
\end{align}
Since the gradient of the optimal solution is $0$, i.e.$\triangledown F(\omega^*)=0$, (\ref{eq:thm1-8}) can be rearranged as
\begin{align*}
F(\omega^t)-F(\omega^*) \geq \frac{\xi}{2}\left\|\omega^t-\omega^*\right\|^2.
\end{align*}
Observing that the upper bound of $\left\|\omega^t-\omega^*\right\|^2$ converges to 0 almost surely by (\ref{eq:thm1-7}), we conclude that the sequence $\left\|\omega^t-\omega^*\right\|^2$ converges to zero almost surely. Hence, the claim in (\ref{eq:thm5-0}) is valid.
\end{proof}
\subsubsection*{Proof of Theorem \ref{thm:6}}
\label{proof:thm6}
\begin{proof}
Replacing $\gamma_{t+1}$ by $\frac{1}{t+1}$ and computing the expected value of (\ref{eq:pro1-0}) given $\mathcal{F}^0$ by using the law of iterated expectation we obtain
\begin{align}
\label{eq:thm2-3}
\expect[\big]{F(\omega^{t+1})-F(\omega^*)}\leq \left( 1-\frac{2 \xi B}{(t+1)d}\right)\expect[\big]{F(\omega^{t})-F(\omega^*)} +\frac{C_1}{(t+1)^2}.
\end{align}
Let us define
\begin{align*}
&a_t:=\expect[\big]{F(\omega^{t+1})-F(\omega^*)}\\
&\lambda:=\frac{2\xi B}{d}\\
&\beta:=C_1.
\end{align*}
Note that $\beta$ is positive. Based on the relationship in (\ref{eq:thm2-3}), we obtain
\begin{align}
\label{eq:thm2-4}
a_{t+1}\leq \left( 1-\frac{\lambda}{t+1}\right)a_t+\frac{\beta}{(t+1)^2}.
\end{align}
for all times $t\geq 0$. Now, we proceed to show
\begin{align}
\label{eq:thm2-5}
a_t\leq \frac{Q}{t+1},
\end{align}
where $Q=\max\left\{a_0,2a_1,\cdots,([\lambda]+1)a_{[\lambda]},\left([\lambda]+2\right) a_{[\lambda]+1}, \frac{\beta}{\lambda-1} \right\}$. The definition of $Q$ implies that the relationship in (\ref{eq:thm2-5}) holds for $t=1,2,\cdots,[\lambda]$. The remaining cases are shown by induction.
 
When $t=[\lambda]+1$, the definition of $Q$ implies
\begin{align*}
a_{[\lambda]+1}\leq \frac{Q}{[\lambda]+2}.
\end{align*} 
 
When $t=k-1$, we assume that the relationship in (\ref{eq:thm2-5}) holds. Considering the case when $t=k$ and using (\ref{eq:thm2-4}) implies
\begin{equation*}
\begin{split}
a_{k+1}\leq \left(1-\frac{\lambda}{k+1}\right)a_k+\frac{\beta}{(k+1)^2}\leq \left(1-\frac{\lambda}{k+1}\right)\frac{Q}{k+1}+\frac{\beta}{(k+1)^2}.
\end{split}
\end{equation*}
In order to satisfy (\ref{eq:thm2-5}), we require
\begin{equation*}
\begin{split}
\left(1-\frac{\lambda}{k+1}\right)\frac{Q}{k+1}+\frac{\beta}{(k+1)^2}\leq \frac{Q}{k+2}.
\end{split}
\end{equation*}
Elementary algebraic manipulation shows that this is equivalent to
\begin{equation*}
\begin{split}
 \beta (k+2)\leq Q\left[ \lambda (k+2)-(k+1) \right]
\end{split}
\end{equation*}
and in turn
\begin{align*}
\frac{\beta(k+2)}{\lambda (k+2)-(k+1)}=\frac{\beta}{\lambda-\frac{k+1}{k+2}}\leq Q,
\end{align*}
where we require $\lambda \geq 1$. The definition of $Q$, i.e. $Q\geq \frac{\beta}{\lambda-1}$ and the relationship that $\lambda-\frac{k+1}{k+2}>\lambda-1$ imply that
\begin{align*}
\frac{\beta}{\lambda-\frac{k+1}{k+2}}<\frac{\beta}{\lambda-1} \leq Q,
\end{align*}
and thus (\ref{eq:thm2-5}) holds for $t=k$. Thus, if $B\geq\frac{d}{2 \xi}$, for any time $t\geq 0$, the result in (\ref{eq:thm6-0}) holds where the constant $Q$ is defined based on (\ref{eq:thm6-1}).
\end{proof}
\noindent
{\bf Corollary 1} {\it If Assumptions \hyperref[as:1]{1}, \hyperref[as:2]{2} and \hyperref[as:3]{3} hold true and the sequence of learning rates are non-summable $\sum_{t=1}^{\infty}\gamma_{t}=\infty$ and square summable $\sum_{t=1}^{\infty}\gamma_{t}^2<\infty$, then the sequence of parameters $\omega^t$ generated by $\mathrm{RADiSA}$ converges almost surely to the optimal solution $\omega^*$, that is
\begin{align}
\lim\limits_{t\to\infty} \left\|\omega^t-\omega^*\right\|^2=0\quad\mathrm{a.s.}
\label{eq:cor1-0}
\end{align}
Moreover, if learning rate is defined as $\gamma_t:=\frac{1}{t}$ for $t=1,2,\cdots$ and the batch size is chosen such that $B\geq \frac{1}{2\xi}$, then the expected loss function errors $\expect[\big]{F(\omega^t)-F(\omega^*)}$ of $\mathrm{RADiSA}$ converges to 0 at least with a sublinear convergence rate of order $\mathcal{O}(1/t)$, i.e.
\begin{align}
\expect[\big]{F(\omega^t)-F(\omega^*)}\leq \frac{Q}{1+t},
\label{eq:cor1-1}
\end{align} 
where constant $Q$ is defined in (\ref{eq:thm6-1}) with some positive constant $C_1'$ taking the place of $C_1$ and $\mathcal{c}^t=d$.
}
\begin{proof}
RADiSA is a special case of SODDA with $\mathcal{c}^t=d$, $\mathcal{d}^t=N$.
\end{proof}
\subsection*{D\tab Constant Learning Rate without Feature Sampling}
\begin{pro}
\label{pro2}
If Assumptions \hyperref[as:1]{1}-\hyperref[as:4]{4} hold true, and the learning rate is constant $\gamma_{t}=\gamma $ such that $BL\gamma QP\leq 1$ and $\gamma\leq 1$, and the sequences $(\mathcal{c}^t,\mathcal{d}^t)_{t=0}^{\infty}$ satisfy the same conditions as in Theorem \ref{thm:5}, then the loss function error sequence $F(\omega^t)-F(\omega^*)$ generated by SODDA satisfies
\begin{align}
\E\left[\left.F(\omega^{t+1})-F(\omega^*)\right|\mathcal{F}^t   \right]\leq \left( 1-\frac{2\xi B}{d}\gamma \right)\left[ F(\omega^t)-F(\omega^*)\right]+C_2B^4\gamma^2,
\label{eq:pro2-0}
\end{align}
where $C_2$ is a positive constant.
\end{pro}
\begin{proof}
For $i=1,2,\cdots,B$, the expected value of $\left\|v^{t,i}\right\|$ given all the preceding information is
\begin{align}
&\E\left[\left\|v^{t,1}\right\|\vert\mathcal{F}^t,\mathcal{B}^t,\mathcal{C}^t,\mathcal{D}^t,\pi\right]=0\label{eq:pro2-1}\\
&\E\left[\left\|v^{t,i}\right\|\vert\mathcal{F}^t,\mathcal{B}^t,\mathcal{C}^t,\mathcal{D}^t,\pi,j_{11}^{(1)},j_{12}^{(1)},\cdots,j_{QP}^{(1)},j_{11}^{(2)},j_{12}^{(2)},\cdots,j_{QP}^{(2)},\cdots,j_{11}^{(i-2)},j_{12}^{(i-2)},\cdots,j_{QP}^{(i-2)}\right]\nonumber\\
&\quad=\sum_{j_{QP}^{(i-1)}=1}^n\cdots\sum_{j_{11}^{(i-1)}=1}^n \frac{1}{n^{QP}} 
\left \|
\begin{pmatrix}
\triangledown_{\omega_{11}}f_{j_{11}^{(i-1)}}^{\pi_1^{-1}(1)}\left( x_{j_{11}^{(i-1)}}^{\pi_1^{-1}(1),1,1}\bar{\omega}_{11}^{t,i-1} \right)-\triangledown_{\omega_{11}}f_{j_{11}^{(i-1)}}^{\pi_1^{-1}(1)}\left( x_{j_{11}^{(i-1)}}^{\pi_1^{-1}(1),1,1}\omega^t_{11} \right)\\
\vdots\\
\triangledown_{\omega_{1P}}f_{j_{1P}^{(i-1)}}^{\pi_1^{-1}(P)}\left( x_{j_{1P}^{(i-1)}}^{\pi_1^{-1}(P),1,P}\bar{\omega}_{1P}^{t,i-1} \right)-\triangledown_{\omega_{1P}}f_{j_{1P}^{(i-1)}}^{\pi_1^{-1}(P)}\left( x_{j_{1P}^{(i-1)}}^{\pi_1^{-1}(P),1,P}\omega^t_{1P} \right)\\
\vdots\\
\triangledown_{\omega_{QP}}f_{j_{QP}^{(i-1)}}^{\pi_Q^{-1}(P)}\left( x_{j_{QP}^{(i-1)}}^{\pi_Q^{-1}(P),Q,P}\bar{\omega}_{QP}^{t,i-1} \right)-\triangledown_{\omega_{QP}}f_{j_{QP}^{(i-1)}}^{\pi_Q^{-1}(P)}\left( x_{j_{QP}^{(i-1)}}^{\pi_Q^{-1}(P),Q,P}\omega^t_{QP} \right)
\end{pmatrix}
\right \|,\label{eq:pro2-2}
\end{align}  
for $i=2,3,\cdots,B$.

We prove a bound of the expected value of $\sum_{i=1}^{B} \left\|v^{t,i}\right\|$ and $\sum_{i=1}^{B} \left\|v^{t,i}\right\|^2$ given $\mathcal{F}^t$ by induction.
\begin{claim}
\label{cl:4}
For any $t$, if $BL\gamma QP \leq 1$ and $\gamma\leq 1$, we have
\begin{align}
&\sum_{i=1}^{B}\E\left[\left.\left\| v^{t,i} \right\|\right|\mathcal{F}^t   \right]=\hat{\mathcal{O}}(B^3\gamma)
\label{eq:cl4-0}\\
&\sum_{i=1}^{B}\E\left[\left.\left\| v^{t,i} \right\|^2\right|\mathcal{F}^t   \right]=\hat{\mathcal{O}}(B^4\gamma^2)+\hat{\mathcal{O}}(B^7\gamma^4).
\label{eq:cl4-1}
\end{align}
\end{claim}
\begin{proof}
By using (\ref{eq:pro2-1}) we have
\begin{align}
\E\left[\left.\left\|v^{t,1}\right\|\right|\mathcal{F}^t\right]=\E\left[\left.\E\left[\left\|v^{t,1}\right\|\vert\mathcal{F}^t,\mathcal{B}^t,\mathcal{C}^t,\mathcal{D}^t,\pi\right]\right|\mathcal{F}^t\right]=0.
\label{eq:cl4-2}
\end{align}
For $i=2,3,\cdots,B$, using (\ref{eq:pro2-2}) gives
\begin{align}
&\E\left[\left\|v^{t,i}\right\||\mathcal{F}^t,\mathcal{B}^t,\mathcal{C}^t,\mathcal{D}^t,\pi,j_{11}^{(1)},j_{12}^{(1)},\cdots,j_{QP}^{(1)},j_{11}^{(2)},j_{12}^{(2)},\cdots,j_{QP}^{(2)},\cdots,j_{11}^{(i-2)},j_{12}^{(i-2)},\cdots,j_{QP}^{(i-2)} \right]\nonumber\\
&\leq\sum_{j_{QP}^{(i-1)}=1}^n\cdots\sum_{j_{11}^{(i-1)}=1}^n\sum_{q=1}^Q\sum_{p=1}^P \frac{1}{n^{QP}}\left\|\triangledown_{\omega_{qp}}f_{j_{qp}^{(k-1)}}^{\pi_q^{-1}(p)}\left( x_{j_{qp}^{(k-1)}}^{\pi_q^{-1}(p),q,p}\bar{\omega}_{qp}^{t,k-1} \right)\right.\nonumber\\
&\left.\quad\quad\quad\quad\quad\quad\quad\quad\quad-\triangledown_{\omega_{qp}}f_{j_{qp}^{(k-1)}}^{\pi_q^{-1}(p)}\left( x_{j_{qp}^{(k-1)}}^{\pi_q^{-1}(p),q,p}\omega^t_{qp} \right)  \right\|\nonumber\\
&=\sum_{q=1}^Q\sum_{p=1}^P \left( \frac{1}{n}\sum_{j_{q,\pi_q(p)}^{(i-1)}=1}^n\left\|\triangledown_{\omega_{qp}}f_{j_{qp}^{(k-1)}}^{\pi_q^{-1}(p)}\left( x_{j_{qp}^{(k-1)}}^{\pi_q^{-1}(p),q,p}\bar{\omega}_{qp}^{t,k-1} \right)-\triangledown_{\omega_{qp}}f_{j_{qp}^{(k-1)}}^{\pi_q^{-1}(p)}\left( x_{j_{qp}^{(k-1)}}^{\pi_q^{-1}(p),q,p}\omega^t_{qp} \right)  \right\| \right)\nonumber\\
&\leq \sum_{q=1}^Q\sum_{p=1}^P\left( \frac{1}{n}\cdot nL\left\|\bar{\omega}_{qp}^{t,i-1}-\bar{\omega}_{qp}^{t,0} \right\| \right)=\sum_{q=1}^Q\sum_{p=1}^P\left(  L\left\|\bar{\omega}_{qp}^{t,i-1}-\bar{\omega}_{qp}^{t,0} \right\| \right)\nonumber\\
& = \sum_{q=1}^Q\sum_{p=1}^PL\left\|\gamma \left[ (i-1)\mu^t_{qp}+v_{qp}^{t,1}+\cdots+v_{qp}^{t,i-1} \right] \right\|. 
\label{eq:cl4-3}
\end{align}
By using the law of iterated expectation and (\ref{eq:cl4-3}) we get
\begin{equation}
\begin{split}
\label{eq:cl4-4}
&\E\left[\left\|v^{t,i}\right\||\mathcal{F}^t\right]=\E\left[\E\left[\left\|v^{t,i}\right\||\mathcal{F}^t,,\mathcal{B}^t,\mathcal{C}^t,\mathcal{D}^t,\pi,j_{11}^{(1)}\cdots,j_{QP}^{(i-2)} \right]|\mathcal{F}^t\right]\\
&\leq \E\left[   \left.\sum_{q=1}^Q\sum_{p=1}^PL\left\|\gamma \left[ (i-1)\mu^t_{qp}+\sum_{j=1}^{i-1}v_{qp}^{t,j} \right]  \right\|  \right|\mathcal{F}^t\right] \\
&\leq L\gamma \sum_{q=1}^Q\sum_{p=1}^P\left( \E\left[\left. \left\| (i-1) \mu^t_{qp} \right\|\right|\mathcal{F}^t\right]+ \sum_{j=1}^{i-1}\E\left[ \left\| v_{qp}^{t,j}\right\||\mathcal{F}^t\right]\right)
\\
&\leq L\gamma QP\left((i-1)\E\left[\left. \left\|  \mu^t \right\|\right|\mathcal{F}^t\right]+ \sum_{j=1}^{i-1}\E\left[ \left\| v^{t,j}\right\||\mathcal{F}^t\right]\right).
\end{split}
\end{equation}
Let us define
\begin{align*}
&a_i:=\E\left[\left.\left\|v^{t,i}\right\|\right|\mathcal{F}^t\right]\\
&\nu=L\gamma QP\\
&D_1:=\E\left[\left.\left\|\mu^t\right\|\right|\mathcal{F}^t\right].
\end{align*}
Then the recursive formula becomes
\begin{align}
&a_1=0\nonumber\\
&a_i\leq \nu \left((i-1)D_1+\sum_{j=1}^{i-1}a_j  \right),
\label{eq:cl4-5}
\end{align}
for $i=2,3,\cdots,B$. Let us define $\bar{a}_i$ as
\begin{align}
\label{eq:cl4-6}
\bar{a}_i=
\left\{\begin{matrix}
0,&i=1\\ 
\nu\left((i-1)D_1+\sum_{j=1}^{i-1}\bar{a}_j  \right),&i\neq 1.
\end{matrix}\right.
\end{align}

Now, let us show that $a_i\leq\bar{a}_i$ for $i=1,2,\cdots,B$ by induction. When $i=1$, applying the definitions of $a_i$ and $\bar{a}_i$ yields $a_1=\bar{a}_1$. Assume that when $i=1,2,\cdots,k-1$, $a_i\leq \bar{a}_i$ holds true. Now, consider $\bar{a}_k$. Since $\nu,D_1,a_i\geq 0$ for any $i$, by using (\ref{eq:cl4-5}) and (\ref{eq:cl4-6}) we have
\begin{align*}
\bar{a}_k=\nu \left((k-1)D_1+\sum_{j=1}^{k-1}\bar{a}_j  \right)\geq \nu \left((k-1)D_1+\sum_{j=1}^{k-1}a_j  \right)\geq a_k.
\end{align*}

Therefore, $S_l\leq \bar{S}_l$, where we define $S_l=\sum_{i=1}^l a_i$ and $\bar{S}_l=\sum_{i=1}^l \bar{a}_i$. Summing up all the recursive equations for $\bar{a}_i$ in (\ref{eq:cl4-6}) up to $l$ implies
\begin{align}
\label{eq:cl4-7}
\bar{S}_l=\frac{l(l-1)}{2}\nu D_1+\nu \left(\bar{S}_1+\cdots+\bar{S}_{l-1}  \right)
\end{align} 
and
\begin{align*}
\bar{S}_{l+1}-\bar{S}_l=l\nu D_1+\nu \bar{S}_l,
\end{align*}
which in turn yields
\begin{align*}
\frac{\bar{S}_{l+1}}{\left(1+\nu \right)^{l+1}}-\frac{\bar{S}_l}{\left(1+\nu\right)^l}=\frac{l\nu D_1}{(1+\nu)^{l+1}}.
\end{align*}
By summing up all increments for $l=1,\cdots,B-1$, we obtain
\begin{align*}
\frac{\bar{S}_{B}}{\left(1+\nu\right)^{B}}=\frac{\bar{S}_{B}}{\left(1+\nu\right)^{B}}-\frac{\bar{S}_1}{\left(1+\nu\right)}=\frac{\nu D_1}{(1+\nu)}\sum_{l=1}^{B}\frac{l}{(1+\nu)^{l}}=\frac{D_1\left[(1+\nu)^{B}-1-\nu B\right]}{\nu (1+\nu)^{B}},
\end{align*}
which in turn yields
\begin{align*}
\sum_{i=1}^{B}\E\left[\left.\left\|v^{t,i}\right\|\right|\mathcal{F}^t\right]=\sum_{i=1}^{B}a_i=S_{B}\leq\bar{S}_{B}=\frac{D_1\left[(1+\nu)^{B}-1-\nu B\right]}{\nu}.
\end{align*}
Since $\binom{B}{l}=\frac{B!}{l!(B-l)!}\leq B^l$, we obtain
\begin{align*}
(1+\nu)^{B}=\sum_{l=0}^{B}\binom{B}{l}\nu^l\leq \sum_{l=0}^{B} (B\nu)^l,
\end{align*}
which in turn yields
\begin{align}
\label{eq:cl4-9}
\sum_{i=1}^{B}\E\left[\left.\left\|v^{t,i}\right\|\right|\mathcal{F}^t\right]\leq \frac{D_1\sum_{l=2}^{B}(B\nu)^l}{\nu}.
\end{align}
Substituting the definitions of $\nu$ and $C$ back into (\ref{eq:cl4-9}) gives
\begin{align}
\label{eq:cl4-10}
\sum_{i=1}^{B}\E\left[\left.\left\|v^{t,i}\right\|\right|\mathcal{F}^t\right]\leq \frac{\E\left[\left.\left\|\mu^t\right\|\right|\mathcal{F}^t\right]\sum_{l=2}^{B} \left(BL\gamma QP\right)^l}{L\gamma QP}.
\end{align}
By using the conditional Jensen's inequality and (\ref{eq:cl3-02}) we get
\begin{align}
\label{eq:cl4-11}
\E\left[\left.\left\|\mu^t\right\| \right|\mathcal{F}^t \right]=\E\left[\left.\sqrt{\left\|\mu^t\right\|^2 }\right|\mathcal{F}^t \right]\leq \sqrt{\E\left[\left.\left\|\mu^t\right\|^2 \right|\mathcal{F}^t \right]}=\sqrt{\hat{\mathcal{O}}(1)}=\hat{\mathcal{O}}(1).
\end{align}
The last equality holds due to the property that $\gamma\leq 1$. Moreover, since $BL\gamma QP\leq1$, we have
\begin{align}
\label{eq:cl4-12}
\sum_{l=2}^{B}\left(BL\gamma QP\right)^l=B\cdot\hat{\mathcal{O}}(B^2\gamma^2)=\hat{\mathcal{O}}(B^3\gamma^2).
\end{align}
Substituting (\ref{eq:cl4-11}) and (\ref{eq:cl4-12}) in (\ref{eq:cl4-10}) implies the claim in (\ref{eq:cl4-0}).

Now, let us proceed to find an upper bound for $\sum_{i=1}^{B}\E\left[\left.\left\| v^{t,i} \right\|^2\right|\mathcal{F}^t\right]$. From (\ref{eq:cl3-3}), we have
\begin{align*}
&\E\left[\left.\left\| v^{t,1} \right\|^2\right|\mathcal{F}^t\right]=0\\
&\E\left[\left.\left\| v^{t,i} \right\|^2\right|\mathcal{F}^t\right]\leq iL^2\gamma^2QP\left((i-1)^2\E\left[\left.\left\|\mu^t\right\|^2\right|\mathcal{F}^t \right]+\sum_{j=1}^{i-1}\E\left[\left.\left\| v^{t,j} \right\|^2\right|\mathcal{F}^t\right]  \right).
\end{align*}
Let us define 
\begin{align*}
&b_i:=\E\left[\left.\left\| v^{t,i} \right\|^2\right|\mathcal{F}^t\right]\\
&\theta:=L^2\gamma^2 QP\\
&D_2:=\E\left[\left.\left\| \mu^t \right\|^2\right|\mathcal{F}^t\right].
\end{align*}
Then the recursive formula becomes
\begin{align}
&b_1=0\nonumber\\
\label{eq:cl4-13}
&b_i\leq i\theta\left((i-1)^2D_2+\sum_{j=1}^{i-1}b_j  \right),
\end{align}
for $i=2,3,\cdots,B$. Let us define $\bar{b}_i$ as
\begin{align}
\label{eq:cl4-14}
\bar{b}_i=
\left\{\begin{matrix}
0,&i=1\\ 
i\theta\left((i-1)^2D_2+\sum_{j=1}^{i-1}\bar{b}_j\right),&i\neq 1.
\end{matrix}\right.
\end{align}
As before we derive $b_i\leq\bar{b}_i$ for $i=1,2,\cdots,B$. Therefore, $\mathcal{S}_l\leq \bar{\mathcal{S}}_l$, where we define $\mathcal{S}_l=\sum_{i=1}^l b_i$ and $\bar{\mathcal{S}}_l=\sum_{i=1}^l\bar{b}_i$. Summing up all the recursive equations for $\bar{b}_i$ in (\ref{eq:cl4-14}) up to $l$ implies
\begin{align}
\label{eq:cl4-15}
\bar{\mathcal{S}}_l=\theta D_2 \sum_{i=1}^{l-1}(i+1)i^2+\theta \sum_{i=1}^{l-1} (i+1)\bar{\mathcal{S}}_i,
\end{align}
and
\begin{align*}
\bar{\mathcal{S}}_{l+1}-\bar{\mathcal{S}}_l=\theta D_2 (l+1)l^2+\theta (l+1)\bar{\mathcal{S}}_l,
\end{align*}
which in turn yields
\begin{align*}
\frac{\bar{\mathcal{S}}_{l+1}}{\Pi_{i=1}^{l+1}\left(1+i\theta\right)}-\frac{\bar{\mathcal{S}}_{l}}{\Pi_{i=1}^{l}\left(1+i\theta\right)}=\frac{\theta D_2 (l+1)l^2}{\Pi_{i=1}^{l+1}\left(1+i\theta\right)}.
\end{align*}
By summing up all increments for $l=1,2,\cdots,B-1$, we obtain
\begin{align*}
\frac{\bar{\mathcal{S}}_{B}}{\Pi_{i=1}^{B}\left(1+i\theta\right)}=\frac{\bar{\mathcal{S}}_{B}}{\Pi_{i=1}^{B}\left(1+i\theta\right)}-\frac{\bar{\mathcal{S}}_1}{\left(1+\theta\right)}=\theta D_2\left[\sum_{l=1}^{B-1}\frac{(l+1)l^2}{\Pi_{i=1}^{l+1}(1+i\theta)}\right],
\end{align*}
which in turn yields
\begin{align}
\label{eq:cl4-17}
\sum_{i=1}^{B}\E\left[\left.\left\|v^{t,i}\right\|^2\right|\mathcal{F}^t\right]&=\sum_{i=1}^{B}b_i=\mathcal{S}_l\leq\bar{\mathcal{S}}_{B}=\theta D_2\left[\sum_{l=1}^{B-1}\left((l+1)l^2\Pi_{i=l+2}^{B}(1+i\theta) \right)\right]\nonumber\\
&\leq \theta D_2B(B-1)^2\left[\sum_{l=1}^{B-1}\left(\Pi_{i=l+2}^{B}(1+i\theta) \right)\right],
\end{align}
where we denote $\Pi_{i=B}^{B+1}(1+i\theta)=1$.
Since
\begin{align*}
\sum_{l=1}^{B-1}\left(\Pi_{i=l+2}^{B}(1+i\theta) \right)\leq \sum_{l=1}^{B-1}\left(\Pi_{i=1}^{B}(1+i\theta) \right)\leq \sum_{l=1}^{B-1}\left(1+B\theta \right)^{B}=(B-1)\left(1+B\theta \right)^{B},
\end{align*}
it in turn yields
\begin{align}
\sum_{i=1}^{B}\E\left[\left.\left\|v^{t,i}\right\|^2\right|\mathcal{F}^t\right]&\leq \theta D_2B(B-1)^3(1+B\theta)^{B}.
\label{eq:cl4-18}
\end{align}
Substituting the definitions of $\theta$ and $D_2$ back into (\ref{eq:cl4-18}) gives
\begin{equation}
\begin{split}
\label{eq:cl4-19}
\sum_{i=1}^{B}\E\left[\left.\left\|v^{t,i}\right\|^2\right|\mathcal{F}^t\right]\leq L^2\gamma^2QP\E\left[\left.\left\|\mu^t\right\|^2 \right|\mathcal{F}^t \right]B^4(BL^2\gamma^2QP+1)^{B}.
\end{split}
\end{equation}
Since  $BL\gamma QP\leq 1$ and $QP\geq 1$, we conclude
\begin{align*}
B^2L^2\gamma^2 QP\leq (BL\gamma QP)^2\leq 1,
\end{align*}
which in turn yields
\begin{align}
(1+BL^2\gamma^2 QP)^{B}&=1+\sum_{i=1}^{B}\binom{B}{i}(BL^2\gamma^2 QP)^i\leq 1+\sum_{i=1}^{B}B^i(BL^2\gamma^2 QP)^i\nonumber\\
&=1+\sum_{i=1}^{B}(B^2L^2\gamma^2 QP)^i=1+\sum_{i=1}^{B}\hat{\mathcal{O}}(B^2L^2\gamma^2 QP)\nonumber\\
&=1+\hat{\mathcal{O}}(B^3L^2\gamma^2 QP).
\label{eq:cl4-20}
\end{align}
Combining (\ref{eq:cl3-02}) and (\ref{eq:cl4-20}) gives
\begin{align*}
\sum_{i=1}^{B}\E\left[\left.\left\|v^{t,i}\right\|^2\right|\mathcal{F}^t\right]=\hat{\mathcal{O}}(B^4\gamma^2)(1+\hat{\mathcal{O}}(B^3L^2\gamma^2 QP))=\hat{\mathcal{O}}(B^4\gamma^2)+\hat{\mathcal{O}}(B^7\gamma^4).
\end{align*}
This completes the proof of the claim.
\end{proof}
By using (\ref{eq:lem2-0}), (\ref{eq:pro1-6}), (\ref{eq:pro1-7}), (\ref{eq:pro1-8}) and Lipschitz continuity of $\triangledown F(\omega)$ we have
\begin{align*}
&\E\left[F(\omega^{t+1})|\mathcal{F}^t\right]\leq F(\omega^t)+\triangledown F(\omega^t)^T\E\left[\left(\omega^{t+1}-\omega^t\right)|\mathcal{F}^t\right]+\frac{L}{2}\E \left[\left\| \omega^{t+1}-\omega^t\right\|^2|\mathcal{F}^t\right]\\
&=F(\omega^{t})+\triangledown F(\omega^t)^T\left\{ -\gamma B\frac{\mathcal{c}^t}{d}\triangledown F(\omega^t) -\gamma \sum_{i=1}^{B}\expect[\big]{v^{t,i}|\mathcal{F}^t} \right\} \\
&\quad\quad\quad\quad +\frac{L}{2}\gamma^2(B+1)\left\{B^2\E\left[\left\| \mu^t\right\|^2|\mathcal{F}^t\right]+\sum_{i=1}^{B}\E\left[\left\| v^{t,i}\right\|^2|\mathcal{F}^t\right]\right\}\\
&\leq F(\omega^t)-\gamma \frac{ B}{d} \left\| \triangledown F(\omega^t)\right\|^2+\gamma \left\|\triangledown F(\omega^t)\right\| \sum_{i=1}^{B} \E\left[\left\|v^{t,i}\right\||\mathcal{F}^t\right]\\
&\quad\quad\quad\quad+\frac{L}{2}\gamma^2(B+1)\left\{B^2\E\left[\left\| \mu^t\right\|^2|\mathcal{F}^t\right]+\sum_{i=1}^{B}\E\left[\left\| v^{t,i}\right\|^2|\mathcal{F}^t\right]\right\}\\
&\leq F(\omega^t)-\gamma  \frac{B}{d} \left\| \triangledown F(\omega^t)\right\|^2+\hat{\mathcal{O}}(B^3\gamma^2)+\hat{\mathcal{O}}(B\gamma^2)\left\{\hat{\mathcal{O}}(B^2)+\hat{\mathcal{O}}(B^4\gamma^2) +\hat{\mathcal{O}}(B^7\gamma^4) \right\}\\
&=F(\omega^t)-\gamma  \frac{B}{d} \left\| \triangledown F(\omega^t)\right\|^2+\hat{\mathcal{O}}(B^3\gamma^2)+\hat{\mathcal{O}}(B^5\gamma^4)+\hat{\mathcal{O}}(B^8\gamma^6).
\end{align*}   
 Since $LQP\geq 1$ and $BL\gamma QP\leq 1$, the above equation becomes
\begin{align}
\label{eq:cl4-21}
\E\left[F(\omega^{t+1})|\mathcal{F}^t\right]\leq F(\omega^t)-\gamma  \frac{B}{d} \left\| \triangledown F(\omega^t)\right\|^2+\hat{\mathcal{O}}(B^4\gamma^2).
\end{align}
Subtracting $F(\omega^*)$ from both sides of (\ref{eq:cl4-21}) and applying (\ref{eq:pro1-14}) yields the claim in  (\ref{eq:pro2-0}), where $C_2$ is a positive constant. 
\end{proof}
\subsubsection*{Proof of Theorem \ref{thm:7}}
\label{proof:thm7}
\begin{proof}
We use the relationship in (\ref{eq:pro2-0}) to construct a supermartingale sequence. Define the stochastic processes $\alpha^t$ and $\beta^t$ as 
\begin{align}
\label{eq:thm3-2}
&\alpha^t:=\left[F(\omega^t)-F(\omega^*)\right]\times \mathds{1}_{\left\{\min_{u\leq t}F(\omega^u)-F(\omega^*)>\frac{C_2dB^3\gamma}{2\xi}\right\}}\\
&\beta^t:=\frac{2\xi B}{d}\gamma\left[F(\omega^t)-F(\omega^*)-\frac{C_2dB^3\gamma}{2\xi}\right]\times\mathds{1}_{\left\{\min_{u\leq t}F(\omega^u)-F(\omega^*)>\frac{C_2dB^3\gamma}{2\xi}\right\}}.
\label{eq:thm3-3}
\end{align}
The process $\alpha^t$ tracks the optimality gap $F(\omega^t)-F(\omega^*)$ until the gap becomes smaller than $\frac{C_2dB^3\gamma}{2\xi}$ for the first time. Notice that the stochastic process $\alpha^t$ is never negative. Likewise, the same properties hold for $\beta^t$. Based on the relationship in (\ref{eq:pro2-0}) and the definitions of stochastic processes $\alpha^t$ and $\beta^t$ in (\ref{eq:thm3-2}) and (\ref{eq:thm3-3}), we obtain that for all $t\geq 0$
\begin{align}
\label{eq:thm3-4}
\E\left[\left.\alpha^{t+1} \right|\mathcal{F}^t  \right]\leq \alpha^t-\beta^t.
\end{align}
Given the relationship in (\ref{eq:thm3-4}) and non-negativity of stochastic processes $\alpha^t$ and $\beta^t$ we obtain that $\alpha^t$ is supermartingale. The supermartingale convergence theorem yields
\begin{align}
&\mathrm{(i)}\quad \alpha^t \mathrm{\,\,converges\,\, to\,\, a\,\, limit\,\, a.s.,\, and }\nonumber\\
\label{eq:thm3-5}
&\mathrm{(ii)}\quad \sum_{t=1}^{\infty}\beta^t<\infty. \quad \mathrm{a.s.}
\end{align}
Property (\ref{eq:thm3-5}) implies that the sequence $\beta^t$ is converging to 0 almost surely, i.e.,
\begin{align}
\label{eq:thm3-6}
\lim\limits_{t\to\infty}\beta^t=0\quad\mathrm{a.s.}
\end{align}
Based on the definition of $\beta^t$ in (\ref{eq:thm3-3}), the limit in (\ref{eq:thm3-6}) is true if one of the following events holds:
\begin{align*}
&\mathrm{(i)}\quad  \mathrm{\,\,the\,\, indicator\,\, function\,\, is\,\,0 \, \, after \,\,large\,\, }t,\\
&\mathrm{(ii)}\quad \lim\limits_{t\to\infty}F(\omega^t)-F(\omega^*)- \frac{C_2dB^3\gamma}{2\xi}=0. 
\end{align*}
From either one of these two events we conclude that
\begin{align}
\label{eq:thm3-7}
\liminf_{t\to\infty}F(\omega^t)-F(\omega^*)\leq \frac{C_2dB^3\gamma}{2\xi}\quad\mathrm{a.s.}
\end{align}
Therefore, the claim in (\ref{eq:thm7-0}) is valid. The result in (\ref{eq:thm3-7}) shows that the loss function value sequence $F(\omega^t)$ almost sure converges to a neighborhood of the optimal loss function value $F(\omega^*)$.

We proceed to prove the result in (\ref{eq:thm7-1}). We compute the expected value of (\ref{eq:pro2-0}) given $\mathcal{F}^0$ to obtain
\begin{align}
\label{eq:thm3-8}
\E\left[F(\omega^{t+1})-F(\omega^*)\right]\leq \left( 1-\frac{2\xi B}{d}\gamma \right)\E\left[ F(\omega^t)-F(\omega^*)\right]+C_2B^4\gamma^2.
\end{align}
Rewriting the relationship in (\ref{eq:thm3-8}) for step $t-1$ gives
\begin{align}
\label{eq:thm3-9}
\E\left[F(\omega^{t})-F(\omega^*)\right]\leq \left( 1-\frac{2\xi B}{d}\gamma \right)\E\left[ F(\omega^{t-1})-F(\omega^*)\right]+C_2B^4\gamma^2.
\end{align}
Substituting the upper bound in (\ref{eq:thm3-9}) for the expectation of $F(\omega^t)-F(\omega^*)$ in (\ref{eq:thm3-8}) implies
\begin{align}
\label{eq:thm3-10}
\E\left[F(\omega^{t+1})-F(\omega^*)\right]\leq \left( 1-\frac{2\xi B}{d}\gamma \right)^2\E\left[ F(\omega^{t-1})-F(\omega^*)\right]+C_2B^4\gamma^2\left(1+\left( 1-\frac{2\xi B}{d}\gamma \right)\right).
\end{align}
By recursively applying steps (\ref{eq:thm3-9}) and (\ref{eq:thm3-10}) we can bound the expectation of $F(\omega^{t+1})-F(\omega^*)$ in terms of the initial loss function error $F(\omega^0)-F(\omega^*)$ as
\begin{align}
\label{eq:thm3-11}
\E\left[F(\omega^{t+1})-F(\omega^*)\right]\leq \left( 1-\frac{2\xi B}{d}\gamma \right)^{t+1}\left[ F(\omega^{0})-F(\omega^*)\right]+C_2B^4\gamma^2\sum_{u=0}^t \left( 1-\frac{2\xi B}{d}\gamma \right)^u.
\end{align}
Substituting $t$ by $t-1$ and simplifying the sum in the right-hand side of (\ref{eq:thm3-11}) yields
\begin{align}
\label{eq:thm3-12}
\E\left[F(\omega^{t})-F(\omega^*)\right]\leq \left( 1-\frac{2\xi B}{d}\gamma \right)^{t}\left[ F(\omega^{0})-F(\omega^*)\right]+\frac{C_2dB^3\gamma}{2\xi}\left[1-\left( 1-\frac{2\xi B}{d}\gamma \right)^t\right].
\end{align}
Since $\gamma< \frac{d}{2\xi B}$, the term $1-\left( 1-\frac{2\xi B}{d}\gamma \right)^t$ in the right-hand side of (\ref{eq:thm3-12}) is strictly smaller than 1 and the claim in (\ref{eq:thm7-1}) follows.
\end{proof}
\subsection*{E\tab Counter Example without Assumption \hyperref[as:5]{5}}
\subsubsection*{Proof of Theorem \ref{thm:8}}
\label{proof:thm8}
\begin{proof}
We consider the setting where there are two samples $[A|b] = \left[ \begin{matrix}
a_{11} &a_{12} \\ 
 a_{21}&a_{22} 
\end{matrix}\left| \begin{matrix}
b_1\\ 
b_2
\end{matrix}\right. \right ]$, which are split into four partitions, and the parameter vector is specified as $\omega = [\omega_1,\omega_2]$. Then, applying MSE and linear regression yields 
\begin{align}
    F([\omega_1,\omega_2])&= \frac{1}{2}\left\|A\omega-b\right\|_2^2.
    \label{eq:thm8-0}
\end{align}
Consequently, the gradient is
\begin{align*}
    \bigtriangledown F = A^{T}A\omega-A^Tb,
\end{align*}
and the Hessian $H=A^TA$. 
Note that Assumption \hyperref[as:2]{2} holds. For Assumption \hyperref[as:3]{3}, we can select $L = \max_{i,j}\left|H_{ij}\right|$. Notice that
\begin{align}
\left\|A^TA\right\|=\left\|A\right\|^2\geq\left(\frac{1}{\sqrt{2}}\left\|A\right\|_F\right)^2=\frac{a_{11}^2+a_{12}^2+a_{21}^2+a_{22}^2}{2}.
\label{eq:thm8-1}
\end{align}
Let us consider the inner loop in steps~\ref{inner_start}-\ref{inner_end}. The approximate individual loss functions using each sample are
\begin{align*}
    f_1(\omega_1)=\frac{1}{2}(\omega_1a_{11}-b_1)^2, \quad f_1(\omega_2)=\frac{1}{2}(\omega_2a_{12}-b_1)^2,\\ f_2(\omega_1)=\frac{1}{2}(\omega_1a_{21}-b_2)^2, \quad f_2(\omega_2)=\frac{1}{2}(\omega_2a_{22}-b_2)^2,
\end{align*}
which in turn yields
\begin{align*}
    \bigtriangledown_{\omega_1}f_1(\omega_1a_{11})=a_{11}^2\omega_1-a_{11}b_1, \quad
    \bigtriangledown_{\omega_2}f_1(\omega_2a_{12})=a_{12}^2\omega_2-a_{12}b_1,\\
    \bigtriangledown_{\omega_1}f_2(\omega_1a_{21})=a_{21}^2\omega_1-a_{21}b_2, \quad
    \bigtriangledown_{\omega_2}f_2(\omega_2a_{22})=a_{22}^2\omega_2-a_{22}b_2.
\end{align*}
The update formulas in the algorithm for the first sample and $\omega_1$ read
\begin{align*}
    &\bar{\omega}^{(i+1)}_1=\bar{\omega}^{(i)}_1-\gamma_{t+1}\left(a_{11}^2\bar{\omega}^{(i)}_1-a_{11}^2\omega^t_1+\mu^t_1 \right)\\
    &\bar{\omega}^{(i+1)}_1=\left(1-a_{11}^2\gamma_{t+1} \right)\bar{\omega}^{(i)}_1+\gamma_{t+1}\left(a_{11}^2\omega^t_1-\mu^t_1\right)\\
    &\bar{\omega}^{(i+1)}_1-\left(\omega^t_1-\frac{\mu^t_1}{a_{11}^2}\right)=\left(1-a_{11}^2\gamma_{t+1}\right)\left(\bar{\omega}^{(i)}_1-\left(\omega^t_1-\frac{\mu^t_1}{a_{11}^2}\right)\right).
\end{align*}
Therefore,
\begin{align*}
    \bar{\omega}^{(i+1)}_1=\left(1-a_{11}^2\gamma_{t+1}\right)^{i+1}\left(\bar{\omega}^{(0)}_1-\left(\omega^t_1-\frac{\mu^t_1}{a_{11}^2}\right)\right)+\left(\omega^t_1-\frac{\mu^t_1}{a_{11}^2}\right).
\end{align*}
Since $\bar{\omega}^{(0)}_1 = \omega^t_1$ due to step~\ref{transfer}, we obtain
\begin{align*}
    \bar{\omega}^{(i+1)}_1 = \left(1-a_{11}^2\gamma_{t+1}\right)^{i+1}\frac{\mu^t_1}{a_{11}^2}+\left(\omega^t_1-\frac{\mu^t_1}{a_{11}^2}\right).
\end{align*}
If $\gamma_t\leq \frac{1}{\min\left\{a_{11}^2,a_{12}^2,a_{21}^2,a_{22}^2 \right\}}$, then $a_{11}^2\gamma_{t}<1$, which in turn yields 
\begin{align*}
    \lim\limits_{i\to\infty}\bar{\omega}^{(i+1)}_1 =\omega^t_1-\frac{\mu^t_1}{a_{11}^2}.
\end{align*}
Thus, if the number of iterations $B$ for the inner loop is big enough, then, approximately, $\omega^{t+1}_1=\omega^t_1-\frac{\mu^t_1}{a_{11}^2}$. Similarly, when using the same data point to update $\omega_2$, we obtain
\begin{align*}
    \omega^{t+1}_2=\omega^t_2-\frac{\mu^t_2}{a_{12}^2}.
\end{align*}
When using $([a_{21},a_{22}],b_2)$ to update $\omega_1$ and $\omega_2$, we obtain
\begin{align*}
    \omega^{t+1}_1=\omega^t_1-\frac{\mu^t_1}{a_{21}^2}, \quad \omega^{t+1}_2=\omega^t_2-\frac{\mu^t_2}{a_{22}^2}.
\end{align*}
Therefore, the inner loop mimics gradient descent with constant learning rate. Then, the explicit update formula for $\omega^t$ is
\begin{align}
    \omega^{t+1} = \omega^{t}-\eta\bigtriangledown F(\omega^t)=(I-\eta A^TA)\omega^t+\eta A^Tb,
    \label{eq:thm8-goal}
\end{align}
where the second equality holds due to the definition of $\bigtriangledown F$, and $\eta=\begin{bmatrix}
 \frac{1}{a_{11}^2}& 0 \\ 
0 & \frac{1}{a_{12}^2}
\end{bmatrix}$ when using the first sample and $\eta=\begin{bmatrix}
 \frac{1}{a_{21}^2}& 0\\ 
0 & \frac{1}{a_{22}^2}
\end{bmatrix}$ when using the second sample. Thus, if $\left|a_{11}\right|=\left|a_{12}\right|=\min\left\{\left|a_{11}\right|,\left|a_{12}\right|,\left|a_{21}\right|,\left|a_{22}\right|\right\}$ and $\max\left\{\left|a_{21}\right|,\left|a_{22}\right|\right\}> \left|a_{11}\right|$, then,
when using $([a_{11},a_{12}],b_1)$ to update $\omega_1$, the corresponding $\left\|I-\eta A^TA\right\|>1$ since $\left\|\eta A^TA\right\|\geq \frac{\left\|A^TA\right\|}{\left\|\eta^{-1}\right\|}> \left|\frac{a_{11}^2+a_{12}^2+a_{21}^2+a_{22}^2}{2a_{11}^2}\right|=2$, which in turn yields the divergence of the algorithm. Similarly, the same conclusion applies to $([a_{21}, a_{22}],b_2)$ when $\left|a_{21}\right|=\left|a_{22}\right|=\min\left\{\left|a_{11}\right|,\left|a_{12}\right|,\left|a_{21}\right|,\left|a_{22}\right|\right\}$ and $\max\left\{\left|a_{11}\right|,\left|a_{12}\right|\right\}> \left|a_{21}\right|$.
Some possible values of $A$ and $b$ are in Table \ref{counter example}.
\begin{table}[H]
\centering
\begin{tabular}{|l|l|l|l|l|l|l|l|}
\hline
 $a_{11}$&$a_{12}$  & $b_{1}$ &$a_{21}$  &$a_{22}$  &$b_2$  &optimal $\omega^*$  &$\omega^{100}$  \\ \hline
 1 & 1 & 1 & 2 & 3 & 0 &$[3,-1]$  &$[3.651 \times 10^{55}, -6.811\times 10^{56}]$  \\ \hline
  2 & 1 & 1 & 1 & 1 & 0 &$[1,-1]$  &$[54.606,-29.148]$  \\ \hline
1 & 2 & 1 & 1 & 3 & 0 &$[3,-1]$  &$[-4.414 \times 10^{11}, -8.765\times 10^{11}]$  \\ \hline
 1&2  &1  &2  &3  &1  &$[-27,17]$  &$[4.973\times 10^{29},-3,455\times 10^{30}]$  \\ \hline
 1& 4 & 1 & 2 &3  &0  &$[-\frac{3}{5},\frac{2}{5}]$  &$[-177.419,2976.815]$  \\ \hline
\end{tabular}
\caption{Counter Examples}
\label{counter example}
\end{table}
\end{proof}

\subsection*{F\tab Diminishing Learning Rate Convergence with Feature Sampling}
We assume that $\omega^{*}$ is the unique optimal solution to \hyperref[obj]{(1)}, and also that $\left\|\omega^*\right\|\leq \frac{M_2}{2}$. By using Assumption \hyperref[as:5]{5}, we conclude that the distance between any $\omega\in\Omega$ and $\omega^{*}$ is bounded, i.e.
\begin{align}
\left\|\omega^t-w^*\right\|\leq M_2.
\label{eq:bd w}
\end{align}
The second moment of $\omega^t$ is also bounded for all $t$, i.e. 
\begin{align}
\left\|\omega^t\right\|^2\leq \frac{M_2^2}{4},
\label{eq:bd w2}
\end{align}
for any $t$. Let us define $\mu^t=\frac{1}{\mathcal{d}^t}\sum_{j\in\mathcal{D}^t}\bar{\triangledown}_{\omega_{\mathcal{C}^t}}f_j(x_j\omega^t)+e^t$, where $e^t$ is defined as
\begin{align}
\label{eq:def_e}
e^t&:=\frac{1}{\mathcal{d}^t}\sum_{j\in\mathcal{D}^t}\bar{\triangledown}_{\omega_{\mathcal{C}^t}}f_j(x_j^{\mathcal{B}^t}\omega_{\mathcal{B}^t}^t)-\frac{1}{\mathcal{d}^t}\sum_{j\in\mathcal{D}^t}\bar{\triangledown}_{\omega_{\mathcal{C}^t}}f_j(x_j\omega^t).
\end{align}
\begin{lem}
\label{lem:3}
If Assumptions \hyperref[as:3]{3}, \hyperref[as:4]{4} and \hyperref[as:5]{5} hold true, then $\left\|\triangledown F(\omega^t)\right\|$ and $\sum_{j=1}^N\left\|\triangledown f_j(x_j\omega^t)\right\|^2$ for any $t$ satisfy
\begin{align}
\label{eq:lem3-0}
\left\|\triangledown F(\omega^t)\right\|&\leq M_2L,\\
\label{eq:lem3-1}
\sum_{j=1}^N\left\|\triangledown f_j(x_j\omega^t)\right\|^2&\leq(N-1)G^2+NM_2^2L^2.
\end{align}
\end{lem}
\begin{proof}
Using the fact that $\omega^{*}$ is the optimal solution and Assumptions \hyperref[as:3]{3} and \hyperref[as:5]{5} we obtain
\begin{align*}
\left\|\triangledown F(\omega^t)\right\|=\left\|\triangledown F(\omega^t)-\triangledown F(\omega^{*})\right\|\leq L\left\|\omega^t-\omega^*\right\|\leq LM_2.
\end{align*}
The last inequality holds due to (\ref{eq:bd w}). Assumption \hyperref[as:4]{4} implies that, for any $\omega^t$ we have
\begin{align*}
\frac{1}{N-1}\sum_{j=1}^N\left(\left\|\triangledown f_j(x_j\omega^t)\right\|^2-\left\|\triangledown F(\omega^t) \right\|^2  \right)\leq G^2,
\end{align*}
which in turn yields
\begin{align*}
\frac{1}{N-1}\sum_{j=1}^N\left\|\triangledown f_j(x_j\omega^t)\right\|^2\leq G^2+\frac{N}{N-1}\left\|\triangledown F(\omega^t)\right\|^2 ,
\end{align*}
and
\begin{align}
\sum_{j=1}^N\left\|\triangledown f_j(x_jw^t)\right\|^2\leq (N-1)G^2+N\left\|\triangledown F(\omega^t)\right\|^2.
\label{eq:lem3-2}
\end{align}
Inserting (\ref{eq:lem3-0}) into (\ref{eq:lem3-2}) gives (\ref{eq:lem3-1}).
\end{proof}
\begin{claim}
\label{cl:e}
If Assumption \hyperref[as:3]{3} and \hyperref[as:5]{5} hold true and for some constant $\eta\geq 0$ and the learning rate $\gamma_{t+1}$, we have
\begin{align}
\mathcal{b}^t\in \left[ \max\left\{\mathcal{c}^t,\frac{d}{1+\frac{4d\eta\gamma_{t+1}^2}{\mathcal{c}^tM_2^2L^2}}\right\},d\right]
\label{eq:cl1-0}
\end{align}
for every $t$, then the expected value of $\left\|e^t\right\|^2$ conditioned on $\mathcal{F}^t$ generated by SODDA is bounded by $\eta\gamma_{t+1}^2$, that is
\begin{align*}
\E\left[ \left\|e^t\right\|^2 \vert \mathcal{F}^t\right]\leq \eta\gamma_{t+1}^2,
\end{align*}
where $\eta$ is a constant unrelated to $B$.
\end{claim}
\begin{proof}
 By using (\ref{eq:def_e}) we obtain
\begin{align}
 \E\left[ \left\|e^t\right\|^2 \vert \mathcal{F}^t\right]&=\E\left[ \left.\left\|\frac{1}{\mathcal{d}^t}\sum_{j\in\mathcal{D}^t}\bar{\triangledown}_{\omega_{\mathcal{C}^t}}f_j(x_j^{\mathcal{B}^t}\omega_{\mathcal{B}^t}^t)-\frac{1}{\mathcal{d}^t}\sum_{j\in\mathcal{D}^t}\bar{\triangledown}_{\omega_{\mathcal{C}^t}}f_j(x_j\omega^t)\right\|^2 \right| \mathcal{F}^t\right]\nonumber\\
&=\frac{1}{(\mathcal{d}^t)^2}\E\left[ \left.\left\|\sum_{j\in\mathcal{D}^t}\left(\bar{\triangledown}_{\omega_{\mathcal{C}^t}}f_j(x_j^{\mathcal{B}^t}\omega_{\mathcal{B}^t}^t)-\bar{\triangledown}_{\omega_{\mathcal{C}^t}}f_j(x_j\omega^t)\right)\right\|^2\right|\mathcal{F}^t\right]\nonumber\\
&\leq \frac{1}{\mathcal{d}^t}\E\left[ \left.\sum_{j\in\mathcal{D}^t}\left\|\bar{\triangledown}_{\omega_{\mathcal{C}^t}}f_j(x_j^{\mathcal{B}^t}\omega_{\mathcal{B}^t}^t)-\bar{\triangledown}_{\omega_{\mathcal{C}^t}}f_j(x_j\omega^t)\right\|^2\right|\mathcal{F}^t\right]\nonumber\\
&=\frac{1}{\mathcal{d}^t}\E\left[ \left.\sum_{j\in\mathcal{D}^t}\E\left[ \left.\left\|\bar{\triangledown}_{\omega_{\mathcal{C}^t}}f_j(x_j^{\mathcal{B}^t}\omega_{\mathcal{B}^t}^t)-\bar{\triangledown}_{\omega_{\mathcal{C}^t}}f_j(x_j\omega^t)\right\|^2\right|\mathcal{F}^t,\mathcal{B}^t,\mathcal{D}^t\right]\right|\mathcal{F}^t\right].
\label{eq:cl1-1}
\end{align}
Applying Lemma \ref{lem:1} with $w_1=1$, $w_2=0$,$\Phi=\left\{\left(\bar{\triangledown}_{\omega_{\mathcal{B}^t}}f_j(x_j^{\mathcal{B}^t}\omega_{\mathcal{B}^t}^t)-\bar{\triangledown}_{\omega_{\mathcal{B}^t}}f_j(x_j\omega^t)\right)_i\right\}_{i=1}^{\mathcal{b}^t}$, $g(z)=z^2$, $\mathcal{H}=\sigma(\mathcal{F}^t,\mathcal{B}^t,\mathcal{D}^t)$ and $\mathcal{B}=\mathcal{C}^t$ to (\ref{eq:cl1-1}) yields
\begin{align}
\E\left[ \left\|e^t\right\|^2 \vert \mathcal{F}^t\right]&\leq \frac{\mathcal{c}^t}{\mathcal{b}^t\mathcal{d}^t}\E\left[ \left.\sum_{j\in\mathcal{D}^t}\left\|\bar{\triangledown}_{\omega_{\mathcal{B}^t}}f_j(x_j^{\mathcal{B}^t}\omega_{\mathcal{B}^t}^t)-\bar{\triangledown}_{\omega_{\mathcal{B}^t}}f_j(x_j\omega^t)\right\|^2\right|\mathcal{F}^t\right]\nonumber\\
&= \frac{\mathcal{c}^t}{\mathcal{b}^t\mathcal{d}^t}\E\left[ \left.\sum_{j\in\mathcal{D}^t}\left\|\bar{\triangledown}_{\omega_{\mathcal{B}^t}}f_j(x_j\omega_{(\mathcal{B}^t,0)}^t)-\bar{\triangledown}_{\omega_{\mathcal{B}^t}}f_j(x_j\omega^t)\right\|^2\right|\mathcal{F}^t\right]\nonumber\\
&\leq \frac{L^2\mathcal{c}^t}{\mathcal{b}^t\mathcal{d}^t}  \E\left[ \left.\sum_{j\in\mathcal{D}^t}\left\|\omega_{(\mathcal{B}^t,0)}^t-\omega^t\right\|^2\right|\mathcal{F}^t\right]\nonumber\\
&=\frac{L^2\mathcal{c}^t}{\mathcal{b}^t\mathcal{d}^t}\E\left[\left.\E\left[\left.\sum_{j\in\mathcal{D}^t}\left\|\omega_{[d]\setminus\mathcal{B}^t}^t\right\|^2\right|\mathcal{F}^t,\mathcal{B}^t\right]\right|\mathcal{F}^t\right]\nonumber\\
&=\frac{L^2\mathcal{c}^t}{\mathcal{b}^t\mathcal{d}^t}\E\left[\left.\E\left[\left.\mathcal{d}^t\left\|\omega_{[d]\setminus\mathcal{B}^t}^t\right\|^2\right|\mathcal{F}^t,\mathcal{B}^t\right]\right|\mathcal{F}^t\right]\nonumber\\
&=\frac{L^2\mathcal{c}^t}{\mathcal{b}^t}\E\left[\left.\left\|\omega_{[d]\setminus\mathcal{B}^t}^t\right\|^2\right|\mathcal{F}^t\right]\nonumber.
\end{align}
The second inequality uses Assumption \hyperref[as:3]{3} and we use $[d]=\left\{1,\cdots,d\right\}$. Now, let us use Lemma \ref{lem:1} again with $w_1=0$, $w_2=1$, $\Phi=\left\{\left(\omega^t\right)_i\right\}_{i=1}^{d}$, $g(z)=z^2$, $\mathcal{H}=\mathcal{F}^t$ and $\mathcal{B}=\mathcal{B}^t$ to get
\begin{align*}
\E\left[ \left\|e^t\right\|^2 \vert \mathcal{F}^t\right]&\leq \frac{L^2\mathcal{c}^t}{\mathcal{b}^t} \frac{d-\mathcal{b}^t}{d}\left\|\omega^t\right\|^2\\
&\leq \frac{L^2\mathcal{c}^t(d-\mathcal{b}^t)}{d\mathcal{b}^t}\frac{M_2^2}{4}.
\end{align*} 
The last inequality uses (\ref{eq:bd w2}). In order to bound the expected value of $\left\|e^t\right\|^2$ by $\eta\gamma_{t+1}^2$, we require
\begin{align*}
\frac{L^2\mathcal{c}^t(d-\mathcal{b}^t)}{d\mathcal{b}^t}\frac{M_2^2}{4}&\leq \eta\gamma_{t+1}^2,
\end{align*}
which is equivalent to
\begin{align*}
\mathcal{b}^t\geq \frac{d}{1+\frac{4d\eta\gamma_{t+1}^2}{\mathcal{c}^tM_2^2L^2}}.
\end{align*}
Meanwhile, in order to make $\triangledown_{\omega_{\mathcal{C}^t}} f_j(x_j^{\mathcal{B}^t}\omega_{\mathcal{B}^t}^t)$ well defined, we require $\mathcal{b}^t\geq \mathcal{c}^t$.
\end{proof}
\noindent Next, similar to Proposition \ref{pro1},  we present the following proposition. 

\begin{pro}
\label{pro3}
If Assumptions \hyperref[as:2]{2}-\hyperref[as:5]{5} hold true, and the sequence of learning rates satisfies $\gamma_t\leq 1$ for all $t$, and the sequences $(\mathcal{b}^t,\mathcal{c}^t,\mathcal{d}^t)_{t=0}^{\infty}$ are selected so that (\ref{eq:cl1-0}) for some constant $\eta\geq 0$, $\mathcal{c}^t\leq d$ and $\mathcal{d}^t\leq N$, then the loss function error sequence $F(\omega^t)-F(\omega^*)$ generated by SODDA satisfies
\begin{align}
\E\left[F(\omega^{t+1})-F(\omega^*)|\mathcal{F}^t\right]\leq (1-\frac{2\xi B}{d}\gamma_{t+1})[F(\omega^t)-F(\omega^*)]+C_3\gamma_{t+1}^2,
\label{eq:pro3-0}
\end{align}
where $C_3$ is a positive constant.
\end{pro} 
\begin{proof}
Since the proof is very similar to the proof of Proposition \ref{pro1}, we point out the differences rather than present all the details.
\begin{claim}
\label{cl5}
For any $t$ we have
\begin{align}
\label{eq:cl5-1}
&\E\left[\left.\left\|\frac{1}{\mathcal{d}^t}\sum_{j\in\mathcal{D}^t}\bar{\triangledown}_{\omega_{\mathcal{C}^t}}f_j(x_j\omega^t)\right\|^2\right|\mathcal{F}^t\right]\leq\frac{\mathcal{c}^t}{Nd}\left[(N-1)G^2+NM_2^2L^2  \right].
\end{align}
\end{claim}
\begin{proof}
Applying the law of iterated expectation and Lemma \ref{lem:1} with $w_1=1$, $w_2=0$, $\Phi=\left\{ \bar{\triangledown}_{\omega_{\mathcal{C}^t}}f_j(x_j\omega^t)\right\}_{j=1}^N$, $g(z)=\left\|z\right\|^2$, $\mathcal{H}=\sigma(\mathcal{F}^t,\mathcal{c}^t)$ and $\mathcal{B}=\mathcal{D}^t$ give
\begin{align}
&\E\left[\left.\left\|\frac{1}{\mathcal{d}^t}\sum_{j\in\mathcal{D}^t}\bar{\triangledown}_{\omega_{\mathcal{C}^t}}f_j(x_j\omega^t)\right\|^2\right|\mathcal{F}^t\right]\leq \frac{1}{\mathcal{d}^t}\E\left[\sum_{j\in\mathcal{D}^t}\left.\left\|\bar{\triangledown}_{\omega_{\mathcal{C}^t}}f_j(x_j\omega^t)\right\|^2\right|\mathcal{F}^t  \right]\nonumber\\
&\quad=\frac{1}{\mathcal{d}^t}\E\left[\left.\E\left[\left.\sum_{j\in\mathcal{D}^t}\left\|\bar{\triangledown}_{\omega_{\mathcal{C}^t}}f_j(x_j\omega^t)\right\|^2\right|\mathcal{F}^t,\mathcal{C}^t\right]\right|\mathcal{F}^t  \right]
=\frac{1}{\mathcal{d}^t}\cdot \frac{\mathcal{d}^t}{N}\E\left[\sum_{j=1}^N\left.\left\|\bar{\triangledown}_{\omega_{\mathcal{C}^t}}f_j(x_j\omega^t)\right\|^2\right|\mathcal{F}^t  \right]\nonumber\\
&\quad=\frac{1}{N}\sum_{j=1}^N\E\left[\left.\left\|\bar{\triangledown}_{\omega_{\mathcal{C}^t}}f_j(x_j\omega^t)\right\|^2\right|\mathcal{F}^t  \right]\nonumber,
\end{align}
which in turn yields
\begin{align}
\E\left[\left.\left\|\frac{1}{\mathcal{d}^t}\sum_{j\in\mathcal{D}^t}\bar{\triangledown}_{\omega_{\mathcal{C}^t}}f_j(x_j\omega^t)\right\|^2\right|\mathcal{F}^t\right]\leq\frac{1}{N}\sum_{j=1}^N\frac{\mathcal{c}^t}{d}\E\left[\left.\left\|\triangledown f_j(x_j\omega^t)\right\|^2\right|\mathcal{F}^t  \right]=\frac{\mathcal{c}^t}{Nd}\sum_{j=1}^N\left\|\triangledown f_j(x_j\omega^t)\right\|^2,
\label{eq:cl5-3}
\end{align}
where we apply Lemma \ref{lem:1} for each $j$ again with  $w_1=1$, $w_2=0$, $\Phi=\left\{\left( \triangledown f_j(x_j\omega^t)\right)_{i}\right\}_{i=1}^{d}$, $g(z)=z^2$, $\mathcal{H}=\mathcal{F}^t$ and $\mathcal{B}=\mathcal{C}^t$. By inserting (\ref{eq:lem3-1}) from  Lemma \ref{lem:3} into (\ref{eq:cl5-3}) the claim in (\ref{eq:cl5-1}) follows.
\end{proof}
From Claim \ref{cl:e} we conclude
\begin{align}
\label{eq:pro3-1}
\E\left[ \left\|e^t\right\|^2 \vert \mathcal{F}^t\right]=\hat{\mathcal{O}}(\gamma_{t+1}^2).
\end{align}
By using the conditional Jensen's inequality and (\ref{eq:pro3-1}) we get
\begin{align}
\E\left[ \left\|e^t\right\| \vert \mathcal{F}^t\right]=\E\left[ \sqrt{\left\|e^t\right\|^2} \vert \mathcal{F}^t\right]\leq \sqrt{\E\left[ \left\|e^t\right\|^2 \vert \mathcal{F}^t\right]}\leq \sqrt{\eta}\gamma_{t+1}=\hat{\mathcal{O}}(\gamma_{t+1}).
\label{eq:pro3-2}
\end{align}
Consequently, applying the definition of $\mu^t$ yields
\begin{align}
\E\left[\left\|\mu^t\right\|^2|\mathcal{F}^t  \right]&\leq 2\left\{\E\left[\left. \left\|  e^t \right\|^2\right|\mathcal{F}^t\right] +\E\left[\left.\left\|\frac{1}{\mathcal{d}^t}\sum_{j\in\mathcal{D}^t}\bar{\triangledown}_{\omega_{\mathcal{C}^t}}f_j(x_j\omega^t)\right\|^2\right|\mathcal{F}^t\right]\right\}\nonumber\\
 &=\hat{\mathcal{O}}(\gamma_{t+1}^2)+\hat{\mathcal{O}}(1).
\label{eq:pro3-3} 
\end{align}
The second equality holds due to  (\ref{eq:cl5-1}) in Claim \ref{cl5} and (\ref{eq:pro3-1}). Then, the conclusion of (\ref{eq:cl3-3}) holds since
\begin{equation}
\begin{split}
\label{eq:pro3-4}
&\E\left[\left\|v^{t,k}\right\|^2|\mathcal{F}^t\right]=\E\left[\E\left[\left\|v^{t,k}\right\|^2|\mathcal{F}^t,,\mathcal{B}^t,\mathcal{C}^t,\mathcal{D}^t,\pi,j_{11}^{(1)}\cdots,j_{QP}^{(k-2)} \right]|\mathcal{F}^t\right]\\
&\leq \E\left[   \left.\sum_{q=1}^Q\sum_{p=1}^PL^2\left\|\gamma_{t+1}\left[ (k-1)\mu^t_{qp}+\sum_{i=1}^{k-1}v_{qp}^{t,i} \right]  \right\|^2  \right|\mathcal{F}^t\right] \\
&\leq kL^2\gamma_{t+1}^2\sum_{q=1}^Q\sum_{p=1}^P\left( \E\left[\left. \left\| (k-1) \mu^t_{qp} \right\|^2\right|\mathcal{F}^t\right]+ \sum_{i=1}^{k-1}\E\left[ \left\| v_{qp}^{t,i}\right\|^2|\mathcal{F}^t\right]\right)
\\
&\leq kL^2\gamma_{t+1}^2QP\left((k-1)^2\E\left[\left. \left\|  \mu^t \right\|^2\right|\mathcal{F}^t\right]+ \sum_{i=1}^{k-1}\E\left[ \left\| v^{t,i}\right\|^2|\mathcal{F}^t\right]\right)
\\
&=kL^2\gamma_{t+1}^2QP\left[2(k-1)^2\left(\hat{\mathcal{O}}(\gamma_{t+1}^2)+\hat{\mathcal{O}}(1)\right)+\hat{\mathcal{O}}(\gamma_{t+1}^2)\right]=\hat{\mathcal{O}}(\gamma_{t+1}^2).
\end{split}
\end{equation}
Thus, we maintain the same claims as those in Claim \ref{cl:3}. Then, (\ref{eq:pro1-6}) is revised to be
\begin{equation}
\begin{split}
\E\left[\omega^{t+1}-\omega^t|\mathcal{F}^t\right]&=-\gamma_{t+1}\E\left[B\mu^t+v^{t,1}+\cdots+v^{t,B}|\mathcal{F}^t\right]\\
&=-\gamma_{t+1}B\left( \expect[\big]{e^t|\mathcal{F}^t}+\frac{\mathcal{c}^t}{d}\triangledown F(\omega^t) \right)-\gamma_{t+1}\sum_{i=1}^{B}\expect[\big]{v^{t,i}|\mathcal{F}^t},
\label{eq:pro3-6}
\end{split}
\end{equation}
by using (\ref{eq:cl2-0}) in Claim \ref{cl2}.
Moreover, the expected value of the squared norm $\left\| \omega^{t+1}-\omega^t \right\|^2$ given $\mathcal{F}^t$ is
\begin{equation}
\begin{split}
\label{eq:pro3-7}
&\E\left[ \left\|\omega^{t+1}-\omega^t\right\|^2|\mathcal{F}^t\right]=\E\left[ \left\|\gamma_{t+1}\left[ B\mu^t+v^{t,1}+v^{t,2}+\cdots+v^{t,B} \right] \right\|^2|\mathcal{F}^t\right]\\
&\quad\leq \gamma_{t+1}^2(B+1) \left\{ B^2\E\left[ \left\| \mu^t\right\|^2|\mathcal{F}^t\right]+\sum_{i=1}^{B}\E\left[ \left\| v^{t,i}\right\|^2|\mathcal{F}^t\right]\right\}\\
&\quad=\hat{\mathcal{O}}(\gamma_{t+1}^2)\left\{\hat{\mathcal{O}}(\gamma_{t+1}^2) +\hat{\mathcal{O}}(1)+\hat{\mathcal{O}}(\gamma_{t+1}^2)  \right\}=\hat{\mathcal{O}}(\gamma_{t+1}^2),
\end{split}
\end{equation}
due to Claim \ref{cl:e}, (\ref{eq:cl5-1}), (\ref{eq:pro3-3}) and (\ref{eq:pro3-4}), which is identical to (\ref{eq:pro1-7}). Similarly, applying (\ref{eq:lem3-0}) and (\ref{eq:pro3-2}) yields
\begin{align}
-\gamma_{t+1}B\triangledown F(\omega^t)^T\E\left[ e^t|\mathcal{F}^t \right]&\leq \gamma_{t+1}B\left\|\triangledown F(\omega^t)\right\|\E\left[ \left\|e^t\right\||\mathcal{F}^t \right]\nonumber\\
&=\hat{\mathcal{O}}(\gamma_{t+1})\cdot\hat{\mathcal{O}}(\gamma_{t+1})=\hat{\mathcal{O}}(\gamma_{t+1}^2).
\label{eq:pro3-9}
\end{align}
Therefore, (\ref{eq:pro1-10}) remains the same as follows
\begin{equation*}
\begin{split}
&\E\left[F(\omega^{t+1})|\mathcal{F}^t\right]\leq F(\omega^t)+\triangledown F(\omega^t)^T\E\left[\left(\omega^{t+1}-\omega^t\right)|\mathcal{F}^t\right]+\frac{L}{2}\E \left[\left\| \omega^{t+1}-\omega^t\right\|^2|\mathcal{F}^t\right]\\
&=F(\omega^{t})+\triangledown F(\omega^t)^T\left\{ -\gamma_{t+1}B\left( \expect[\big]{e^t|\mathcal{F}^t}+\frac{\mathcal{c}^t}{d}\triangledown F(\omega^t) \right)-\gamma_{t+1}\sum_{i=1}^{B}\expect[\big]{v^{t,i}|\mathcal{F}^t} \right\} \\
&\quad\quad\quad\quad +\frac{L}{2}\E\left[\left\| \omega^{t+1}-\omega^t\right\|^2|\mathcal{F}^t\right]\\
&=F(\omega^t)-\gamma_{t+1} \frac{\mathcal{c}^tB}{d} \left\| \triangledown F(\omega^t)\right\|^2-\gamma_{t+1}B\triangledown F(\omega^t)^T\E\left[e^t|\mathcal{F}^t\right]-\gamma_{t+1} \triangledown F(\omega^t)^T \sum_{i=1}^{B} \E\left[v^{t,i}|\mathcal{F}^t\right]\\
&\quad\quad\quad\quad+\frac{L}{2}\E\left[\left\| \omega^{t+1}-\omega^t\right\|^2|\mathcal{F}^t\right]\\
&\leq F(\omega^t)-\gamma_{t+1} \frac{\mathcal{c}^tB}{d} \left\| \triangledown F(\omega^t)\right\|^2+\hat{\mathcal{O}}(\gamma_{t+1}^2)\leq  F(\omega^t)-\gamma_{t+1} \frac{\mathcal{c}^tB}{d} \left\| \triangledown F(\omega^t)\right\|^2+C_3\gamma_{t+1}^2,
\end{split}
\end{equation*}
where $C_3$ is a positive constant and we use (\ref{eq:pro1-8}), (\ref{eq:pro3-6}), (\ref{eq:pro3-7}) and (\ref{eq:pro3-9}). The remaining steps are identical with those in Proposition \ref{pro1}.
\end{proof}
The proof of Theorem \ref{thm:1} is the same as the proof of Theorem \ref{thm:5}, and the proof of Theorem \ref{thm:2} is the same as the proof of Theorem \ref{thm:6}.
\label{proof:thm1}
\label{proof:thm2}
\subsection*{G\tab Constant Learning Rate with Feature Sampling}
\begin{pro}
If Assumptions \hyperref[as:2]{2}-\hyperref[as:5]{5} hold true, and the learning rate is constant $\gamma_{t}=\gamma $ such that $BL\gamma QP\leq 1$ and $\gamma\leq 1$, and the sequences $(\mathcal{b}^t,\mathcal{c}^t,\mathcal{d}^t)_{t=0}^{\infty}$ satisfy the same conditions as in Theorem \ref{thm:1}, then the loss function error sequence $F(\omega^t)-F(\omega^*)$ generated by SODDA satisfies
\begin{align}
\E\left[\left.F(\omega^{t+1})-F(\omega^*)\right|\mathcal{F}^t   \right]\leq \left( 1-\frac{2\xi B}{d}\gamma \right)\left[ F(\omega^t)-F(\omega^*)\right]+C_4B^4\gamma^2,
\label{eq:pro4-0}
\end{align}
where $C_4$ is a positive constant.
\end{pro}
\begin{proof}
The proof is the same as the Proof of Proposition \ref{pro2} since
\begin{align}
\label{eq:cl4-11}
\E\left[\left.\left\|\mu^t\right\| \right|\mathcal{F}^t \right]=\E\left[\left.\sqrt{\left\|\mu^t\right\|^2 }\right|\mathcal{F}^t \right]\leq \sqrt{\E\left[\left.\left\|\mu^t\right\|^2 \right|\mathcal{F}^t \right]}=\sqrt{\hat{\mathcal{O}}(1)+\hat{\mathcal{O}}(\gamma^2)}=\hat{\mathcal{O}}(1),
\end{align}
due to (\ref{eq:pro3-3}).
\end{proof}
The proof of Theorem \ref{thm:3} is the same as the proof of Theorem \ref{thm:7}.
\label{proof:thm3}

\subsection*{H\tab Convergence to Optimality of Constant Learning Rate with Feature Sampling}
\subsubsection*{Proof of Theorem \ref{thm:4}}
\label{proof:thm4}
\begin{proof}
Given $\mathcal{b}^t=d$ and $\mathcal{d}^t=N$, applying Claim \ref{cl2} implies  
\begin{align}
\label{eq:thm4-1}
&\E\left[\left.\frac{1}{N}\sum_{j=1}^N\bar{\triangledown}_{\mathcal{C}^t}f_j(x_j\omega^t)\right|\mathcal{F}^t \right]=\frac{\mathcal{c}^t}{N}\triangledown F(\omega^t),\\
&\E\left[\left.\left\|\frac{1}{N}\sum_{j=1}^N\bar{\triangledown}_{\mathcal{C}^t}f_j(x_j\omega^t)\right\|^2\right|\mathcal{F}^t \right]=\frac{1}{\binom{d}{\mathcal{c}^t}}\sum_{\mathcal{C}^t}\left\|\frac{1}{N}\sum_{j=1}^N\bar{\triangledown}_{\mathcal{C}^t}f_j(x_j\omega^t)\right\|^2\nonumber\\
&\tab=\frac{1}{\binom{d}{\mathcal{c}^t}}\binom{d-1}{\mathcal{c}^t-1}\left\|\frac{1}{N}\sum_{j=1}^N\bar{\triangledown}_{\mathcal{C}^t}f_j(x_j\omega^t)\right\|^2=\frac{\mathcal{c}^t}{d}\left\|\triangledown F(\omega^t)\right\|^2,
\label{eq:thm4-2}
\end{align}
which in turn gives an upper bound to $\E\left[\left.\left\|\frac{1}{N}\sum_{j=1}^N\bar{\triangledown}_{\mathcal{C}^t}f_j(x_j\omega^t)\right\|\right|\mathcal{F}^t \right]$; that is
\begin{align}
\label{eq:thm4-3}
\E\left[\left.\left\|\frac{1}{N}\sum_{j=1}^N\bar{\triangledown}_{\mathcal{C}^t}f_j(x_j\omega^t)\right\|\right|\mathcal{F}^t \right]\leq \sqrt{\E\left[\left.\left\|\frac{1}{N}\sum_{j=1}^N\bar{\triangledown}_{\mathcal{C}^t}f_j(x_j\omega^t)\right\|\right|^2\mathcal{F}^t \right]}=\sqrt{\frac{\mathcal{c}^t}{d}}\left\|\triangledown F(\omega^t)\right\|.
\end{align}
On the one hand, given $BL\gamma QP\leq 1$, we find that
\begin{align}
\label{eq:thm4-4}
\sum_{l=2}^{B}(BL\gamma QP)\leq (B-1)(BL\gamma QP)^2=(B-1)B^2L^2\gamma^2 Q^2P^2.
\end{align}
By combining this expression with (\ref{eq:cl4-10}), we conclude that 
\begin{align}
\label{eq:thm4-5}
\E\left[\left.\left\|\sum_{i=1}^{B} v^{t,i} \right\| \right|\mathcal{F}^t \right]\leq \sum_{i=1}^{B}\E\left[\left.\left\|v^{t,i} \right\| \right|\mathcal{F}^t \right]\leq \E\left[\left.\left\|\mu^t \right\|\right|\mathcal{F}^t \right](B-1)B^2L\gamma QP.
\end{align}
On the other hand, given $\gamma\leq 1$,  from expression (\ref{eq:cl4-20}), we find that
\begin{equation}
\begin{split}
\label{eq:thm4-6}
(1+BL^2\gamma^2QP)^{B}\leq 1+\sum_{i=1}^{B}(B^2L^2\gamma^2QP)^i=1+B^3L^2QP.
\end{split}
\end{equation}
By applying (\ref{eq:thm4-6}) to (\ref{eq:cl4-19}), we deduce that
\begin{align}
\label{eq:thm4-7}
\E\left[\left.\left\|\sum_{i=1}^{B} v^{t,i}\right\|^2 \right|\mathcal{F}^t \right]\leq B\sum_{i=1}^{B}\E\left[\left.\left\| v^{t,i}\right\|^2 \right|\mathcal{F}^t \right]\leq B^5(1+B^3L^2QP)L^2\gamma^2QP\E\left[\left.\left\|\mu^t\right\|^2 \right|\mathcal{F}^t \right].
\end{align}
Second, let us evaluate the error after each iteration. By summing up all increments in iteration $t$, we obtain
\begin{equation}
\begin{split}
\label{eq:thm4-8}
\E&\left[\left.\left\|\omega^{t+1}-\omega^* \right\|^2\right|\mathcal{F}^t\right]=\E\left[\left. \left\|\omega^t-\gamma\left(B\mu^t+\sum_{i=1}^{B}v^{t,i} \right)-\omega^* \right\|^2\right|\mathcal{F}^t \right]\\
&=\left\|\omega^t-\omega^*\right\|^2-2\left \langle\E\left[\left.\gamma\left(B\mu^t+\sum_{i=1}^{B}v^{t,i} \right) \right|\mathcal{F}^t \right],\omega^t-\omega^*  \right \rangle+\E\left[\left.\left\|\gamma\left(B\mu^t+\sum_{i=1}^{B}v^{t,i} \right) \right\|^2\right|\mathcal{F}^t \right]\\
&\leq \left\|\omega^t-\omega^*\right\|^2-2\gamma B\left \langle\E\left[\left.\mu^t \right|\mathcal{F}^t \right],\omega^t-\omega^*  \right \rangle+2\gamma \E\left[\left.\left\|\sum_{i=1}^{B}v^{t,i}\right\|  \right|\mathcal{F}^t \right]\left\|\omega^t-\omega^* \right\|\\
&\tab +2\gamma^2B^2\E\left[\left.\left\|\mu^t\right\|^2\right|\mathcal{F}^t \right]+ 2\gamma^2\E\left[\left.\left\|\sum_{i=1}^{B}v^{t,i}\right\|^2\right|\mathcal{F}^t \right].
\end{split}
\end{equation}
Based on (\ref{eq:thm4-1}), (\ref{eq:thm4-2}), (\ref{eq:thm4-3}), (\ref{eq:thm4-5}) and (\ref{eq:thm4-7}), (\ref{eq:thm4-8}) can be further simplified as
\begin{equation}
\begin{split}
\label{eq:thm4-9}
\E&\left[\left.\left\|\omega^{t+1}-\omega^* \right\|^2\right|\mathcal{F}^t\right]\leq \left\|\omega^t-\omega^*\right\|^2-\frac{2\gamma B\mathcal{c}^t}{d}\left\langle\triangledown F(\omega^t),\omega^t-\omega^*  \right \rangle+\frac{2\gamma^2B^2\mathcal{c}^t}{d}\left\|\triangledown F(\omega^t)\right\|^2\\
& +2(B-1)B^2L\gamma^2 QP\E\left[\left.\left\|\mu^t \right\|\right|\mathcal{F}^t \right]\left\|\omega^t-\omega^*\right\|+2B^5(1+B^3L^2QP)L^2\gamma^4QP\E\left[\left.\left\|\mu^t\right\|^2 \right|\mathcal{F}^t \right]\\
&\leq \left\|\omega^t-\omega^*\right\|^2-\frac{2\gamma B\mathcal{c}^t}{d}\left\langle\triangledown F(\omega^t),\omega^t-\omega^*  \right \rangle+\frac{2\gamma^2B^2\mathcal{c}^t}{d}\left\|\triangledown F(\omega^t)\right\|^2\\
& +2(B-1)B^2L\gamma^2 QP \sqrt{\frac{\mathcal{c}^t}{d}}\left\| \triangledown F(\omega^t)\right\|\left\|\omega^t-\omega^*\right\|+2B^5(1+B^3L^2QP)L^2\gamma^4QP\frac{\mathcal{c}^t}{d}\left\|\triangledown F(\omega^t)\right\|^2.
\end{split}
\end{equation}
Recalling the Lipschitz continuity of the gradient of the objective function stated in Assumption \hyperref[as:3]{3}, we have that [\cite{nesterov2013introductory},Theorem 2.1.5]
\begin{align}
\label{eq:thm4-10}
\frac{1}{L}\left\|\triangledown F(\omega^t)\right\|^2=\frac{1}{L}\left\|\triangledown F(\omega^t)-\triangledown F(\omega^*)\right\|^2\leq \left \langle\triangledown F(\omega^t)- \triangledown F(\omega^*),\omega^t-\omega^* \right \rangle.
\end{align}
Because $F(\omega)$ is strongly convex by Assumption \hyperref[as:2]{2}, we have
\begin{align}
\label{eq:thm4-11}
\frac{1}{\xi}\left\|\triangledown F(\omega^t)\right\|=\frac{1}{\xi}\left\|\triangledown F(\omega^t)-\triangledown F(\omega^*)\right\|\geq \left\|\omega^t-\omega^*\right\|.
\end{align}
By using (\ref{eq:thm4-10}) and (\ref{eq:thm4-11}), \ref{eq:thm4-9} can be reformulated as 
\begin{equation}
\begin{split}
\label{eq:thm4-12}
\E&\left[\left.\left\|\omega^{t+1}-\omega^* \right\|^2\right|\mathcal{F}^t\right]\leq \left\|\omega^t-\omega^*\right\|^2-\frac{2\gamma B\mathcal{c}^t}{Ld}\left\|\triangledown F(\omega^t)\right\|^2+\frac{2\gamma^2B^2\mathcal{c}^t}{d}\left\|\triangledown F(\omega^t)\right\|^2\\
&\tab+\frac{2(B-1)B^2L\gamma^2QP}{\xi}\sqrt{\frac{\mathcal{c}^t}{d}}\left\|\triangledown F(\omega^t)\right\|^2+2B^5(1+B^3L^2QP)L^2\gamma^4QP\frac{\mathcal{c}^t}{d}\left\|\triangledown F(\omega^t)\right\|^2\\
&=\left\|\omega^t-\omega^*\right\|^2+\left(-\frac{2\gamma B\mathcal{c}^t}{Ld} +\frac{2\gamma^2B^2\mathcal{c}^t}{d}+\frac{2(B-1)B^2L\gamma^2QP}{\xi}\sqrt{\frac{\mathcal{c}^t}{d}}\right.\\
&\tab\left.+2B^5(1+B^3L^2QP)L^2\gamma^4QP\frac{\mathcal{c}^t}{d}\right)\left\|\triangledown F(\omega^t)\right\|^2\\
&=\left\|\omega^t-\omega^*\right\|^2+A(t)\left\|\triangledown F(\omega^t)\right\|^2,
\end{split}
\end{equation}
with 
\begin{align*}
A(t)=-\frac{2\gamma B\mathcal{c}^t}{Ld} +\frac{2\gamma^2B^2\mathcal{c}^t}{d}+\frac{2(B-1)B^2L\gamma^2QP}{\xi}\sqrt{\frac{\mathcal{c}^t}{d}}+2B^5(1+B^3L^2QP)L^2\gamma^4QP\frac{\mathcal{c}^t}{d}.
\end{align*}
Therefore, $\E\left[\left.\left\|\omega^{t+1}-\omega^* \right\|^2\right|\mathcal{F}^t\right]\leq \left\|\omega^t-\omega^*\right\|^2$ as long as $A(t)\leq 0$ for all $t$, $BL\gamma QP\leq 1$ and $\gamma\leq 1$.

In view of Assumption \hyperref[as:3]{3} we have
\begin{equation}
\begin{split}
\label{eq:thm4-13}
\E&\left[\left.F(\omega^{t+1}) \right|\mathcal{F}^t \right]\leq \E\left[\left.F(\omega^{t})+\left \langle \triangledown F(\omega^t),\omega^{t+1}-\omega^t \right \rangle +\frac{L}{2}\left\|\omega^{t+1}-\omega^t \right\|^2\right|\mathcal{F}^t \right]\\
&=F(\omega^t)+\E\left[\left.\left \langle \triangledown F(\omega^t),-\gamma B\mu^t-\gamma \sum_{i=1}^tv^{t,i} \right \rangle \right|\mathcal{F}^t \right]+\frac{L}{2}\E\left[\left.\left\|\gamma B\mu^t+\gamma \sum_{i=1}^{B}v^{t,i} \right\|^2 \right|\mathcal{F}^t \right]\\
&\leq F(\omega^t)-\gamma B\left \langle \triangledown F(\omega^t),\E\left[\left.\mu^t\right|\mathcal{F}^t \right] \right \rangle +\gamma\left\|\triangledown F(\omega^t)\right\|\E\left[\left.\left\|\sum_{i=1}^{B}v^{t,i} \right\|\right|\mathcal{F}^t\right]+L\gamma^2 B^2\E\left[\left.\left\|\mu^t\right\|^2 \right|\mathcal{F}^t \right]\\
&\tab +L\gamma^2\E\left[\left.\left\| \sum_{i=1}^{B}v^{t,i} \right\|^2 \right|\mathcal{F}^t \right].
\end{split}
\end{equation}
Substituting (\ref{eq:thm4-1}), (\ref{eq:thm4-2}), (\ref{eq:thm4-5}) and (\ref{eq:thm4-7}) implies
\begin{equation}
\begin{split}
\label{eq:thm4-14}
\E&\left[\left.F(\omega^{t+1}) \right|\mathcal{F}^t \right]\leq F(\omega^t)-\frac{\gamma B\mathcal{c}^t}{d}\left\|\triangledown F(\omega^t)\right\|^2+(B-1)B^2L\gamma^2 QP\sqrt{\frac{\mathcal{c}^t}{d}}\left\|\triangledown F(\omega^t)\right\|^2\\
&\tab +\frac{L\gamma^2 B^2 \mathcal{c}^t}{d}\left\|\triangledown F(\omega^t)\right\|^2+B^5(1+B^3L^2QP)L^3\gamma^4QP\frac{\mathcal{c}^t}{d}\left\|\triangledown F(\omega^t)\right\|^2\\
&=F(\omega^t)+\left(-\frac{\gamma B\mathcal{c}^t}{d}+(B-1)B^2L\gamma^2 QP\sqrt{\frac{\mathcal{c}^t}{d}}+\frac{L\gamma^2 B^2 \mathcal{c}^t}{d}\right.\\
&\tab\left.+B^5(1+B^3L^2QP)L^3\gamma^4QP\frac{\mathcal{c}^t}{d} \right)\left\|\triangledown F(\omega^t)\right\|^2\\
&=F(\omega^t)+B(t)\left\|\triangledown F(\omega^t)\right\|^2.
\end{split}
\end{equation}
A similar requirement is needed in (\ref{eq:thm4-14}) as that in (\ref{eq:thm4-12}), i.e. $B(t)<0$ for all $t$.

Let us denote $\bigtriangleup_t=F(\omega^t)-F(\omega^*)$. Then if $A(t)\leq 0$ for all $t$, we obtain
\begin{align*}
\bigtriangleup_t=F(\omega^t)-F(\omega^*)\leq \left\langle\triangledown F(\omega^t),\omega^t-\omega^* \right\rangle\leq \left\|\omega^0-\omega^*\right\|\left|\triangledown F(\omega^t)\right\|.
\end{align*}
Thus, if $B(t)<0$ for all $t$ and by using the law of iterated expectation, (\ref{eq:thm4-14}) becomes
\begin{align*}
\E\left[\bigtriangleup_{t+1}\right]\leq \E\left[\bigtriangleup_k\right]+\min_tB(t)\E\left[\left\|\triangledown F(\omega^t)\right\|^2\right]\leq \E\left[\bigtriangleup_k\right]+\frac{\min_tB(t)}{\left\|\omega^0-\omega^*\right\|^2}\E\left[\bigtriangleup_k^2\right],
\end{align*}
which in turn yields
\begin{align*}
\frac{1}{\E\left[\bigtriangleup_{t+1}\right]}\geq \frac{1}{\E\left[\bigtriangleup_{t}\right]}-\frac{\min_tB(t)}{\left\|\omega^0-\omega^*\right\|^2}\frac{\E\left[\bigtriangleup_{t}\right]}{\E\left[\bigtriangleup_{t+1}\right]}\geq \frac{1}{\E\left[\bigtriangleup_{t}\right]}-\frac{\min_tB(t)}{\left\|\omega^0-\omega^*\right\|^2}.
\end{align*}
Summing up these inequalities, we get
\begin{align*}
\frac{1}{\E\left[\bigtriangleup_{t+1}\right]}\geq \frac{1}{\bigtriangleup_{0}}-\frac{\min_tB(t)}{\left\|\omega^0-\omega^*\right\|^2}(t+1),
\end{align*}
which in turn yields
\begin{equation}
\label{eq:thm4-15}
\lim\limits_{t\to\infty}\E\left[\bigtriangleup_{t+1}\right]=0. 
\end{equation}
Hence, (\ref{eq:thm4-0}) follows from (\ref{eq:thm4-15}) and similar reasoning as Theorem \ref{thm:1}, provided $A(t)\leq 0$, $B(t)<0$ for all $t$ and $BL\gamma QP\leq 1$, i.e.
\begin{align}
\label{eq:thm4-16}
&\frac{2\gamma B\mathcal{c}^t}{Ld} \geq \frac{2\gamma^2B^2\mathcal{c}^t}{d}+\frac{2(B-1)B^2L\gamma^2QP}{\xi}\sqrt{\frac{\mathcal{c}^t}{d}}+2B^5(1+B^3L^2QP)L^2\gamma^4QP\frac{\mathcal{c}^t}{d}\\
\label{eq:thm4-17}
&\frac{\gamma B\mathcal{c}^t}{d}>(B-1)B^2L\gamma^2 QP\sqrt{\frac{\mathcal{c}^t}{d}}+\frac{L\gamma^2 B^2 \mathcal{c}^t}{d}+B^5(1+B^3L^2QP)L^3\gamma^4QP\frac{\mathcal{c}^t}{d}\\
\label{eq:thm4-18}
&BL\gamma QP\leq 1\\
\label{eq:thm4-18-2}
&\gamma\leq 1.
\end{align}
By multiplying (\ref{eq:thm4-16}) and (\ref{eq:thm4-17}) by $\frac{1}{2\gamma B}$ and $\frac{1}{\gamma B}$,  respectively, we have that
\begin{align}
\label{eq:thm4-19}
&\frac{\mathcal{c}^t}{Ld} \geq \frac{\gamma B\mathcal{c}^t}{d}+\frac{(B-1)BL\gamma QP}{\xi}\sqrt{\frac{\mathcal{c}^t}{d}}+B^4(1+B^3L^2QP)L^2\gamma^3QP\frac{\mathcal{c}^t}{d},\\
\label{eq:thm4-20}
&\frac{\mathcal{c}^t}{d}>(B-1)BL\gamma QP\sqrt{\frac{\mathcal{c}^t}{d}}+\frac{L\gamma B \mathcal{c}^t}{d}+B^4(1+B^3L^2QP)L^3\gamma^3QP\frac{\mathcal{c}^t}{d}.
\end{align}
Note that if we are able to find a constant learning rate $\gamma$ satisfying 
\begin{align}
\label{eq:thm4-21}
&\bar{A}_1=\frac{\min_t\mathcal{c}^t}{Ld} \geq \gamma \left[\left(B+\frac{(B-1)BL QP}{\xi}\right)+B^4(1+B^3L^2QP)L^2\gamma^2QP\right]=\bar{B}_1\gamma+\bar{C}_1\gamma^3\\
\label{eq:thm4-22}
&\bar{A}_2=\frac{\min_t\mathcal{c}^t}{d}>\gamma\left[\left((B-1)BL QP+LB\right) +B^4(1+B^3L^2QP)L^3\gamma^2QP\right]=\bar{B}_2\gamma+\bar{C}_2\gamma^3,
\end{align}
with
\begin{align*}
&\bar{A}_1=\frac{\min_t\mathcal{c}^t}{Ld} \\
&\bar{B}_1=B+\frac{(B-1)BL QP}{\xi}\\
&\bar{C}_1=B^4(1+B^3L^2QP)L^2QP\\
&\bar{A}_2=\frac{\min_t\mathcal{c}^t}{d}\\
&\bar{B}_2=(B-1)BL QP+LB\\
&\bar{C}_2=B^4(1+B^3L^2QP)L^3QP,
\end{align*}
then the same constant learning rate $\gamma$ is also valid for (\ref{eq:thm4-19}) and (\ref{eq:thm4-20}). Observing that the right-hand sides of both (\ref{eq:thm4-21}) and (\ref{eq:thm4-22}) have the same form, i.e. they are both cubic equations with $0$ being the only real root. Solving (\ref{eq:thm4-21}) and (\ref{eq:thm4-22}) show that
\begin{align*}
\gamma\in \left(0,\min\left\{\gamma_1,\gamma_2\right\}\right),
\end{align*} 
where
\begin{align*}
&\gamma_1=-2\sqrt{\frac{\bar{B}_1}{3\bar{C}_1}}\sinh\left(\frac{1}{3}\arcsinh\left(-\frac{3\bar{A}_1}{2\bar{B}_1}\sqrt{\frac{3\bar{C}_1}{\bar{B}_1}} \right) \right)\\
&\gamma_2=-2\sqrt{\frac{\bar{B}_2}{3\bar{C}_2}}\sinh\left(\frac{1}{3}\arcsinh\left(-\frac{3\bar{A}_2}{2\bar{B}_2}\sqrt{\frac{3\bar{C}_2}{\bar{B}_2}} \right) \right).
\end{align*} 
Combining (\ref{eq:thm4-18}), (\ref{eq:thm4-18-2}) with the above equation, finally, the constant learning rate is required to be 
\begin{align*}
\gamma\in\left(0,\min\left\{1,\frac{1}{BLQP},\gamma_1,\gamma_2\right\}\right).
\end{align*}
\end{proof}

\end{document}